\documentclass[12pt, hidelinks]{article}
 
\usepackage[margin=1in]{geometry} 

\usepackage{amsmath,amsthm,amssymb, graphicx, multicol, array,hyperref, comment, caption, subcaption, rotating, longtable, xpatch, bbm, color,xcolor, float,mathtools,bm,setspace,amsfonts,dcolumn,booktabs,adjustbox,diagbox}
\usepackage[section]{placeins}

\usepackage{pdflscape}
\usepackage{textcomp}

\def\*#1{\mathbf{#1}}
\def\+#1{\mathbb{#1}}

\newcommand{\I}{{-1}}
\newcommand{\T}{\top}

\newcommand{\pool}{\mathrm{fed}}
\newcommand{\npool}{\mathrm{pool}}
\newcommand{\concat}{\mathrm{cb}}
\newcommand{\mle}{\mathrm{mle}}
\newcommand{\HT}{\mathrm{ipw{\text -}mle}}

\newcommand{\aipw}{\mathrm{aipw}}
\newcommand{\Var}{\mathrm{Var}}
\newcommand{\norm}[1]{\left\lVert#1\right\rVert}
\newcommand{\RNum}[1]{\uppercase\expandafter{\romannumeral #1\relax}}

\newcommand{\diag}{\mathrm{diag}}
\newcommand{\pad}{\mathrm{pad}}

\newcommand{\betavarpi}{\bm\beta,\varpi}

\newcommand{\sep}{\mathrm{sep}}
\newcommand{\joint}{\mathrm{joint}}

\newcommand{\ate}{\mathrm{ate}}
\newcommand{\att}{\mathrm{att}}

\newtheorem{assumption}{Assumption}
\newtheorem{theorem}{Theorem}
\newtheorem{lemma}{Lemma}
\newtheorem{proposition}{Proposition}

\newtheorem{condition}{Condition}
\newtheorem{remark}{Remark}
\newcommand{\sk}{{(k)}}
\newcommand{\sj}{{(j)}}
\newcommand{\pr}{\mathrm{pr}}
\newcommand{\ms}{\mathrm{M}}
\newcommand{\optum}{\mathrm{O}}
\newcommand{\BM}{\mathrm{bm}}
\newcommand{\s}{\mathrm{r}}
\newcommand{\uns}{\mathrm{unr}}
\newcommand{\ivw}{\mathrm{ivw}}

\definecolor{Red}{rgb}{1,1,1}
\definecolor{Blue}{rgb}{0,0,1}
\definecolor{Black}{rgb}{0,0,0}

\def\blue{\color{Black}}

\def\cmtfinal#1{{\textcolor{black}{#1}}}

\usepackage[toc,title,page]{appendix}

\usepackage[margin=10pt,labelfont=bf]{caption}
\newcommand\tcaptab[1]{\captionsetup{position=top, font=normalsize, labelfont=bf, textfont=normalfont, justification=centering, margin=0mm, aboveskip=1mm, belowskip=0mm, labelsep=colon, singlelinecheck=false}\caption{#1}}
\newcommand\bnotetab[1]{\captionsetup{position=bottom, font=footnotesize,  textfont=normalfont, margin=1mm, skip=2mm, justification=justified, singlelinecheck=false}\caption*{#1}}
\newcommand\tcapfig[1]{\captionsetup{position=top, font=normalsize, labelfont=bf, textfont=normalfont, justification=centering, margin=0mm, aboveskip=2mm, belowskip=0mm, labelsep=colon, singlelinecheck=false}\caption{#1}}
\newcommand\bnotefig[1]{\captionsetup{position=bottom, font=footnotesize,  textfont=normalfont, margin=1mm, skip=2mm, justification=justified, singlelinecheck=false}\caption*{#1}}

\usepackage{natbib}
\setcitestyle{open={(},close={)},citesep={,}}
\begin{document}
\title{Federated Causal Inference in Heterogeneous Observational Data\footnote{{\scriptsize  This research is generously supported by Microsoft Research, the Office of Naval Research grant N00014-19-1-2468, and DARPA L2M program FA8650-18-2-7834. We thank Kristine Koutout, Molly Offer-Westort, seminar and conference participants at Berkeley, Microsoft Research, and American Causal Inference Conference for helpful comments. Code is available at \url{https://github.com/ruoxuanxiong/federated-causal-inference}. Data for this project were accessed using the Stanford Center for Population Health Sciences Data Core.}
}}
\author{Ruoxuan Xiong\thanks{ \scriptsize Emory University, Department of Quantitative Theory and Methods, \url{ruoxuan.xiong@emory.edu}.}
\and
Allison Koenecke\thanks{ \scriptsize Cornell University, Department of Information Science, \url{akoenecke@cornell.com}.}
\and
Michael Powell\thanks{\scriptsize United States Military Academy, Department of Mathematical Sciences, \url{mike.powell@westpoint.edu}.}
\and
Zhu Shen\thanks{\scriptsize Harvard University, Department of Biostatistics,  \url{zhushen@g.harvard.edu}.}
\and
Joshua T.~Vogelstein\thanks{\scriptsize Johns Hopkins University, Department of Biomedical Engineering, Institute for Computational Medicine, \url{jovo@jhu.edu}.}
\and
Susan Athey\thanks{\scriptsize Stanford University, Graduate School of Business,  \url{athey@stanford.edu}.}}

\begin{titlepage}
	\maketitle
	\thispagestyle{empty}

	\begin{abstract}
        We are interested in estimating the effect of a treatment applied to individuals at multiple sites, where data is stored locally for each site. Due to privacy constraints, individual-level data cannot be shared across sites; the sites may also have heterogeneous populations and treatment assignment mechanisms. Motivated by these considerations, we develop federated methods to draw inference on the average treatment effects of combined data across sites. Our methods first compute summary statistics locally using propensity scores and then aggregate these statistics across sites to obtain point and variance estimators of average treatment effects. We show that these estimators are consistent and asymptotically normal. To achieve these asymptotic properties, we find that the aggregation schemes need to account for the heterogeneity in treatment assignments and in outcomes across sites. We demonstrate the validity of our federated methods through a comparative study of two large medical claims databases. 

		\vspace{0.5cm}
		

	\end{abstract}
\end{titlepage}

\begin{onehalfspacing}

\section{Introduction}\label{sec:introduction}

In many settings, the same treatment is applied to populations in different environments, but data is stored separately for each environment. When the sample size in any one data set is too small to obtain precise estimates of treatment effects, it would often be beneficial, if possible, to use data across environments. However, the combination of individual-level data may be restricted by legal constraints, privacy concerns, proprietary interests, or competitive barriers. Therefore, it is useful to develop analytical tools that can reap the benefits of data combination without pooling individual-level data. Methods that accomplish this while sharing only aggregate data are referred to as ``federated'' learning methods. In this paper, we develop federated learning methods tailored to the problem of causal inference.  The methods allow for heterogeneous treatment effects and heterogeneous outcome models across data sets, and adjust for the imbalance in covariate distributions between treated and control samples. These methods provide treatment effect estimation and inference, that are shown to perform as well asymptotically as if the data sets were combined.

A motivating example for these methods is from \cite{koenecke2020alpha} who study two separate medical claims data sets, MarketScan and Optum. The two data sets are noticeably different: the data from Optum has more elderly patients and covers more years than the data from MarketScan. They found evidence from both data sets that exposure to alpha blockers, a class of commonly prescribed drugs, reduced the risk of adverse outcomes for patients with acute respiratory distress. However, existing federated methods are insufficient to draw inference on the drug effect, while accounting for the heterogeneity in populations between treated and control groups\footnote{Treated group that is exposed to alpha blockers has more elderly patients than control groups. This is because alpha blockers are commonly prescribed for chronic prostatitis, and the prostate generally worsens with age.} and across two separate data sets.

In this paper, we propose two main categories of federated inference methods to address this problem. One category is based on the Inverse Propensity-Weighted Maximum Likelihood
Estimator (IPW-MLE).\footnote{IPW-MLE includes linear models, logit models, Poisson models, and Cox models weighted by inverse propensity scores as special cases.} The other one is based on the Augmented Inverse Propensity Weighted (AIPW) Estimator. 
Our federated methods only use summary statistics of each data set and aim to estimate the parameters, such as average treatment effects, on the combined, individual-level data. Our methods provide point estimates and confidence intervals of these parameters that are asymptotically the same as if individual-level data were combined. We focus on IPW-MLE and AIPW for two main reasons. First, both estimators use propensity scores to balance covariate distributions between treated and control groups.
Second, both estimators enjoy the double robustness property \citep{bang2005doubly,wooldridge2007inverse}, that are robust to the misspecification of one of the propensity and outcome models. As a building block, we propose a supplementary category of federated methods based on MLE for the estimation of either propensity or outcome model, and used as the inputs for the two main categories.

We make four contributions in developing federated inference methods. First, we identify the conditions that need to be considered in federation for valid inference, such as the stability of propensity and outcome models across data sets. Our federated inference methods are then designed to vary with these conditions. Second, to support the validity of inference, we develop inferential theory for all of our federated methods. Our federated methods achieve the optimal convergence rate in the estimation of average treatment effects and other parameters of interest. Third, our federated methods are communication-efficient. We show one-way and one-time sharing of carefully constructed summary statistics is sufficient to obtain consistent federated estimators. Fourth, for IPW-MLE, the estimation error in the propensity model carries over to the estimation of the outcome model \citep{wooldridge2002inverse,wooldridge2007inverse}, which is often overlooked in practice, such as the standard \textsc{svyglm} package in \textsc{R}.\footnote{Overlooking this effect leads to an overestimate of variance and a loss of efficiency.} Our federated IPW-MLE explicitly accounts for this estimation error.

Our federated methods are particularly relevant when separate data sets have heterogeneous populations with heterogeneous treatment assignment and outcome models. This is the setting where conventional pooling methods, such as inverse variance weighting (IVW), can fail.\footnote{IVW is asymptotically the same as our federated IPW-MLE when data sets are homogeneous in the sense that covariate distributions, as well as propensity and outcome models, are stable across data sets.}  Let us revisit the example in \cite{koenecke2020alpha}.
We first estimate the effect of alpha blockers by IPW logistic regression\footnote{IPW logistic regression is a special case of IPW-MLE.} on each data set. We then combine the estimated effects by IVW and by our federated IPW-MLE across data sets. As shown in Figure \ref{fig:compare-ivw-ours}, the federated coefficient of alpha blockers from IVW lies outside of the interval defined by coefficients estimated on two separate data sets. This observation is counterintuitive as we expect the federated coefficient to measure the average effect of alpha blockers for patients in two data sets.\footnote{The main reason for the federated coefficient from IVW to lie outside this interval is that we have heterogeneous coefficients and variance-covariance matrices across datasets.
See Appendix \ref{subsec:toy-example-ivw} for a numerical example for more intuition. } In contrast, the federated coefficient from our proposed method lies between the coefficients estimated separately on two data sets, which makes more sense than IVW.

\begin{figure}[ht!]
		\centering
		\tcapfig{Coefficient of the Exposure to Alpha Blockers}
		\includegraphics[width=0.6\linewidth]{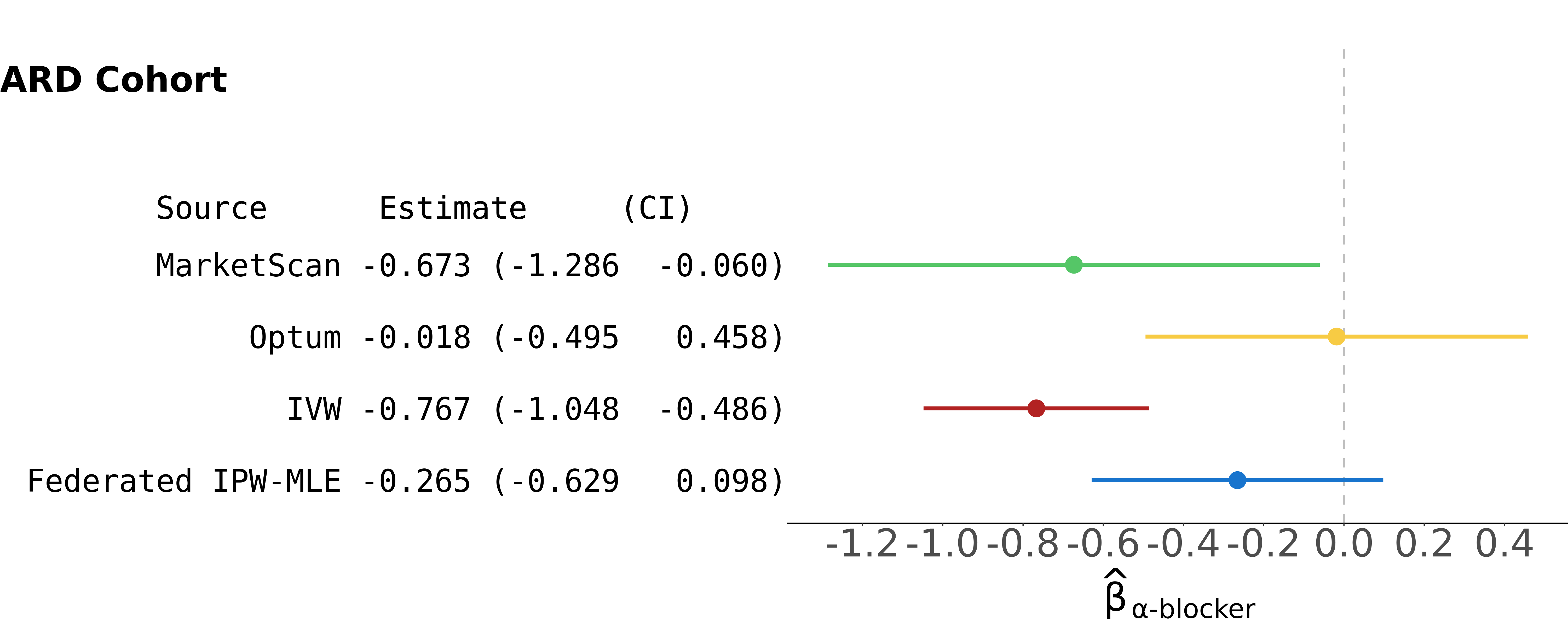}
		\label{fig:compare-ivw-ours}
		\bnotefig{This figure shows the estimated coefficient and its 95\% confidence interval of the exposure to alpha blockers in a logit outcome model, where the outcome indicates whether the patient with acute respiratory distress (ARD) received mechanical ventilation and then had in-hospital death.  We use IVW and our federated IPW-MLE to estimate the coefficient of alpha blockers on the combined data of MarketScan and Optum. The estimated coefficient from our federated IPW-MLE is more credible than that from IVW, because our federated coefficient lies between the interval defined by estimated coefficients on MarketScan and Optum, while coefficient from IVW does not. See Section \ref{sec:empirical} for more details.
		}
\end{figure}

Our work is related to multiple streams of literature which aim to learn and analyze data from multiple sources, including streams from biostatistics, data mining, and federated learning. Most studies in data mining and federated learning focus on estimating a centralized model, mostly through an iterative approach while preserving privacy, without considering inference.\footnote{Early developments in data mining provide methods to combine point estimates of model parameters in linear models \citep{du2004privacy,karr2005secure},  logit models \citep{fienberg2006secure,slavkovic2007secure}, and maximum likelihood estimators \citep{blatt2004distributed,karr2007secure,zhao2007information,lin2010privacy} across distributed information systems, with most methods being iterative. 
Recent advances, mainly in federated learning, aim to develop communication-efficient methods to optimize parameters across a large number of distributed heterogeneous agents, while preserving privacy \citep{konevcny2016federated,mcmahan2017communication,li2020federated}. Importantly, statistical inference is not a primary consideration in the aforementioned literature.} In contrast, our federated methods are non-iterative and are supported by asymptotic theory.\footnote{An iterative approach can provide estimators that are closer to those from the pooled individual-level data. However, we show that the difference between iterative and non-iterative approaches can be neglected asymptotically.}
Studies that provide inference are mostly concentrated in biostatistics.
Specifically, early studies in meta-analysis and meta-regression analysis provide inference, but largely center around combining randomized controlled trials, and a typically used pooling approach is IVW
(\cite{dersimonian1986meta,whitehead1991general}
among others). \cmtfinal{Recently, a growing number of studies} develop privacy-preserving methods to provide inference by pooling aggregate data across multiple studies\cmtfinal{: most of them are} tailored to specific parametric models, including linear models \citep{toh2018combining,toh2020privacy}, logit models \citep{duan2020learning}, Poisson models \citep{shu2019privacy}, Cox models \citep{shu2020inverse,shu2020variance}, {\blue and generalized linear models \citep{wolfson2010datashield}}\cmtfinal{, while \cite{jordan2018communication,duan2022heterogeneity} consider the efficient pooling of the more general MLE}. Among these studies, only \cite{toh2018combining} and \cite{shu2020inverse} account for nonrandom treatment assignments by using propensity scores, though the asymptotic theory is lacking. In contrast, we provide federated methods for a general class of parametric models that adjust for nonrandom treatment assignments and are supported by asymptotic theory.

Our work is most closely related to the recent studies of privacy-preserving methods for causal inference by \cite{vo2021federated}, \cite{han2021federated}, and \cite{han2022privacy}.\footnote{There has been a growing literature surrounding the development of causal inference methods, when individual-level data can be shared across multiple data sets, but data sets are collected under heterogeneous conditions
\citep[e.g.,][]{peters2016causal,bareinboim2016causal,rosenman2018propensity,rosenman2020combining,athey2020combining,rothenhausler2021anchor}.} 
\cite{vo2021federated} estimate treatment effects by modeling potential outcomes by Gaussian processes. {\blue \cite{han2021federated,han2022privacy} propose to estimate treatment effects for target populations by adaptively and optimally weighing source populations, accounting for the risk of negative transfer when source and target populations are heterogeneous. In contrast, our federated inference methods focus on treatment effects and other parameters of interest defined on the combined data, as opposed to on specific target data as in \cite{han2021federated,han2022privacy}. }

\section{Model, Assumptions, and Preliminaries}\label{sec:model}
In this section, we begin by stating the model setup and estimands for individual data sets in Section \ref{subsec:model-setup}. {\blue Next, we define the target parameters in our federated estimators in Section \ref{subsec:target-parameters}. We then review three widely used estimators (MLE, IPW-MLE, AIPW) on which our federated estimators are built in Section \ref{subsec:estimation-one-data}.} 
Next, we list the covariate and model conditions that need to be considered in federation in Section \ref{subsec:covariate-model-conditions}.  Finally, in Section \ref{subsec:weighting-methods}, we state the three weighting methods to aggregate information in our federated estimators. All the matrices in the asymptotic variance of MLE and IPW-MLE are summarized in Table \ref{tab:def-matrices}.

\subsection{Model Setup}\label{subsec:model-setup}

Suppose we have $D$ data sets, where $D$ is finite. Suppose data set $k \in \{1, \cdots, D\}$ has $n_k$ observations $(\*X^\sk_i, Y^\sk_i, W^\sk_i) \in \mathcal{X}_k \times \+R \times \{0,1\}$ that are drawn i.i.d. from some distribution $\mathbb{P}^\sk$. Here, $i \in \{1, \cdots, n_k\}$ indexes the subjects (e.g., patients), $\*X^\sk_i$ is a vector of $d_k$ observed covariates, $Y^\sk_i$ is the outcome of interest, $W^\sk_i$ is the treatment assignment, and $\mathcal{X}_k \subseteq \+R^{d_k}$. Both the types and the number of covariates can vary with data sets.  Let $n_\npool = \sum_{i=1}^D n_k$ be the total number of observations. \cmtfinal{ Here we study the setting where each data set has many observations, i.e., $n_k$ is large for all $k$. We assume the population fraction of observations in data set $k$, i.e., $p_k = \lim n_k/n_\npool$, exists, and is bounded away from $0$ and $1$. }

Under the Neyman-Rubin potential outcome model and the stable unit treatment value assumption \citep{imbens2015causal}, let $Y^\sk_i(1)$ be the outcome of subject $i$ if it is assigned treatment, and let $Y^\sk_i(0)$ be the outcome for the opposite case. For each data set $k$, suppose the following standard unconfoundedness assumption \citep{rosenbaum1983central} holds 
\[\{Y^\sk_i(0), Y^\sk_i(1) \} \perp W^\sk_i  \mid  \*X^\sk_i \]
and the following overlap assumption \citep{rosenbaum1983central} for the propensity score $e^\sk(\*x) = \pr(W^\sk_i = 1 \mid \*X^\sk_i = \*x )$ holds
\[\eta < e^\sk(\*x) < 1 - \eta \quad\quad \forall \*x \in \mathcal{X}_k  \]
for some $\eta > 0$. For each data set $k$, we define the average treatment effect (ATE), denoted as $\tau^\sk_\ate$, and average treatment effect on the treated (ATT), denoted as $\tau^\sk_\att$, as follows 
\begin{align}\label{eqn:definition-ate-att}
\tau^\sk_\ate  \coloneqq \+E[Y^\sk_i(1) - Y^\sk_i(0)], \quad \tau^\sk_\att \coloneqq \+E[Y^\sk_i(1) - Y^\sk_i(0) \mid W^\sk_i = 1].
\end{align}

\subsubsection{Parametric Models}
In this paper, we focus on parametric outcome and propensity models stated in Conditions \ref{cond:parametric-outcome} and \ref{cond:parametric-propensity} below. This is motivated by the common use of parametric outcome models in medical applications, for example, the use of logistic regression for estimating the odds ratio in epidemiological studies \citep{sperandei2014understanding}, Cox regression for survival analysis in clinical trials \citep{singh2011survival}, and generalized linear models (GLM) for assessing medical costs \citep{blough1999modeling,blough2000using}.
In addition, parametric models, such as logit models, are also commonly used to estimate propensity scores (e.g., \cite{imbens2015causal}, Ch. 13). The estimated parametric outcome and/or propensity model can also be used as the input in the estimation of the ATE and ATT. 

\begin{condition}[Parametric Outcome Model]\label{cond:parametric-outcome}
	For any data set $k$, the conditional density function of outcome $y$ on $\*x$ and $w$ follows a parametric model, denoted as $f^\sk_0(y \mid \*x, w, \bm\beta)$ with the true parameter values to be $\bm\beta_0^\sk$.
\end{condition}

\begin{condition}[Parametric Propensity Model]\label{cond:parametric-propensity}
	For any data set $k$, the conditional treatment probability $\pr(w=1 \mid \*x)$ follows a parametric model, denoted as $ e^\sk_0(\*x, \bm\gamma)$, with the true parameter values to be $\bm\gamma_0^\sk$.
\end{condition}

Given Conditions \ref{cond:parametric-outcome} and \ref{cond:parametric-propensity}, we can estimate the outcome and propensity models by maximizing the (weighted) likelihood function. 
Since the parametric models $f^\sk_0(y \mid \*x, w, \bm\beta)$ and $ e^\sk_0(\*x, \bm\gamma)$ are unknown a priori, the family of distributions chosen in the estimation of outcome and propensity models, denoted as $f^\sk(y \mid \*x, w, \bm\beta)$ and $e^\sk(\*x,\bm\gamma)$, may or may not contain the true structure, $f^\sk_0(y \mid \*x, w, \bm\beta)$ and $ e^\sk_0(\*x, \bm\gamma)$. Our federated estimators account for the possibility of model misspecification. We further discuss when the particular parameters of interest, e.g., ATE or ATT, on the combined data can still be consistently estimated by federated estimators in the presence of misspecification.

{\blue 
\subsection{Target Parameters}\label{subsec:target-parameters}

In this subsection, we define the target parameters that our federated methods aim to estimate.  Throughout this paper, the superscript ``$\sk$'' in a notation denotes an object estimated using data set $k$; the superscript ``$\concat$" denotes an object on the combined, individual-level data; and the superscript ``$\pool$" denotes a federated estimator. 

The target parameters are defined on the combined data that concatenate individual data across $D$ data sets together. \cmtfinal{The first set of target parameters are the parameters in the true conditional outcome density $f_0^\concat(\cdot)$ on the combined data, denoted as $\bm{\beta}_0^\concat$, where $f_0^\concat(\cdot)$ is defined as
\[f^\concat_0(Y^\sk_i \mid \*X^\sk_i, W^\sk_i, \bm{\beta}_0^\concat) \coloneqq \prod_{j = 1}^K\left[f^\sj_0(Y^\sj_i \mid \*X^\sj_i, W^\sj_i, \bm{\beta}_0^\sj)\right]^{\bm{1}(j = k)}, \qquad\qquad \forall k, \]
that equals the true conditional outcome density of data set $k$ when the observation is from data set $k$. $\bm{\beta}_0^\concat$ is defined as the union of $\bm{\beta}_0^{(1)}, \cdots, \bm{\beta}_0^{(K)}$. For example, if $\bm{\beta}_0^{(1)} = \cdots = \bm{\beta}_0^{(K)}$, then $\bm{\beta}_0^\concat = \bm{\beta}_0^\sk$ for any $k$; if $\bm{\beta}_0^{(1)}, \cdots, \bm{\beta}_0^{(K)}$ is completely different from one another, then $\bm{\beta}_0^\concat = (\bm{\beta}_0^{(1)}, \cdots, \bm{\beta}_0^{(K)})$.}

\cmtfinal{The second set of target parameters are the parameters in the true propensity $e_0^\concat(\cdot)$ on the combined data, denoted as $\bm{\gamma}_0^\concat$, 
where $e^\concat_0(\cdot)$ is defined as
\[e^\concat_0(W^\sk_i\mid \*X^\sk_i, \bm{\gamma}_0^\concat) \coloneqq \prod_{j = 1}^K \left[e^\sj_0(W^\sj_i\mid \*X^\sj_i, \bm{\gamma}_0^\sj)\right]^{\bm{1}(j = k)}, \qquad\qquad \forall k , \]
that equals the true propensity of data set $k$ when the observation is from data set $k$. Similar to $\bm{\beta}_0^\concat$, $\bm{\gamma}_0^\concat$ is defined as the union of $ \bm{\gamma}_0^{(1)}, \cdots, \bm{\gamma}_0^{(K)}$.}

The third set of target parameters are the ATE and ATT on the combined data, denoted as $\tau^\concat_\ate$ and $\tau^\concat_\att$, and are defined as 
\[\tau^\concat_\ate \coloneqq \sum_{k = 1}^D p_k \tau^\sk_\ate, \qquad \tau^\concat_\att \coloneqq \sum_{k = 1}^D p_k \tau^\sk_\att, \]
\cmtfinal{where $\tau^\concat_\ate$ and $\tau^\concat_\att$ are the averages of $\tau^\sk_\ate$ and $\tau^\sk_\att$ weighted by $p_k$, and $p_k$ is the population fraction of observations in data set $k$. Both $\tau^\concat_\ate$ and $\tau^\concat_\att$ do not depend on the sample size. }

If data sets can be combined at the individual level, then the standard approaches for a single data set (as reviewed in Section \ref{subsec:estimation-one-data} below) are applicable to estimate and draw inference on these target parameters. However, when data sets cannot be combined at the individual level, standard approaches are not applicable. 

We develop federated inference methods for these target parameters that only use aggregate information from each data set. The federated inference methods consist of both point and variance estimators of target parameters, thus allowing for the construction of confidence intervals of target parameters. These confidence intervals can be narrower than those obtained from a single data set. When treatment assignments are randomized, our federated methods include classical approaches such as IVW in meta-analysis, whereas when they are nonrandom, our federated estimators adjust for selection bias.

Note that in some settings, such as those in transfer learning, the target parameters of interest are defined on a specific target data set. Other data sets are used to improve the estimation efficiency on target data. In these settings, if propensity and outcome models are stable (defined in Conditions \ref{cond:stable-propensity} and \ref{cond:stable-outcome} below), then our federated estimators continue to be valid; otherwise, we need to account for the discrepancy between supplementary and target data sets to avoid the negative transfer. See \cite{han2021federated} for more discussion.



\subsection{Estimation Methods for Combined Individual-Level Data}\label{subsec:estimation-one-data}

This subsection reviews MLE, IPW-MLE, and AIPW that could be used to estimate the target parameters in Section \ref{subsec:target-parameters} when individual-level data could have been combined. \cmtfinal{As the individual data cannot be combined in practice, the estimators in this subsection are not feasible. In Section \ref{sec:estimation}, we introduce our federated estimators that are designed to approximate the estimators in this section using only the summary statistics of each data set.} 

\subsubsection{MLE for Model Parameters}\label{subsec:mle-estimation}
Under the parametric outcome model, 
we define the log-likelihood function of outcome conditional on covariates and treatment assignment on the combined data as 
\begin{equation}
    \bm\ell_{n_\npool}(\bm\beta) = \sum_{k = 1}^D \underbrace{\sum_{i = 1 }^{n_k} \log  f(Y^\sk_i \mid \*X^\sk_i, W^\sk_i,  \bm\beta)}_{\bm\ell_{n_k}(\bm\beta)}, \label{eqn:obj-mle}
\end{equation}
where $\bm\ell_{n_k}(\bm\beta)$ is the log-likelihood function on data set $k$. \cmtfinal{Let  $\hat{\bm\beta}^\concat_\mle$ be the solution that maximizes the log-likelihood function $\bm\ell_{n_\npool}(\bm\beta)$ and $\hat{\bm\beta}^\concat_\mle$ is an estimator of $\bm\beta^\concat$.} We can analogously use MLE to estimate the parameters in the parametric propensity model on the combined data.

\subsubsection{IPW-MLE for Model Parameters and Average Treatment Effects}\label{subsec:ht-estimation}
An alternative approach to estimating parameters in the outcome model is to use IPW-MLE, which adjusts the log-likelihood function by inverse propensity scores to estimate the population mean when data is nonrandomly missing} 
\begin{align}
\bm\ell_{n_\npool}(\bm\beta, \hat{e}) =  \sum_{k = 1}^D \underbrace{\sum_{i = 1 }^{n_k} \varpi^\sk_{i, \hat{e}}  \log  f(Y^\sk_i \mid \*X^\sk_i, W^\sk_i,  \bm\beta)}_{\bm\ell_{n_k}(\bm\beta, \hat{e})}, \label{eqn:obj-ht}
\end{align}
where the subscript ``$\hat{e}$'' is the abbreviation of the estimated propensity on the combined data, $\bm\ell_{n_k}(\bm\beta, \hat{e})$ is the weighted log-likelihood function on data set $k$, and $\varpi^\sk_{i, \hat{e}} $ is the weight for unit $i$ that can be
\begin{align*}
    \varpi^\sk_{i, \hat{e}} = \begin{cases}
        W^\sk_i/\hat{e}(\*X^\sk_i) + \big(1 - W^\sk_i\big)/\big(1 - \hat{e}(\*X^\sk_i)\big)  & \text{ATE weighting} \\
        W^\sk_i + \hat{e}(\*X^\sk_i) \big(1 - W^\sk_i\big)/\big(1 - \hat{e}(\*X^\sk_i)\big) & \text{ATT weighting.}
    \end{cases}
\end{align*}
Let $\hat{\bm\beta}^\concat_\HT$ be the estimator than maximizes the weighted log-likelihood $\bm\ell_{n_\npool}(\bm\beta, \hat{e})$. This estimator can be used to estimate treated and control outcomes, and form a doubly robust estimator for ATE and ATT \citep{wooldridge2007inverse}. See Appendix \ref{subsec:ipw-mle-double-robustness} for more details.

\subsubsection{AIPW for Average Treatment Effects}\label{subsec:aipw-estimate}
We can estimate ATE on the combined data using the AIPW estimator
\begin{align}\label{eqn:score}
\hat{\tau}^\concat_\ate =& \sum_{k = 1}^D \frac{n_k}{n_\npool} \cdot \underbrace{\frac{1}{n_k} \sum_{i = 1 }^{n_k} \hat{\phi}(\*X^\sk_i, W^\sk_i, Y^\sk_i )}_{\hat{\tau}^\sk_\ate} ,
\end{align}
that can be written as a weighted average of ATE across data sets by sample size, where \cmtfinal{$\hat{\phi}(\cdot)$ is the estimated score on the combined data and is defined as }
\begin{equation}
  \hat{\phi}(\*x, w, y) = \hat{\mu}_{(1)}(\*x) - \hat{\mu}_{(0)}(\*x)  + \frac{w}{\hat{e}(\*x)} \big( y -\hat{\mu}_{(1)}(\*x)  \big)- \frac{(1 - w )}{1 - \hat{e}(\*x)} \big(y - \hat{\mu}_{(0)}(\*x) \big)\, ,    \label{eqn:aipw}
\end{equation}
and where $\hat{\mu}_{(1)}(\*x) $ and $\hat{\mu}_{(0)}(\*x) $ are estimated conditional treated and control outcome models \cmtfinal{on the combined data}.\footnote{\cmtfinal{The parameters in $\hat{\mu}_{(w)}(\*x) $ and $\hat{e}(\*x)$ are omitted to account for the case where $\hat{\mu}_{(w)}(\*x) $ and $\hat{e}(\*x)$ are estimated by nonparametric methods when the individual-level data could have been combined.} }
If the estimand is ATT, then we can also use \eqref{eqn:score}, but the estimated score $\hat{\phi}(\cdot)$ is defined as
\begin{align}
\hat{\phi}(\*x, w, y ) =&  w \big(y - \hat{\mu}_{(1)}(\*x) \big)- \frac{\hat{e}(\*x)(1 - w )}{1 - \hat{e}(\*x)} \big(y - \hat{\mu}_{(0)}(\*x) \big).  \label{eqn:aipw-att}
\end{align}
AIPW has two prominent properties: doubly robustness \citep{robins1994estimation} and semiparametric efficiency.

\subsection{Covariate and Model Considerations in Federated Estimators}\label{subsec:covariate-model-conditions}

In this subsection, we introduce the conditions that need to be considered in the federation to obtain valid point and variance estimators of target parameters.

\begin{condition}[Known Propensity Score]\label{cond:known-propensity}
	For all data sets, the true propensity scores are known and used.  
\end{condition}

When true propensity scores are known and used, then we do not need to federate propensity models in federated IPW-MLE.

\begin{condition}[Stable Propensity Model]\label{cond:stable-propensity}
The set of covariates and the parameters in the propensity model are the same for all data sets, that is,  $\bm\gamma_0^\sj = \bm\gamma_0^\sk$ for any $j$ and $k$. 
\end{condition}

\begin{condition}[Stable Outcome Model]\label{cond:stable-outcome}
The set of covariates and the parameters in the outcome model are the same for all data sets, that is, $\bm\beta_0^\sj =\bm\beta_0^\sk$ for any $j$ and $k$.  
\end{condition}

\begin{condition}[Stable Covariate Distribution]\label{cond:stable-covariate} The set of covariates and their joint distribution are the same across all data sets. That is, $d_j = d_k$ and $\mathbb{P}^\sj(\*x) = \mathbb{P}^\sk(\*x) $ for any two data sets $j$ and $k$.
\end{condition}

{\blue We refer to data sets as being ``heterogeneous'' in settings where either Condition \ref{cond:stable-propensity}, \ref{cond:stable-outcome}, or \ref{cond:stable-covariate} is violated. If Condition \ref{cond:stable-outcome} holds (similarly for Condition \ref{cond:stable-propensity}), then the parameters on the combined data $\bm{\beta}^\concat_0$ equals $\bm{\beta}^\sk_0$ for any $k$; otherwise, we partition the parameters $\bm\beta^\sk = \big(\bm\beta_{\mathrm{s}}, \bm\beta^\sk_{\mathrm{uns}} \big)$ into shared parameters $\bm\beta_{\mathrm{s}}$ and dataset-specific parameters $\bm\beta^\sk_{\mathrm{uns}}$ for any $k$, and define the parameters on the combined data as $\bm\beta^\concat = (\bm\beta_{\mathrm{s}}, \bm\beta_{\mathrm{uns}}^{(1)}, \bm\beta_{\mathrm{uns}}^{(2)}, \cdots, \bm\beta_{\mathrm{uns}}^{(D)})$.\footnote{For ease of presentation, we assume there are no shared parameters across only a subset of data sets, but our estimator can be easily generalized to the opposite case. If there are some shared parameters across several but not all data sets, we just need to combine these parameters in $\bm\beta^\concat$. For example, if $\bm\beta_{\mathrm{uns}}^\sj$ and $\bm\beta_{\mathrm{uns}}^\sk$ are the same for $j$ and $k$,  then we merge $\bm\beta_{\mathrm{uns}}^\sj$ and $\bm\beta_{\mathrm{uns}}^\sk$ in $\bm\beta^\concat$.} For example, $\bm\beta_{\mathrm{s}}$ could include the parameters of interest, such as the treatment coefficient that we want to precisely estimate;  $\bm\beta^\sk_{\mathrm{uns}}$ could include nuisance parameters, such as the age coefficient in our empirical study.\footnote{Age coefficient has opposite signs in the two data sets in our empirical study, as shown in Figure \ref{fig:coef-age}. }  Note that choosing the partition generally encompasses a tradeoff between efficiency and robustness to model misspecification. See Section \ref{subsec:pool-mle-model-shift} for more discussion, and Section \ref{subsec:practical-guidance} for practical guidance on choosing the partition.

}



\begin{table}
	\centering
	\tcaptab{A Summary of Matrices in the Asymptotic Variance of MLE and IPW-MLE}
	{\scriptsize
		\begin{tabular}{c|c|c|c}
			\toprule
			Matrix & Expression  & Matrix & Expression  \\
			\midrule
			$\*A_{\bm\beta}$ & $ \+E \Big[  - \frac{\partial^2 \log f(y \mid \*x, w, \bm\beta)}{\partial \bm\beta \partial \bm\beta^\T}   \Big] $ & $\*A_{\bm\gamma}$  &  $ \+E \Big[ -  \frac{\partial^2 \log e(\*x, \bm\gamma)}{\partial \bm\gamma \partial \bm\gamma^\T}   \Big] $  \\
			\midrule
			$\*B_{\bm\beta}$ & $ \+E\Big[ \frac{\partial \log f(y \mid \*x, w, \bm\beta) }{\partial \bm\beta}  \big(\frac{\partial \log f(y \mid \*x, w, \bm\beta) }{ \partial \bm\beta}  \big)^\T \Big] $ & $\*B_{\bm\gamma}$ & $ \+E\Big[ \frac{\partial \log e(\*x, \bm\gamma) }{\partial \bm\gamma}  \big(\frac{\partial \log e(\*x, \bm\gamma) }{\partial \bm\gamma}  \big)^\T \Big] $ \\
			\midrule
			\midrule
			\multicolumn{2}{c|}{ATE weighting $\varpi_{i, {e}_{\bm\gamma}} = \frac{w_i}{{e}_{\bm\gamma}(\*x_i)} + \frac{1 - w_i}{1 - {e}_{\bm\gamma}(\*x_i)}$}  & \multicolumn{2}{c}{ATT weighting $\varpi_{i, {e}_{\bm\gamma}} = w_i + \frac{{e}_{\bm\gamma}(\*x_i)}{1 - {e}_{\bm\gamma}(\*x_i)} (1 - w_i)$} \\
			\midrule
			$\*A_{\betavarpi}$ & $ \+E \Big[ \Big( \frac{w}{e_{\bm\gamma}}   + \frac{1 - w}{1 - e_{\bm\gamma}}\Big)  \frac{\partial^2 \log f(y \mid  \*x, w, {\bm\beta}) }{\partial \bm\beta \partial \bm\beta^\T}  \Big] $ & $\*A_{\betavarpi}$  & $ \+E \Big[ \Big( w  + \frac{e_{\bm\gamma}(1 - w)}{1 - e_{\bm\gamma}}\Big)  \frac{\partial^2 \log f(y \mid  \*x, w, {\bm\beta}) }{\partial \bm\beta \partial \bm\beta^\T}\Big] $ \\
			\midrule
			$\*D_{\betavarpi}$ & $ \+E\Big[ \Big( \frac{w}{e_{\bm\gamma}}   + \frac{1 - w}{1 - e_{\bm\gamma}}\Big)^2  \frac{\partial \log f(y \mid \*x, w, \bm\beta)  }{\partial \bm\beta} \cdot \big(\frac{\partial \log f(y \mid \*x, w, \bm\beta)  }{ \partial \bm\beta}  \big)^\T \Big] $ & $\*D_{\betavarpi}$ & $ \+E\Big[ \Big( w  + \frac{e_{\bm\gamma}(1 - w)}{1 - e_{\bm\gamma}}\Big)^2  \frac{\partial \log f(y \mid \*x, w, \bm\beta)  }{\partial \bm\beta} \cdot \big(\frac{\partial \log f(y \mid \*x, w, \bm\beta)  }{ \partial \bm\beta}  \big)^\T \Big] $ \\
			\midrule
			$\*C_{\betavarpi}$ & $ \+E\Big[ \Big( \frac{w}{e^2_{\bm\gamma}}   - \frac{1 - w}{(1 - e_{\bm\gamma})^2}\Big)  \frac{\partial \log f(y \mid \*x, w, \bm\beta)  }{\partial \bm\beta} \cdot \big(\frac{\partial \log e(\*x, \bm\gamma)  }{ \partial \bm\gamma}  \big)^\T \Big] $ & $\*C_{\betavarpi,1}$  & $ \+E\Big[ - \frac{(1 - w)}{(1 - e_{\bm\gamma})^2}  \frac{\partial \log f(y \mid \*x, w, \bm\beta)  }{\partial \bm\beta} \cdot \big(\frac{\partial \log e(\*x, \bm\gamma)  }{ \partial \bm\gamma}    \big)^\T \Big] $ \\
			& & $\*C_{\betavarpi,2}$  & $ \+E\Big[ \Big( \frac{w}{e_{\bm\gamma}}  - \frac{e_{\bm\gamma}(1 - w)}{(1 - e_{\bm\gamma})^2}\Big)  \frac{\partial \log f(y \mid \*x, w, \bm\beta)  }{\partial \bm\beta} \cdot \big(\frac{\partial \log e(\*x, \bm\gamma)  }{ \partial \bm\gamma}    \big)^\T \Big] $ \\
			\bottomrule
		\end{tabular}
		\bnotetab{In the definitions of these matrices, $e_{\bm\gamma}$ denotes $e_{\bm\gamma}(\*x_i) = e(\*x_i, \bm\gamma)$ by a slight abuse of notation. }
	}
\label{tab:def-matrices}
\end{table}

\subsection{Three Weighting Methods}\label{subsec:weighting-methods}

{\blue We list the three weighting methods used in our federated estimators. The choice of weighting methods in each federated estimator 
is based on the functional form of the corresponding estimator for a single data set, as shown in Section \ref{sec:estimation}, and ensures that the federated estimators can be consistent, as shown in Section \ref{sec:results}.}

\subsubsection{Hessian Weighting}
{\blue Hessian weighting is used to estimate target parameters $\bm{\beta}^\concat_0$ and $\bm{\gamma}^\concat_0$ in the outcome and propensity models,} and is defined as
\begin{align}
\hat{\bm\beta}^\pool  = \bigg( \sum_{k = 1}^{D} \hat{\*H}_{\bm\beta}^\sk \bigg)^\I \bigg( \sum_{k = 1}^{D} \hat{\*H}_{\bm\beta}^\sk \hat{\bm\beta}^\sk  \bigg)\, ,\quad \text{ where } \hat{\*H}_{\bm\beta}^\sk = \frac{\partial^2 \bm\ell_{n_k}(\hat{\bm\beta}^\sk)}{\partial \bm\beta^\sk (\partial \bm\beta^\sk)^\top }\, . \label{eqn:hessian-weighting}
\end{align}
for parameters in the outcome model. 
For the propensity model, we just replace $\hat{\bm\beta}^\sk$ by $\hat{\bm\gamma}^\sk$ and $\hat{\*H}_{\bm\beta}^\sk$ by  $\hat{\*H}_{\bm\gamma}^\sk$ in  \eqref{eqn:hessian-weighting}.

\subsubsection{Sample Size Weighting}
{\blue Sample size weighting is used to obtain variance estimators (see more details in Tables \ref{tab:mle}, \ref{tab:ht}, and \ref{tab:aipw}), and is used to estimate ATE and ATT under unstable propensity or outcome models.}
For some generic scalar or matrix $\*M$, we refer to sample size weighting as
\begin{align}
\*M^\pool = \sum_{k = 1}^D \frac{n_k}{n_\npool} \*M^\sk\, ,\quad \text{ where }  n_\npool = \sum_{k = 1}^D n_k\, . \label{eqn:sample-size-weighting}
\end{align}

\subsubsection{Inverse Variance Weighting}\label{subsubsec:ivw}
{\blue Inverse variance weighting (IVW) is used to estimate ATE and ATT and their variance under stable propensity and outcome models.}
For some generic point estimator $\hat{\bm{\nu}}$, we refer to inverse variance weighting as
\begin{align}
\hat{\bm{\nu}}^\pool =& \Bigg( \sum_{k = 1}^{D} \left(\Var(\hat{\bm{\nu}}^\sk )\right)^\I \Bigg)^\I \Bigg( \sum_{k = 1}^{D} \left(\Var(\hat{\bm{\nu}}^\sk )\right)^\I \bm{\nu}^\sk  \Bigg)\,  ,\label{eqn:inverse-variance-weighting} \\
\widetilde{\Var}(\hat{\bm{\nu}}^\pool) =& n_\npool  \left( \sum_{k = 1}^{D} \left(\Var(\hat{\bm{\nu}}^\sk )\right)^\I \right)^\I, \label{eqn:inverse-variance-weighting-variance}
\end{align}
where $\Var(\hat{\bm{\nu}})$ is the variance of $\hat{\bm{\nu}}$, and $\widetilde{\Var}(\hat{\bm{\nu}})$ is $\Var(\hat{\bm{\nu}})$ multipled by the sample size.

\section{Federated Estimators}\label{sec:estimation}
In this section, we introduce three categories of federated inference methods that consist of both point and variance estimators of target parameters in Section \ref{subsec:target-parameters}. These three categories are based on MLE, IPW-MLE and AIPW, respectively. For each category, we start with the simple case in which the propensity and outcome models are stable.
We refer to the federated estimators in this case as \textbf{restricted} federated estimators. 
Next we consider the more challenging case in which at least one of propensity and outcome models is unstable.
The federated estimators for this case are referred to as \textbf{unrestricted} federated estimators, which are built on the corresponding restricted federated estimators. 

Figures \ref{flowchart:mle-unstable}, \ref{flowchart:ht-unstable} and \ref{flowchart:aipw} show the flowcharts of our federated inference methods under different conditions. Tables \ref{tab:mle}, \ref{tab:ht} and \ref{tab:aipw} provide the details of our federated methods.

\subsection{Federated MLE}
We introduce our federated MLE using the outcome model, where the target parameter is $\bm{\beta}^\concat$. However, our federated MLE is also applicable to the propensity model.

\subsubsection{Restricted Federated MLE for Stable Models (Condition \ref{cond:stable-propensity}/ \ref{cond:stable-outcome} Holds)}\label{subsec:pool-stable-mle}

{\blue
When outcome models are stable (i.e., $\bm\beta^\concat = \bm\beta^\sk$ for all $k$), we can use the restricted federated MLE for $\bm{\beta}^\concat$. Let $\hat{\bm\beta}^\pool_\mle$ be the federated point estimator that is obtained by first applying MLE on each data set $k$ to estimate  
parameter $\bm\beta^\sk$, and then using Hessian weighting in \eqref{eqn:hessian-weighting} to combine estimated parameters across all data sets.

We propose this federated estimator based on the objective of satisfying the first-order condition of MLE. When we use Hessian weighting, this objective can be satisfied with the key steps outlined below:
\begin{align*}
    \frac{\partial \sum_{k=1}^D \bm\ell_{n_k}(\hat{\bm\beta}^\pool_\mle) }{\partial \bm\beta} =& \sum_{k=1}^D \frac{\partial  \bm\ell_{n_k}(\bm\beta_0) }{\partial \bm\beta } + \sum_{k=1}^D \*H_{\bm{\beta}}^\sk  \left( \hat{\bm\beta}^\pool_\mle - \bm\beta_0  \right)  \\
   =& \sum_{k=1}^D \frac{\partial  \bm\ell_{n_k}(\bm\beta_0) }{\partial \bm\beta } + \sum_{k=1}^D \*H_{\bm{\beta}}^\sk  \left(  \hat{\bm\beta}^\sk_\mle - \bm\beta_0  \right) \tag{Hessian weighting of $\hat{\bm\beta}^\pool_\mle$} \\
   =& \sum_{k=1}^D \frac{\partial  \bm\ell_{n_k}(\hat{\bm\beta}^\sk_\mle) }{\partial \bm\beta} = 0 \tag{gradient at $\hat{\bm\beta}^\sk_\mle$ is zero for all $k$}
\end{align*}

}

Our federated variance estimator is obtained via a two-step procedure. First, we estimate the terms in the robust variance formula, $\*A_{\bm\beta}$ and $\*B_{\bm\beta}$ (see Table \ref{tab:def-matrices} for the definition), on each data set. Let $\hat{\*A}_{\bm\beta}^\sk$ and $\hat{\*B}_{\bm\beta}^\sk$ be the estimators on data set $k$. Second, we obtain the federated variance using sample size weighting\footnote{If the outcome model is correctly specified, the information matrix equivalence holds, implying that $ \*A_{\bm\beta} = \*B_{\bm\beta}$ and $\*V_{\bm\beta} = \*A_{\bm\beta}^\I$. Then we only need to estimate and combine $\*A_{\bm\beta}^\sk$.} 
\begin{align*}
    \hat{\*V}^\pool_{\bm\beta} =& (\hat{\*A}_{\bm\beta}^\pool)^\I \cdot \hat{\*B}_{\bm\beta}^\pool \cdot (\hat{\*A}_{\bm\beta}^\pool)^\I
\end{align*} 
where
\begin{align}\label{eqn:sample-size-weighting-federation}
\hat{\*A}_{\bm\beta}^\pool = \sum_{k = 1}^D \frac{n_k}{n_\npool} \hat{\*A}_{\bm\beta}^\sk\text{~~~~and~~~~} \hat{\*B}_{\bm\beta}^\pool = \sum_{k = 1}^D \frac{n_k}{n_\npool} \*B_{\bm\beta}^\sk.
\end{align}
This federated variance uses the robust variance formula and is, therefore, robust to outcome model misspecification \citep{white1982maximum}. We use sample size weighting here based on the property that $\*A_{\bm\beta}$ and $\*B_{\bm\beta}$ on the combined data equals the weighted average of the corresponding matrices on individual data sets by sample size.

\begin{figure}[t!]
	\centering
	\tcapfig{Flowchart for Federated MLE}
	\includegraphics[width=.6\linewidth]{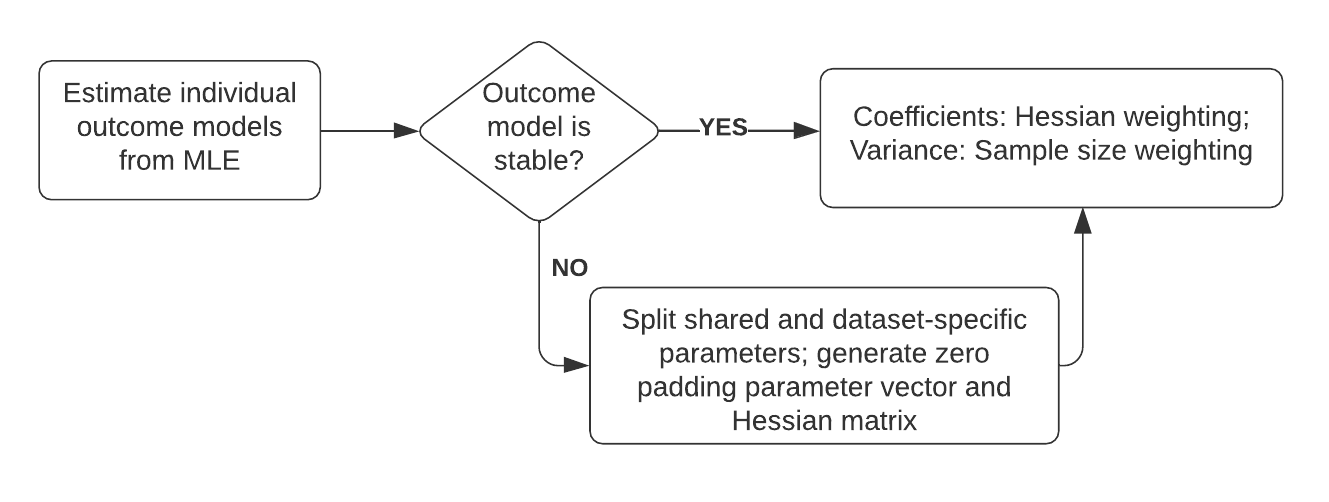}	
	\bnotefig{See Section \ref{subsec:practical-guidance} for practical guidance on determining whether the outcome model is stable.}
	\label{flowchart:mle-unstable}
\end{figure}

\subsubsection{Unrestricted Federated MLE for Unstable Models (Condition \ref{cond:stable-propensity}/ \ref{cond:stable-outcome} is Violated)}\label{subsec:pool-mle-model-shift}

{\blue 
Our unrestricted federated MLE is conceptually similar to our restricted federated MLE, but additionally handles the instability of parameters across datasets. Specifically, our unrestricted estimator only combines the shared parameters across data sets and leaves the dataset-specific parameters as they are in federation. The key to treating shared and dataset-specific parameters differently is to use a zero-padding technique.\footnote{Zero-padding is a commonly used technique in signal processing \citep{madan2016optimal} and deep learning \citep{o2015introduction} to pre-process inputs to the same length.}

Specifically, for each data set $k$, we pad $\bm\beta^\sk$ with zeros so that the padded $\bm\beta^\sk$, denoted as $\bm\beta^{\pad,\sk}$, is aligned with $\bm\beta^\concat = (\bm\beta_{\mathrm{s}}, \bm\beta_{\mathrm{uns}}^{(1)}, \bm\beta_{\mathrm{uns}}^{(2)}, \cdots, \bm\beta_{\mathrm{uns}}^{(D)})$. We similarly pad each matrix on data set $k$ so that it is aligned with the corresonding matrix on the combined data. 
Below we provide an example of zero-padding $\bm\beta^{(1)}$ and $\*H^{(1)}_{\bm{\beta}}$ for data set $k = 1$:
}
{\small
\begin{align}
\bm\beta^{\pad,(1)} = \begin{pmatrix}
\bm\beta_{\mathrm{s}}  \\ {\bm\beta}^\sk_{\mathrm{uns}} \\ \mathbf{0} 
\end{pmatrix}, \quad
{\*H}_{\bm\beta}^{\pad,(1)}= \begin{pmatrix}
\begin{tabular}{ccccc}
${\*H}_{\bm\beta,\mathrm{s},\mathrm{s}}$  & ${\*H}_{\bm\beta,\mathrm{s}, \mathrm{uns}}^{(1)}$ & $\mathbf{0}$  \\
${\*H}_{\bm\beta,\mathrm{uns},\mathrm{s}}^{(1)}$ & ${\*H}_{\bm\beta,\mathrm{uns}, \mathrm{uns}}^{(1)}$ & $\mathbf{0}$ \\
$\mathbf{0}$ & $\mathbf{0}$ & $\mathbf{0}$   
\end{tabular}
\end{pmatrix}.\label{eqn:expand-beta-H}
\end{align}}

The zero-padding of other vectors and matrices for other $k$ is conceptually the same. The unrestricted point and variance estimator essentially applies the restricted point and variance estimator to the padded parameters and matrices. In this way, the unrestricted estimator only federates the shared parameters. 

Note that it is possible to treat some parameters as dataset-specific parameters even though they are stable. This approach does not affect the consistency of the federated estimator; however, as the number of parameters on the combined data increases, the federated estimator is weakly less efficient than that using the most parsimonious specification, as stated in the following proposition. See Table \ref{tab:efficiency-comparison} Appendix  \ref{efficiency-comparison} for a numerical example.

\begin{table}[t!]
	\centering
	\tcaptab{Federated Maximum Likelihood Estimator}
	{\renewcommand\arraystretch{1.25}
		\footnotesize
		\begin{tabular}{l|l|l} \toprule
			\multicolumn{1}{p{2.5cm}|}{\raggedright Description} & \multicolumn{1}{p{5.2cm}|}{\raggedright Assume Stable Outcome Model (MLE \#1)}  &  \multicolumn{1}{p{6.2cm}}{\raggedright Assume Unstable Outcome Model (MLE \#2)}    \\  \midrule
			\multicolumn{1}{p{3.cm}|}{\raggedright Stable outcome model} &  \multicolumn{1}{p{3.cm}|}{\raggedright yes}  &  \multicolumn{1}{p{3.cm}}{\raggedright no}    \\ \midrule
			\multicolumn{1}{p{3.cm}|}{\raggedright Parameter $\bm\beta$ federation} &  \multicolumn{1}{p{5.2cm}|}{\raggedright $ \Big( \sum_{k = 1}^{D} \hat{\*H}_{\bm\beta}^\sk\Big)^\I \Big( \sum_{k = 1}^{D} \hat{\*H}_{\bm\beta}^\sk  \hat{\bm\beta}^\sk  \Big)$}  &  \multicolumn{1}{p{6.6cm}}{\raggedright $ \Big( \sum_{k = 1}^{D} \hat{\*H}_{\bm\beta}^{\pad,\sk}\Big)^\I \Big( \sum_{k = 1}^{D} \hat{\*H}_{\bm\beta}^{\pad,\sk}  \hat{\bm\beta}^{\pad,\sk} \Big)$ }   \\ \midrule
			\multicolumn{1}{p{3.cm}|}{\raggedright Variance $\*V_{\bm\beta}$ federation} 
			& \multicolumn{1}{p{5.5cm}|}{\raggedright Sample size weighting  $\hat{\*A}^{(k)}_{\bm\beta}$ and $\hat{\*B}^{(k)}_{\bm\beta}$ in $\*V_{\bm\beta} = \*A_{\bm\beta}^\I \*B_{\bm\beta} \*A_{\bm\beta}^\I$}  & \multicolumn{1}{p{5.5cm}}{\raggedright Sample size weighting  $\hat{\*A}^{\pad,(k)}_{\bm\beta}$ and $\hat{\*B}^{\pad,(k)}_{\bm\beta}$ in $\*V_{\bm\beta} = \*A_{\bm\beta}^\I \*B_{\bm\beta} \*A_{\bm\beta}^\I$  }  \\
			\midrule
            \multicolumn{1}{p{3.cm}|}{\raggedright Asymptotic results} & \multicolumn{2}{m{7.5cm}}{\raggedright Theorem  \ref{thm:pool-mle}}   \\
			\bottomrule
	\end{tabular}}
	\bnotetab{
    This table also holds for the propensity model.
	The second row correspond to Condition \ref{cond:stable-outcome}.  $\hat{\*H}_{\bm\beta}^\sk $ denotes the estimated Hessian. $\*A_{\bm\beta}$ and $\*B_{\bm\beta}$ are defined in Table \ref{tab:def-matrices}. $\hat{\*H}_{\bm\beta}^\sk$ increases with sample size $n_k$, while $\*A_{\bm\beta}$ and $\*B_{\bm\beta}$ do not. For a generic vector or matrix $\*x$, $\*x^{\pad}$ denotes $\*x$ padded with zeros.
	}
	
	\label{tab:mle}
\end{table}

\begin{proposition}\label{proposition:adjust-data-efficiency}
	Suppose $Y_i$ follows a generalized linear model that is stable across data sets (Condition \ref{cond:stable-outcome} holds). If we use unrestricted federated MLE with a flexible outcome model specification on the combined data  (i.e., $\bm\beta^\concat$ has a higher dimension than the most parsimonious specification), then we get a weakly less efficient estimate of $\bm\beta_{\mathrm{s}}$ than that from restricted federated MLE.
\end{proposition}

\subsection{Federated IPW-MLE}\label{subsec:federate-ht}

The target parameter of our federated IPW-MLE is $\bm{\beta}^\concat$ in the outcome model on the combined data. As IPW-MLE uses the propensity scores, we need to account for whether the propensity scores are known or estimated. If they are estimated, then our federated IPW-MLE also estimates and federates the propensity models.

\subsubsection{Restricted Federated IPW-MLE for Stable Models (Conditions \ref{cond:stable-propensity} and \ref{cond:stable-outcome} Hold)}\label{subsec:pool-stable-ht}

{\blue 
Let $\hat{\bm{\beta}}^\pool_\HT$ be our restricted federated point estimator for $\bm\beta^\concat$ obtained via a three-step procedure. First, if the propensity scores are unknown, we use restricted MLE to estimate the parameters in the propensity model on the combined data and obtain the federated propensity scores; otherwise, skip this step. Second, we use IPW-MLE with federated propensity scores to estimate $\bm\beta^\sk$ on each data set $k$. Third, we combine estimated $\bm\beta^\sk$ by Hessian weighting to obtain $\hat{\bm{\beta}}^\pool_\HT$.
Similar to federated MLE, this federated point estimator is designed to satisfy the first-order condition of IPW-MLE. 

The federated variance estimator of IPW-MLE is designed based on the variance formula of IPW-MLE in Lemma \ref{lemma:expression-ht-var} in Section \ref{subsec:ht-results} for a single data set. For every term in the variance formula, we estimate it on each data set. We combine the estimated terms across data sets by sample size weighting, and plug the sample size weighted terms into the variance formula to obtain the federated variance. The procedure is conceptually similar to that for MLE, but operates on a different variance formula. See Table \ref{tab:ht} for more details.

}

\subsubsection{Unrestricted Federated IPW-MLE for Unstable Models (Condition \ref{cond:stable-propensity} or \ref{cond:stable-outcome} is Violated)}\label{subsec:pool-ht-model-shift}

Similar to unrestricted federated MLE, our unrestricted federated IPW-MLE only federates shared parameters in the propensity and outcome models, and leaves the dataset-specific parameters as they are in federation. We first pad the parameters and matrices on each data set with zeros to match the dimensionality of the corresponding parameters and matrices on the combined data. Then we apply restricted federated IPW-MLE to the zero-padded parameters and matrices to obtain point and variance estimates of the target parameter.

\begin{figure}[t]
	\centering
	\tcapfig{Flowchart for Federated IPW-MLE}
	\includegraphics[width=.95\linewidth]{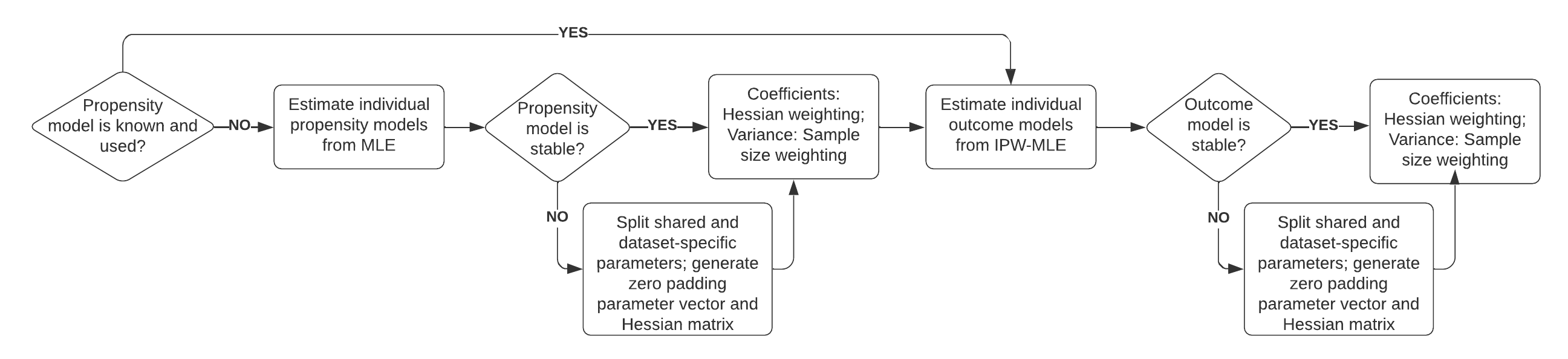}
	\label{flowchart:ht-unstable}
	\bnotefig{See Section \ref{subsec:practical-guidance} for practical guidance on determining whether the propensity/outcome model is stable.}
\end{figure}

\subsection{Federated AIPW Estimator}\label{subsec:federate-aipw}
{\blue

Our federated AIPW estimates ATE or ATT on the combined data. 
The illustration of federated AIPW uses ATE as an example. The federation of ATT is conceptually the same. 

\begin{landscape}
	\begin{table}
			\vspace{-1.5cm}
		\centering
		\tcaptab{Federated Inverse Propensity-Weighted Maximum Likelihood Estimator}
		{\renewcommand\arraystretch{1.25}
			\footnotesize
			\begin{tabular}{l|l|l|l|l}\toprule
				\multicolumn{1}{p{3.cm}|}{\raggedright Description} & \multicolumn{1}{p{5.cm}|}{\raggedright Assume Stable Known Propensity and Stable Outcome Model (IPW-MLE \#1)}  &  \multicolumn{1}{p{7.cm}|}{\raggedright Assume Stable Misspecified Propensity and Stable Outcome Model (IPW-MLE \#2)} &  \multicolumn{1}{p{5.cm}}{\raggedright Assume Unstable Propensity or Unstable Outcome Model (IPW-MLE \#3)}    \\  \midrule
				\multicolumn{1}{p{3.cm}|}{\raggedright Stable propensity model} & \multicolumn{1}{p{3.cm}|}{\raggedright yes}  &  \multicolumn{1}{p{3.cm}|}{\raggedright yes}  &  \multicolumn{1}{p{3.cm}}{\raggedright yes or no}    \\ \midrule
				\multicolumn{1}{p{3.cm}|}{\raggedright Stable outcome model} & \multicolumn{1}{p{3.cm}|}{\raggedright yes}  &  \multicolumn{1}{p{3.cm}|}{\raggedright yes}  &  \multicolumn{1}{p{3.cm}}{\raggedright yes or no}   \\ \midrule 
	\multicolumn{1}{p{3.cm}|}{\raggedright Parameter $\bm\beta$ federation} &  \multicolumn{1}{p{5.cm}|}{\raggedright (1) Estimate  ${\bm\beta}^\sk$ using  ${\bm\gamma}_0$; (2) Federate $\hat{\bm\beta}^\sk$ by Hessian weighting} & \multicolumn{1}{p{7.cm}|}{\raggedright (1) Federate  $\hat{\bm\gamma}^\sk $ by Hessian weighting; (2) Estimate  ${\bm\beta}^\sk$ using  $\hat{\bm\gamma}^\pool$; (3) Federate $\hat{\bm\beta}^\sk$ by Hessian weighting} & \multicolumn{1}{p{5.cm}}{\raggedright Same federation procedure, but with $\hat{\bm\gamma}^{\pad,\sk}$ and $\hat{\*H}_{\bm\gamma}^{\pad,\sk}$ if propensity models are unstable and estimated, and with $\hat{\bm\beta}^{\pad,\sk}$ and $\hat{\*H}_{\bm\beta}^{\pad,\sk}$ if outcomes models are unstable}  \\ \midrule
	\multicolumn{1}{p{3.cm}|}{\raggedright Variance $\*V_{\bm\beta}$ federation} & \multicolumn{1}{p{5.cm}|}{\raggedright  $\*V_{\bm\beta} = \*A_{\betavarpi}^\I \*D_{\betavarpi} \*A_{\betavarpi}^\I$ \\  (1) Estimate $\*A_{\betavarpi}^\sk$, $\*D_{\betavarpi}^\sk$ using $\hat{\bm\beta}^\pool$; (2) Federate $\hat{\*A}_{\betavarpi}^\sk$ and $\hat{\*D}_{\betavarpi}^\sk$ by sample size weighting}  &  \multicolumn{1}{p{7.cm}|}{\raggedright  $\*V_{\bm\beta} = \*A_{\betavarpi}^\I (\*D_{\betavarpi} - \*M_{\bm\beta, \varpi,\bm\gamma}) \*A_{\betavarpi}^\I$,  $\*M_{\bm\beta, \varpi,\bm\gamma} = \*C_{\betavarpi} \*V_{\bm\gamma} \*C_{\betavarpi}^\T$ for ATE weighting; $\*M_{\bm\beta, \varpi,\bm\gamma} =\*C_{\bm\beta, \varpi,1} \*V_{\bm\gamma}  \*C_{\bm\beta, \varpi,2}^\T +   \*C_{\bm\beta, \varpi,2} \*V_{\bm\gamma}  \*C_{\bm\beta, \varpi,1}^\T - \*C_{\bm\beta, \varpi,2} \*V_{\bm\gamma}  \*C_{\bm\beta, \varpi,2}^\T $  for ATT weighting; $\*V_{\bm\gamma} = \*A_{\bm\gamma}^\I \*B_{\bm\gamma} \*A_{\bm\gamma}^\I$. \\  (1) Estimate $\*A_{\betavarpi}^\sk$, $\*C_{\betavarpi}^\sk$, $\*D_{\betavarpi}^\sk$, $\*A^\sk_{\bm\gamma}$, and $ \*B^\sk_{\bm\gamma}$ using $\hat{\bm\gamma}^\pool$ and $\hat{\bm\beta}^\pool$; (2) Federate $\hat{\*A}_{\betavarpi}^\sk$, $\hat{\*C}_{\betavarpi}^\sk$, $\hat{\*D}_{\betavarpi}^\sk$, $\hat{\*A}^\sk_{\bm\gamma}$, and $ \hat{\*B}^\sk_{\bm\gamma}$ by sample size weighting }  & \multicolumn{1}{p{5.cm}}{\raggedright  Same federation procedure, but with $\hat{\bm\gamma}^{\pad,\sk}$, $\hat{\*A}_{\bm\gamma}^{\pad,\sk}$, $\hat{\*C}^{\pad,\sk}_{\betavarpi}$ and $\hat{\*B}_{\bm\gamma}^{\pad,\sk}$ if propensity models are unstable and estimated, and with $\hat{\bm\beta}^{\pad,\sk}$, $\hat{\*A}^{\pad,\sk}_{\betavarpi}$, $\hat{\*D}^{\pad,\sk}_{\betavarpi}$, and $\hat{\*C}^{\pad,\sk}_{\betavarpi}$ if outcomes models are unstable}   \\ 	\hline
		\multicolumn{1}{p{3cm}|}{\raggedright Asymptotic results} &   \multicolumn{3}{m{10.cm}}{Theorem \ref{theorem:ht-est-prop}}   \\ \bottomrule
		\end{tabular}}
		\bnotetab{
		The second and third rows correspond to Conditions \ref{cond:stable-propensity} and \ref{cond:stable-outcome}. ``yes or no" means that the solution does not vary with whether the condition is satisfied or not. The definitions of $\*A_{\betavarpi}, \*D_{\betavarpi},\*C_{\betavarpi},\*C_{\betavarpi,1},\*C_{\betavarpi,2},  \*A_{\bm\gamma}$, and $ \*B_{\bm\gamma}$ can be found in Table \ref{tab:def-matrices}. When the propensity model is estimated (Condition \ref{cond:known-propensity} is violated), the coefficient federation procedure is the same for all scenarios, but is simplified when the true propensity is used (Condition \ref{cond:known-propensity} holds). The variance federation procedure varies with whether the true propensity is used and whether ATE or ATT weighting is used. The definitions of ATE and ATT weighting can be found in Section \ref{subsec:ht-estimation}. For a generic vector or matrix $\*x$, $\*x^{\pad}$ denotes $\*x$ padded with zeros.
		}
		\label{tab:ht}
	\end{table}
\end{landscape}


\subsubsection{Restricted AIPW Estimator for Stable Models and Stable Covariate Distributions (Conditions \ref{cond:stable-propensity}, \ref{cond:stable-outcome} and \ref{cond:stable-covariate} Hold)}\label{subsec:pool-stable-aipw}
As the AIPW estimator uses both outcome and propensity models, we need to federate both propensity and outcome models. When covariate distributions, propensity models, and outcome models are stable, we propose to use the restricted federated AIPW, which has three steps. First, we use federated MLE to obtain a federated propensity model and a federated outcome model.\footnote{When the true propensity model is known and used, we do not need to federate the individual propensity models.} Second, we use AIPW with the federated propensity and outcome models to estimate ATE on each data set. Finally, we obtain the federated ATE by inverse variance weighting the estimated ATE on each data set, as in formula \eqref{eqn:inverse-variance-weighting}.

\begin{figure}[t!]
	\centering
	\caption{Flowchart for Federated AIPW}
	\includegraphics[width=1.\linewidth]{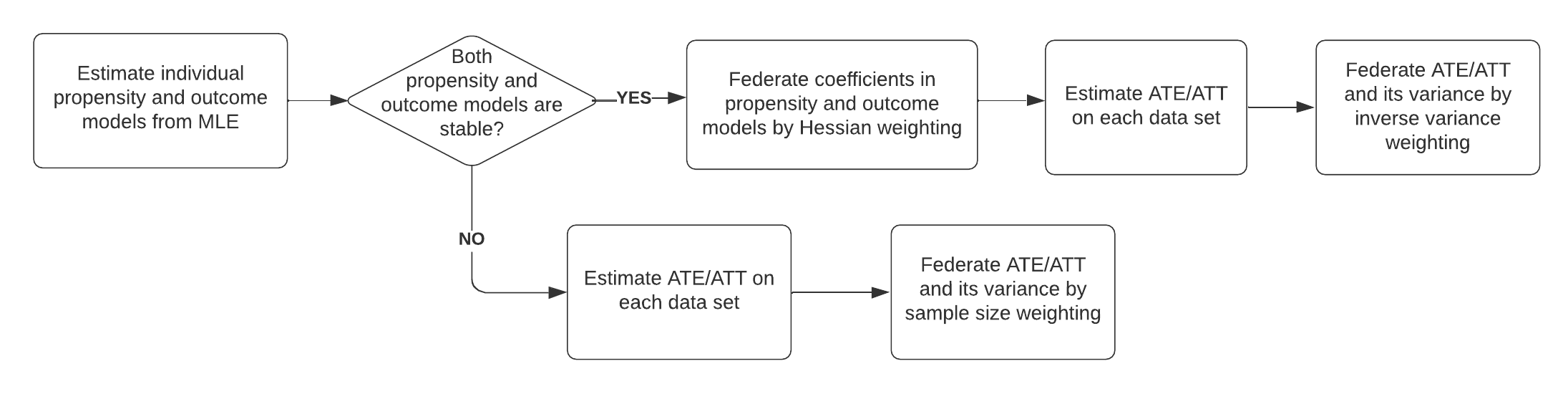}
	\bnotefig{See Section \ref{subsec:practical-guidance} for practical guidance on determining whether propensity and outcome models are stable.}
	\label{flowchart:aipw}
\end{figure}

To obtain the federated variance, we first estimate the variance of the estimated ATE on each data set, and then use inverse variance weighting to combine the estimated variances on all data sets together, as in formula \eqref{eqn:inverse-variance-weighting-variance}.

Note that, under stable covariate distributions and stable propensity and outcome models, ATE and asymptotic variance of ATE are the same for all data sets. In this case, we can apply any weighting scheme to combine the estimated ATE together. We choose IVW because it has the smallest variance among all weighting schemes, as shown in Appendix \ref{subsec:ivw-min-var}.

\subsubsection{Unrestricted AIPW Estimator for Unstable Models or Unstable Covariate Distributions (Either Condition \ref{cond:stable-propensity}, \ref{cond:stable-outcome} or \ref{cond:stable-covariate} is Violated)}\label{subsec:pool-aipw-model-shift}
When either propensity model, outcome model, or covariate distribution is unstable, ATE may not be the same across data sets. For this case, we suggest using the unrestricted federated AIPW. For this unrestricted estimator, we first estimate ATE and its asymptotic variance on each data set and then use sample size weighting to combine the estimated ATE and variances together:\footnote{\cmtfinal{The unrestricted AIPW is equivalent to the AIPW in \eqref{eqn:score} with the score on combined data estimated by $\hat{\phi}(\*X^\sk_i, W^\sk_i, Y^\sk_i ) \coloneqq \prod_{j=1}^K [\hat{\phi}^\sj(\*X^\sj_i, W^\sj_i, Y^\sj_i )]^{\bm{1}(j = k)} $ and $\hat{\phi}^\sj(\*X^\sj_i, W^\sj_i, Y^\sj_i )$ estimated using the estimated outcome and propensity models on data set $j$. } }
\begin{align}
\label{eqn:pool-aipw-var-model-shift}
\hat{\tau}^\pool_\aipw =  \sum_{k = 1}^{D} \frac{n_k}{n_\npool}  \hat{\tau}_\aipw^\sk \qquad \hat{\*V}_{\tau}^\pool  = \sum_{k = 1}^{D} \frac{n_k}{n_\npool} \hat{\*V}^\sk_\tau,
\end{align}
where $\hat{\tau}_\aipw^\sk$ is the estimated ATE on data set $k$, and $\hat{\*V}^\sk_\tau$ is the estimated variance of $\hat{\tau}_\aipw^\sk$.

This federated AIPW estimator is quite general. First, it is robust to propensity or outcome model misspecification. Second, it allows the propensity or/and outcome models to vary arbitrarily across data sets. Third, it allows $\hat{\tau}_\aipw^\sk$ to be estimated from flexible machine learning methods, such as random forests \citep{wager2018estimation}, as we do not need an approach to federate estimated propensity and outcome models across data sets. The tradeoff is that the unrestricted estimator is less efficient than the restricted estimator, under stable covariance distribution and stable propensity and outcome models. 

\begin{table}[t!]
	\centering
	\tcaptab{Federated AIPW Estimator}
	{\renewcommand\arraystretch{1.25}
		\footnotesize
		\begin{tabular}{l|l|l}\toprule
			\multicolumn{1}{p{2cm}|}{\raggedright Description}   & \multicolumn{1}{p{5.cm}|}{\raggedright Assume Stable Propensity and Stable Outcome Model (AIPW \#1)} &   \multicolumn{1}{p{5.cm}}{\raggedright Assume Unstable Propensity or Unstable Outcome Model (AIPW \#2)}     \\ 
			\midrule
			
			\multicolumn{1}{p{4.cm}|}{\raggedright Stable propensity model} & \multicolumn{1}{p{3.cm}|}{\raggedright yes} &  \multicolumn{1}{p{3.cm}}{\raggedright yes or no}    \\ \midrule
			\multicolumn{1}{p{4.cm}|}{\raggedright Stable outcome model} & \multicolumn{1}{p{3.cm}|}{\raggedright yes} &  \multicolumn{1}{p{3.cm}}{\raggedright yes or no}   \\ \midrule
   \multicolumn{1}{p{4.8cm}|}{\raggedright Stable covariate distribution} & \multicolumn{1}{p{3.cm}|}{\raggedright yes} &  \multicolumn{1}{p{3.cm}}{\raggedright yes or no}   \\ 
   \midrule 
			\multicolumn{1}{p{4.2cm}|}{\raggedright ATE or ATT $\tau$ federation}   & \multicolumn{1}{p{5.cm}|}{\raggedright (1) Federate $\hat{\bm\beta}^\sk $ (and $\hat{\bm\gamma}^\sk $ if necessary) by Hessian weighting; (2) Estimate  ${\tau}^\sk$ using  $\hat{\bm\beta}^\pool$ and $\hat{\bm\gamma}^\pool$ (or $\bm\gamma^\sk_0$ if known); (3) Federate $\hat{\tau}^\sk$ by inverse variance weighting.} &   \multicolumn{1}{p{5.cm}}{\raggedright (1) Estimate  ${\tau}^\sk$ using  $\hat{\bm\beta}^\sk$ and $\hat{\bm\gamma}^\sk$ (or $\bm\gamma^\sk_0$ if known); (2) Federate $\hat{\tau}^\sk$ by sample size weighting.}     \\  \midrule
			\multicolumn{1}{p{4.2cm}|}{\raggedright Variance $\*V_\tau$ federation}   & \multicolumn{1}{p{5.cm}|}{\raggedright Inverse variance weighting} &   \multicolumn{1}{p{5.cm}}{\raggedright Sample size weighting}     \\ 
			\midrule
			\multicolumn{1}{p{2cm}|}{\raggedright Results} &   \multicolumn{2}{m{6cm}}{Theorem \ref{theorem:pool-aipw} } \\ \bottomrule
	\end{tabular}}
	\bnotetab{The second to fourth rows correspond to Conditions \ref{cond:stable-propensity}, \ref{cond:stable-outcome} and \ref{cond:stable-covariate}. 
	}
	\label{tab:aipw}
\end{table}

 }

{\blue 
\subsection{Practical Guidance}\label{subsec:practical-guidance}

In this subsection, we suggest some diagnostic tests that may help practitioners choose between restricted and unrestricted methods and determine the set of shared parameters. For ease of discussion, our empirical application is used as a running example with a generalized linear model (GLM) specification for outcomes. 

First, we can examine whether the link function of the GLM is the same across data sets. If not  (for example, one is linear and the other one is logit), then it is natural to choose the unrestricted method without shared parameters.

Suppose the link function of the GLM is the same across data sets. Second, we can examine whether there exist some covariates that are unique to a data set. If yes, then it is natural to specify the parameters of these covariates as unstable parameters. For example, Optum covers more years than MarketScan, and the outcome model incorporates several year dummies that are unique to Optum. The coefficients of these dummies are unstable parameters.

Third, we can run hypothesis tests for whether the parameter values are the same across data sets. Suppose we would like to test whether the $p$-dimensional parameters on MarketScan  $\bm{\beta}_\ms$ and on Optum $\bm{\beta}_\optum$ are the same, i.e.,
\begin{equation}\label{eqn:hypothesis-test}
    \mathcal{H}_0: \bm{\beta}_\ms = \bm{\beta}_\optum \qquad \mathcal{H}_1: \bm{\beta}_\ms \neq \bm{\beta}_\optum.
\end{equation}
We can construct the modified Hotelling's T-square test statistic,
\[T^2 = \left( \hat{\bm{\beta}}_\ms - \hat{\bm{\beta}}_\optum \right)^\T \left(\frac{n_\ms}{(n_\ms + n_\optum)^2} \hat{\*{V}}_\ms + \frac{n_\optum}{(n_\ms + n_\optum)^2} \hat{\*{V}}_\optum \right)^{-1} \left( \hat{\bm{\beta}}_\ms - \hat{\bm{\beta}}_\optum \right), \]
where $\hat{\bm{\beta}}_\ms$ and $\hat{\bm{\beta}}_\optum$ are estimated parameters on MarketScan and on Optum, with estimated asymptotic variances $\hat{\*{V}}_\ms$ and $\hat{\*{V}}_\optum$. 

$T^2$ is approximately chi-square distributed with $p$ degree of freedom when both $\hat{\bm{\beta}}_\ms$ and $\hat{\bm{\beta}}_\optum$ are asymptotically normal.
If we do not reject the null, then we can treat $\bm{\beta}_\ms$ and $\bm{\beta}_\optum$ as stable parameters. Otherwise, we have two options. First, we can treat every entry in $\bm{\beta}_\ms$ and $\bm{\beta}_\optum$ as an unstable parameter. Second, we can test again on a subset of $\bm{\beta}_\ms$ and $\bm{\beta}_\optum$ using a similar procedure to determine whether this subset of parameters are stable. We may want to choose the second option when we want to specify as many stable parameters as possible for efficiency consideration (following the intuition in Proposition \ref{proposition:adjust-data-efficiency}). 

Last but not least, we suggest running a data-driven simulation study using real data to compare various federated methods with different specifications of shared and dataset-specific parameters. See Section \ref{subsec:empirical-sampling} for an example. In this simulation study, we draw patient records from one data set to construct subsamples that mimic the demographics of the multiple data sets we seek to federate. Then we federate subsamples using various federated methods. The benchmarks are the results from the combined data, as in this case, combining patient records across subsamples is permissible, given that they are sampled from one data set. Finally, we choose the federated method that is closest to the benchmarks.

}

\section{Asymptotic Results}\label{sec:results}
In this section, we show the asymptotic results of our federated MLE, IPW-MLE and AIPW. The federated point estimators have the same asymptotic distributions as their corresponding estimators using the combined, individual-level data. The federated variance estimators are consistent, which allows us to construct valid confidence intervals of target parameters. Appendix \ref{sec:simulation} demonstrate the finite-sample properties of the asymptotic results.  Appendix \ref{sec:proof} collects all the proofs. 

{\blue To show the asymptotic results, we impose standard regularity assumptions on $f(y \mid \*x, w, \bm\beta)$ and $e(\*x,\bm\gamma)$, similar to  \cite{white1982maximum} and \cite{wooldridge2007inverse}, among others. To conserve space, the regularity assumptions are deferred to Assumption \ref{assumption:parametric} in Appendix \ref{subsec:regularity-conditions}. Let $\bm{\gamma}^{\concat \ast}$  and $\bm{\beta}^{\concat \ast}$ be the solutions that maximize the expected log-likelihood $\+E[\log  e^\concat(\*x, \bm\gamma)]$ and $\+E[\log f^\concat(y \mid \*x, w,\bm\beta)]$. The solutions may or may not equal the true parameter values $\bm{\gamma}^{\concat}_0$ and $\bm{\beta}^{\concat}_0$, depending on whether the propensity and outcome models are correctly specified. See Appendix \ref{subsec:regularity-conditions} for more discussion. In this section, we show that our federated MLE or IPW-MLE can consistently estimate $\bm{\beta}^{\concat \ast}$ (and $\bm{\gamma}^{\concat \ast}$).}

\subsection{Federated MLE}\label{subsec:mle-results}

We illustrate the asymptotic results of federated MLE using the estimated parameters in the outcome model, but the asymptotic results also apply to estimated parameters in the propensity model.
The following theorem shows that in federated MLE, the federated point estimator of target parameters, denoted by $\hat{\bm\beta}^\pool_\mle$, have the same asymptotic distribution as MLE on the combined, individual-level data. In addition, the federated variance estimator, denoted by $\hat{\*V}_{\bm\beta}^\pool $, is consistent. 

\begin{theorem}[Federated MLE]\label{thm:pool-mle}
	Suppose Assumption \ref{assumption:parametric}.1 holds. If Condition \ref{cond:stable-outcome} holds, we use restricted federated MLE in Section \ref{subsec:pool-stable-mle}; otherwise, we use unrestricted federated MLE in Section \ref{subsec:pool-mle-model-shift}. 
 Suppose the information matrices satisfy $\norm{  \mathcal{I}^\concat(\bm\beta)^\I \mathcal{I}^\sk(\bm\beta)  }_2 \leq M$ for some $M < \infty$ and for all $k$. 
 As $n_1, \cdots, n_D \rightarrow \infty$, we have 
	\begin{align}
	n_\npool^{1/2} (\hat{\*V}_{\bm\beta}^\concat)^{-1/2} (\hat{\bm\beta}^\pool_\mle - \bm\beta^{\concat \ast}) \xrightarrow{d}& \mathcal{N} (0, \*I_d), \label{eqn:theorem-pool-mle}
	\end{align}
	where $d$ is the dimension of $\bm{\beta}^{\concat \ast}$.

	If we replace $\hat{\*V}_{\bm\beta}^{\concat}$  by $\hat{\*V}_{\bm\beta}^{\pool}$  and/or replace $\hat{\bm\beta}^\pool_\mle$ by  $\hat{\bm\beta}^\concat_\mle$,  then \eqref{eqn:theorem-pool-mle} continues to hold.
	
\end{theorem}

The federated point estimator $\hat{\bm\beta}^\pool_\mle $ converges at the optimal rate $n_\npool^{-1/2}$ convergence rate. The convergence rate is therefore improved via federation, as compared to the rate $n_k^{-1/2}$ of $\hat{\bm\beta}^\sk_\mle$ for any $k$. 
%
Theorem \ref{thm:pool-mle} holds regardless of whether the outcome model is correctly specified or not. If the outcome model is correctly specified, $\hat{\bm\beta}^\pool_\mle$ is a consistent estimator of $\bm{\beta}^{\concat}_0$; otherwise, $\hat{\bm\beta}^\pool_\mle$ converges to the limit $\bm\beta^{\concat \ast}$ that generally differs from $\bm{\beta}^{\concat}_0$.

\begin{remark}\label{remark:stable-formula-for-unstable-model}
\normalfont
If outcome models are unstable, but we use restricted federated MLE in Section \ref{subsec:pool-stable-mle}, Proposition \ref{proposition:model-shift-mle-weighted} in Appendix \ref{sec:additional-asymptotic} shows that, under some special cases,
Theorem \ref{thm:pool-mle} continues to hold, but with a limit that potentially differs from $\bm\beta^{\concat \ast}$.
\end{remark}

\subsection{Federated IPW-MLE}\label{subsec:ht-results}
We start with a lemma that provides the asymptotic distribution of IPW-MLE on a single data set, on which the asymptotic results of federated IPW-MLE are built.

\begin{lemma}\label{lemma:expression-ht-var}
	Suppose Assumption \ref{assumption:parametric} holds and we estimate $e(\*X_i)$ from MLE. As $n \rightarrow \infty$,
	$	\hat{\bm\beta}_\HT$ estimated from  IPW-MLE is consistent and asymptotically normal, 
	\begin{align*}
	\sqrt{n} \big(\hat{\bm\beta}_\HT -  \bm\beta^\ast   \big) \xrightarrow{d}&\mathcal{N} \big(0,  \*V_{\bm\beta^\ast, \HT, \hat{e}}^\dagger  \big),
	\end{align*}
	where
	\begin{align}\label{eqn:expression-ht-var}
		 \*V_{\bm\beta^\ast, \HT, \hat{e}}^\dagger = \*A_{\bm\beta^\ast, \varpi}^\I \big(\*D_{\bm\beta^\ast, \varpi} -  \*M_{\bm\beta^\ast, \varpi,\bm\gamma^\ast}  \big) \*A_{\bm\beta^\ast, \varpi}^\I
	\end{align}
	with 
	\begin{align*}
	    \*M_{\bm\beta^\ast, \varpi,\bm\gamma^\ast} = \begin{cases} \*C_{\bm\beta^\ast, \varpi} \*V_{\bm\gamma^\ast}  \*C_{\bm\beta^\ast, \varpi}^\T  & \text{ ATE weighting} \\
	    \*C_{\bm\beta^\ast, \varpi,1} \*V_{\bm\gamma^\ast}  \*C_{\bm\beta^\ast, \varpi,2}^\T +   \*C_{\bm\beta^\ast, \varpi,2} \*V_{\bm\gamma^\ast}  \*C_{\bm\beta^\ast, \varpi,1}^\T - \*C_{\bm\beta^\ast, \varpi,2} \*V_{\bm\gamma^\ast}  \*C_{\bm\beta^\ast, \varpi,2}^\T  & \text{ ATT weighting}, 
	    \end{cases}
	\end{align*}
	$\*V_{\bm\gamma^\ast} = \*A_{\bm\gamma^\ast}^\I \*B_{\bm\gamma^\ast} \*A_{\bm\gamma^\ast}^\I $, and $\*A_{\bm\beta^\ast, \varpi}$ is matrix $\*A_{\betavarpi}$ evaluated at $\bm\beta^\ast$, with the definition of $\*A_{\betavarpi}$ provided in Table \ref{tab:def-matrices}. Other terms in formula \eqref{eqn:expression-ht-var} are defined similarly. 
	
	If IPW-MLE uses true propensities, then the asymptotic variance is simplified to
	\begin{align}\label{eqn:expression-ht-var-known-e}
	\*V_{\bm\beta^\ast, \HT}^\dagger   =  \*A_{\bm\beta^\ast, \varpi}^\I \*D_{\bm\beta^\ast, \varpi} \*A_{\bm\beta^\ast, \varpi}^\I. 
	\end{align}
\end{lemma}

Lemma \ref{lemma:expression-ht-var} coincides with the the results in \cite{wooldridge2002inverse,wooldridge2007inverse} for ATE weighting, and Lemma \ref{lemma:expression-ht-var} additionally provides the results for ATT weighting.

Note that the estimation error of the propensity model carries over to the asymptotic variance of IPW-MLE. This explains why our federated variance estimator in Section \ref{subsec:federate-ht} needs to vary with whether the true propensities are used. In addition, if federated IPW-MLE varies properly with whether propensity and/or outcome models are stable or not, then the federated point estimator, denoted by $\hat{\bm\beta}^\pool_\HT$, have the same asymptotic distribution as IPW-MLE on the combined, individual-level data. Moreover, the federated variance, denoted by $\hat{\*V}_{\bm\beta,\HT, \hat{e}}^{\pool, \dagger}$, is consistent when it is obtained based on the formulas in Lemma \ref{lemma:expression-ht-var}.

\begin{theorem}[Federated IPW-MLE]\label{theorem:ht-est-prop}
	Suppose Assumption \ref{assumption:parametric} holds. If Conditions \ref{cond:stable-propensity} and \ref{cond:stable-outcome} hold, we use restricted federated IPW-MLE in Section \ref{subsec:pool-stable-ht}; otherwise, we use unrestricted federated IPW-MLE in Section \ref{subsec:pool-ht-model-shift}. 
 Suppose $\norm{(\*A^\concat_{\bm\beta^\ast, \varpi})^\I  \*A^\sk_{\bm\beta^\ast, \varpi}}_2 \leq M$ and $\norm{(\*A^\concat_{\bm\gamma^\ast})^\I  \*A^\sk_{\bm\gamma^\ast}}_2 \leq M$ for some $M < \infty$ and for all $k$. 
	As $n_1, \cdots, n_D \rightarrow \infty$, we have 
	\begin{align}
	n_\npool^{1/2} (\hat{\*V}_{\bm\beta,\HT,\hat{e}}^{\concat, \dagger})^{-1/2} (\hat{\bm\beta}^\pool_\HT  - \bm\beta^{\concat \ast}) \xrightarrow{d}& \mathcal{N} (0, \*I_d), \label{eqn:theorem-pool-ht}
	\end{align}
	
	If we replace $\hat{\*V}_{\bm\beta,\HT,\hat{e}}^{\concat, \dagger}$  by $\hat{\*V}_{\bm\beta,\HT,\hat{e}}^{\pool, \dagger}$  and/or replace $\hat{\bm\beta}^\pool_\HT$ by  $\hat{\bm\beta}^\concat_\HT$, then \eqref{eqn:theorem-pool-ht} continues to hold. 
	If we use true propensities, the above statements continue to hold with $\hat{\*V}_{\bm\beta,\HT,\hat{e}}^{\concat, \dagger}$ replaced by the corresponding variance terms for the true propensities.
	
\end{theorem}

Federated IPW-MLE converges at the rate $n_\npool^{-1/2}$, which is faster than $n_k^{-1/2}$ on a single data set $k$. 
{\blue This theorem holds regardless of whether covariate distributions are stable or not, as long as the limiting objects, $\bm{\beta}^{\concat \ast}$ and $\*V^\ast_{\bm{\beta}^{\concat \ast},\HT}$, are well-defined on the combined data, though their definitions may vary with whether covariate distributions are stable.}

Moreover, Theorem \ref{theorem:ht-est-prop} holds regardless of whether we use the true or estimated propensities. In practice, even if we know the true propensities, it is better to use the estimated propensities for the efficiency consideration  (\cite{wooldridge2002inverse,hirano2003efficient} among others), as $ \*V_{\bm\beta^\ast, \HT}^\dagger - 	\*V_{\bm\beta^\ast, \HT, \hat{e}}^\dagger  $ is positive semidefinite from Lemma \ref{lemma:expression-ht-var} for ATE weighting. 
If the estimated propensities are used, we could still use the federated variance estimator for the true propensity case, which takes a simpler form, but overestimates the variance.

{\blue 

\subsection{Federated AIPW}\label{subsec:aipw-results}

The following theorem shows that our federated AIPW for ATE and ATT has the same asymptotic distribution as AIPW on the combined data. In addition, our federated variance estimators for AIPW are consistent.

\begin{theorem}\label{theorem:pool-aipw}
	Suppose either of the following cases holds: (a) the score $ \phi^\sk(\*x, w, y) $ is the same for all $k$, and we use the federation procedure in Section \ref{subsec:pool-stable-aipw}; or (b) $ \phi^\sk(\*x, w, y) $ varies with $k$, and we use the federation procedure in Section \ref{subsec:pool-aipw-model-shift}. Furthermore, suppose for any data set $k$, at least one condition holds: (a) $\mu^\sk_{(w)}(\*x)$ is correctly specified and consistently estimated for $w \in \{0,1\}$, or (b) $e^\sk(\*x)$ is correctly specified and consistently estimated. 
	As $n_1, \cdots, n_D \rightarrow \infty$, if the estimand is ATE, we have 
	\begin{align}
	n_\npool^{1/2} (\hat{\*V}^\concat_{\tau_\ate})^{-1/2} (\hat{\tau}^\pool_\ate - \tau^\concat_\ate)  \xrightarrow{d}& \mathcal{N}(0,1). \label{eqn:theorem-pool-aipw}
	\end{align}
	If we replace $\hat{\*V}^\concat_{\tau_\ate}$  by $\hat{\*V}^\pool_{\tau_\ate}$  and/or replace $\hat{\tau}^\pool_\ate$ by  $\hat{\tau}^\concat_\ate$, then \eqref{eqn:theorem-pool-aipw} continues to hold. If the estimand is ATT, \eqref{eqn:theorem-pool-aipw} continues to hold analogously for the federated estimator of ATT and the corresponding federated variance estimator. 
\end{theorem}

Analogous to federated MLE and IPW-MLE, federated AIPW achieves a faster convergence rate than AIPW on a single data set. The estimation efficiency of ATE and ATT can be improved through federation. In addition, federated AIPW achieves the semiparametric efficiency bound.

Note that if the propensity and outcome models are estimated from flexible machine learning methods, then we can use unrestricted federated AIPW to combine the estimated ATE or ATT on individual data sets together without combining individual propensity and outcome models. \cmtfinal{If individual propensity and outcome models can be estimated at rate $o(n_k^{-1/4})$, then the estimated ATE or ATT on individual data sets by using cross-fitting converges at the rate $n_k^{-1/2}$ and is asymptotically normal \citep{chernozhukov2017double}. We can then show that the federated ATE or ATT is asymptotically normal, and the federated variance estimator is consistent. This approach is asymptotically efficient. However, when propensity and outcome models are stable, the variance may be reduced in finite samples, by developing new approaches to federate flexible machine learning methods and using restricted federated AIPW.  }
}

\section{Empirical Studies Based on Medical Claims Data}\label{sec:empirical}
{\blue In this section, we further study the effect of alpha blockers on two distributed medical databases, MarketScan and Optum, introduced in Section \ref{sec:introduction}.\footnote{Our analysis builds on the studies by \cite{konig2020preventing,koenecke2020alpha, rose2021alpha, powell2021ten}, and \cite{thomsen2021association}.} We first evaluate various federation methods through a data-driven simulation study on one medical claims data, select the optimal federated method, and then apply this method to federate MarketScan and Optum.\footnote{Note that our findings reproduce similar results to \cite{koenecke2020alpha}, validating the prior result suggesting that alpha blockers are effective in reducing ventilation and death in ARD and pneumonia patients; however, the confidence levels are narrower because our federated methods presented here are improved from those used in \cite{koenecke2020alpha}. \cite{koenecke2020alpha} only use the treatment coefficient and variance in federation, whereas here, we use the full variance-covariance matrix from all covariates. Our approach leverages the stable part of the model across two data sets, which could improve the estimation precision of the treatment coefficient.} }

\subsection{Simulation on One Medical Claims Data Set}\label{subsec:empirical-sampling}

{\blue In the data-driven simulation study, we first construct subsamples from one cohort to reflect patient demographics from MarketScan and Optum. Next we compare estimates from various federated methods with those on the combined data. We seek to evaluate how well the federated methods recover the known result from the combined data in a setting where combining data is permissible. We can then select the most effective federated methods and apply these methods to combine the summary-level information from MarketScan and Optum in Section \ref{subsec:empirical-federate-two}. 
}

We start by presenting our approach to simulate subsamples from one patient cohort in Section \ref{subsec:sampling}.  Then we list benchmark methods and tested federated methods in Section \ref{subsec:estimation-benchmark}. We compare the results from federated methods against benchmarks in Section \ref{subsec:results}.

\subsubsection{Sampling Schemes for Subsamples}\label{subsec:sampling}
We draw two subsamples, denoted as $\mathcal{S}_1$ and $\mathcal{S}_2$, based on patient records from one cohort in a database (denoted as $\mathcal{C}$), to mimic the demographics of the distributed databases that we aim to federate. Our simulation design is based on the observation that cohorts in MarketScan include patients younger than age 65 from 2009 to 2015, while cohorts in Optum include patients up to age 85 from 2005 to 2019, with a majority to be over age 65.

To simulate subsamples, we first partition one cohort $\mathcal{C}$ into four disjoint sub-cohorts, denoted as $\mathcal{C}_1$, $\mathcal{C}_2$, $\mathcal{C}_3$, and $\mathcal{C}_4$, by age and fiscal year. $\mathcal{C}_1$ include patients younger than the median age of $\mathcal{C}$ up to year 2012; $\mathcal{C}_2$ include patients younger than the median age after 2012; $\mathcal{C}_3$ include patients older than the median age up to 2012; $\mathcal{C}_4$ include patients older than the median age after 2012. Next we simulate $\mathcal{S}_1$ and $\mathcal{S}_2$.  $\mathcal{S}_1$ mimics the demographics of MarketScan, with 70\% and 30\% sampled from $\mathcal{C}_1$ and $\mathcal{C}_3$, respectively, with replacement. $\mathcal{S}_2$ mimics the demographics of Optum, with 10\%, 10\%, 10\%, and 70\% are sampled from   $\mathcal{C}_1$, $\mathcal{C}_2$, $\mathcal{C}_3$, and $\mathcal{C}_4$, respectively, with replacement. 

As a robustness check,  we consider other approaches in Appendix \ref{sec:additional-empirical} to construct subsamples, including varying the sampling ratios from different sub-cohorts, varying subsample sizes, and varying the number of subsamples.

\subsubsection{Estimation and Benchmarks}\label{subsec:estimation-benchmark}

We consider two benchmark estimators and three federated estimators.

\paragraph{Restricted Benchmarks} Parameters in the propensity and outcome models are assumed to be stable across subsamples. On the combined data, we specify restricted propensity and outcome models as
\begin{equation}\label{eqn:empirical-stable-specification-main}
    \begin{aligned}
        \frac{\pr(W_i=1 \mid \*X_i)}{\pr(W_i=0 \mid \*X_i)} =& \*X_i^\T  \bm\gamma \\
    \frac{\pr(Y_i=1 \mid \*X_i, W_i)}{\pr(Y_i=0 \mid \*X_i, W_i)} =&  W_i \beta_{w} + \*X_i^\T \bm\beta_{\*X} \,,
    \end{aligned}
\end{equation}
where outcome $Y_i$ is binary indicating whether a patient received mechanical ventilation and then had an in-hospital death ($Y_i = 1$) or not ($Y_i = 0$), treatment $W_i$ is binary indicating whether a patient is exposed to alpha blockers ($W_i = 1$) or not ($W_i = 0$), and  $\*X_i$ consists of age, fiscal year dummies, and health-related confounders.\footnote{See Appendix \ref{subsec:study-definition} for the full list of confounders.}

The restricted benchmarks are the estimate of $\beta_{w}$ in \eqref{eqn:empirical-stable-specification-main}, denoted as $\hat\beta^{\bm{\s}}_{w,\BM}$, and its estimated variance, denoted as $\hat{V}^{\bm{\s}}_{\beta_w,\BM}$, from the combined data. 

\paragraph{Unrestricted Benchmarks} Parameters in the propensity and outcome models can be unstable across subsamples. On the combined data, we specify a flexible functional form for the propensity and outcome models\footnote{\eqref{eqn:empirical-unstable-specification-main} can be easily generalized to the case with more than two subsamples.}
\begin{equation}\label{eqn:empirical-unstable-specification-main}
        \begin{aligned}
        \frac{\pr(W_i=1 \mid \*X_i)}{\pr(W_i=0 \mid \*X_i)} =& \*X_{i,\mathrm{s}}^\T  \bm\gamma_{\mathrm{s}} + \*X_{i,\mathrm{uns}}^\T  \bigg( \bm{1}(A_i = 1) \bm\gamma^{(1)}_{\mathrm{uns}} + \bm{1}(A_i = 2) \bm\gamma^{(2)}_{\mathrm{uns}}  \bigg) \\
    \frac{\pr(Y_i=1 \mid \*X_i, W_i)}{\pr(Y_i=0 \mid \*X_i, W_i)} =&  W_i \beta_{w} + \*X_{i,\mathrm{s}}^\T \bm\beta_{\*X,\mathrm{s}} + \*X_{i,\mathrm{uns}}^\T \bigg( \bm{1}(A_i = 1) \bm\beta^{(1)}_{\*X,\mathrm{uns}}  + \bm{1}(A_i = 2) \bm\beta^{(2)}_{\*X,\mathrm{uns}} \bigg)\, ,
    \end{aligned}
\end{equation}
where $A_i \in \{1, 2\}$ indicates whether the patient record belongs to $\mathcal{S}_1$ or $\mathcal{S}_2$.\footnote{Note that $\mathcal{S}_1$ has patient records up to 2012, while $\mathcal{S}_2$ has patient records for all years. The coefficients of year dummies after 2012 are treated as unstable parameters in both restricted and unrestricted benchmarks.} Parameters are partitioned into stable parameters ($\bm\gamma_{\mathrm{s}}$, $\beta_{w} $ and $\bm\beta_{\*X,\mathrm{s}}$) and unstable parameters ($\bm\gamma^{(a)}_{\mathrm{uns}}$ and $\bm\beta^{(a)}_{\*X,\mathrm{uns}}$ for $a \in \{1,2\}$). The unstable variables include the coefficients of age 
confounders and year dummies unique to $\mathcal{S}_2$,\footnote{Age confounders include age, age-squared, and age-cubed.} which is motivated by the observation that age coefficient has opposite signs on MarketScan and Optum (see Figure \ref{fig:coef-age} in Appendix \ref{sec:additional-empirical}), and Optum covers more years. Note that $\beta_{w}$ is stable across subsamples, which can be interpreted as the average treatment coefficient across subsamples. 

The unrestricted benchmarks are the estimates of $\beta_{w}$  in \eqref{eqn:empirical-unstable-specification-main}, denoted as $\hat\beta^{\bm{\uns}}_{w,\BM}$, and its estimated variance, denoted as $\hat{V}^{\bm{\uns}}_{\beta_w,\BM}$, from the combined data.

\paragraph{Restricted Federated Estimators} Under the restricted model specification \eqref{eqn:empirical-stable-specification-main}, we use restricted federated IPW-MLE to estimate $\beta_{w}$ in \eqref{eqn:empirical-stable-specification-main} and its variance. Let $\hat\beta^{\bm{\s}.\pool}_{w,\HT}$ and $\hat{V}^{\bm{\s}.\pool}_{\beta_w,\HT}$ be the estimated coefficient and variance.

\paragraph{Unrestricted Federated Estimators} Under the flexible model specification \eqref{eqn:empirical-unstable-specification-main}, we use our unrestricted federated IPW-MLE to estimate $\beta_{w}$ in \eqref{eqn:empirical-unstable-specification-main} and its variance. Let $\hat\beta^{\bm{\uns}.\pool}_{w,\HT}$ and $\hat{V}^{\bm{\uns}.\pool}_{\beta_w,\HT}$ be the estimated coefficient and variance.

\paragraph{Inverse Variance Weighting (IVW)} Under the restricted model specification \eqref{eqn:empirical-stable-specification-main}, we use IVW to estimate $\beta_{0,w}$ in \eqref{eqn:empirical-stable-specification-main} and its variance. Let $\hat{\beta}_{w,\ivw}$ and $\hat{V}_{\beta_w,\ivw}$ be the estimated coefficient and variance.\footnote{IVW is appropriate when Conditions \ref{cond:stable-propensity}, \ref{cond:stable-outcome}, and \ref{cond:stable-covariate} hold. In this case, Hessians and other matrices in the asymptotic variance are asymptotically stable across data sets. Then we can show that our federated estimators in Section \ref{sec:estimation} are asymptotically the same as IVW.}

\begin{table}[t!]
	\centering
	\tcaptab{Comparison Between Restricted/Unrestricted Federated Estimators and IVW with Corresponding Restricted/Unrestricted Benchmarks}
	\vspace{0.5cm}
	\begin{subtable}[t]{.5\textwidth}
	\centering
	\caption{Restricted Benchmarks $\hat\beta^{\bm{\s}}_{w,\BM}$, $ \hat{V}^{\bm{\s}}_{w,\BM}$}\label{tab:fed-ht-ivw-res-main}
	{\footnotesize
	
	\begin{tabular}{l|r|r|r|r}
		\toprule
	&	\multicolumn{1}{c|}{$\hat\beta^{\bm{\s}}_{w,\BM}$}     & \multicolumn{1}{c|}{$\hat{\beta}_{w,\ivw}$} & \multicolumn{1}{c|}{$\hat\beta^{\bm{\s}.\pool}_{w,\HT}$} & \multicolumn{1}{c}{$\hat\beta^{\bm{\uns}.\pool}_{w,\HT}$} \\
	&	\textbf{mean}  & \textbf{MAE} &  \textbf{MAE} & \textbf{MAE}  \\ 
		\midrule
ARD & \textsl{-0.6757} & 1.2349 & 0.0538 & 0.0677 \\ 
PNA & \textsl{-0.3250} & 0.6482 & 0.0541 & 0.0384 \\ 
\midrule
&	\multicolumn{1}{c|}{$\hat{V}^{\bm{\s}}_{w,\BM}$}    & \multicolumn{1}{c|}{$\hat{V}_{w,\ivw}$}  & \multicolumn{1}{c|}{$\hat{V}^{\bm{\s}.\pool}_{w,\HT}$} & \multicolumn{1}{c}{$\hat{V}^{\bm{\uns}.\pool}_{w,\HT}$} \\
&	\textbf{mean}  & \textbf{MAE} &  \textbf{MAE} & \textbf{MAE}  \\ 
\midrule
ARD & \textsl{0.1098} & 0.0848 & 0.0395 & 0.0363 \\ 
PNA & \textsl{ 0.0641} & 0.0376 & 0.0158 & 0.0129 \\ 
		\bottomrule
	\end{tabular}
	}
	 \end{subtable}%
   \begin{subtable}[t]{0.5\textwidth}
  \centering
\caption{Unrestricted Benchmarks $\hat\beta^{\bm{\uns}}_{w,\BM}$, $\hat{V}^{\bm{\uns}}_{w,\BM}$}\label{tab:fed-ht-ivw-unres-main}
{\footnotesize
\centering
	\begin{tabular}{l|r|r|r|r}
		\toprule
	&	\multicolumn{1}{c|}{$\hat\beta^{\bm{\uns}}_{w,\BM}$}     & \multicolumn{1}{c|}{$\hat{\beta}_{w,\ivw}$} & \multicolumn{1}{c|}{$\hat\beta^{\bm{\s}.\pool}_{w,\HT}$} & \multicolumn{1}{c}{$\hat\beta^{\bm{\uns}.\pool}_{w,\HT}$} \\
  &	\textbf{mean}  & \textbf{MAE} &  \textbf{MAE} & \textbf{MAE}  \\ 
		\midrule
ARD & \textsl{-0.6497} & 1.2608 & 0.0622 & 0.0467 \\ 
PNA & \textsl{-0.3328} & 0.6403 & 0.0617 & 0.0321 \\ 
\midrule
&	\multicolumn{1}{c|}{$\hat{V}^{\bm{\uns}}_{w,\BM}$}    & \multicolumn{1}{c|}{$\hat{V}_{w,\ivw}$}  & \multicolumn{1}{c|}{$\hat{V}^{\bm{\s}.\pool}_{w,\HT}$} & \multicolumn{1}{c}{$\hat{V}^{\bm{\uns}.\pool}_{w,\HT}$} \\
 &	\textbf{mean} & \textbf{MAE} &  \textbf{MAE} & \textbf{MAE}  \\ 
\midrule
ARD & \textsl{0.1088} & 0.0837 & 0.0385 & 0.0352 \\
PNA & \textsl{0.0629} & 0.0364 & 0.0146 & 0.0118 \\ 
		\bottomrule
	\end{tabular}
			}
	\end{subtable}
	\bnotetab{Subsamples are simulated from the MarketScan ARD cohort, and from the MarketScan pneumonia (PNA) cohort with $D=2$. For subsamples drawn from ARD cohort, $n_1 = n_2 = 6,000$; for subsamples drawn from PNA cohort, $n_1 = n_2 = 10,000$. We use ATE weighting in IPW-MLE in these tables. The mean absolute error (MAE) is calculated relative to the benchmark mean values (first column of each table) based on 50 iterations of independent draws of subsamples. We report the mean value of benchmarks because the combined data $\mathcal{C}_1 \cup \mathcal{C}_2$ from which benchmarks are estimated vary across iterations. }
	\label{tab:fed-ivw-main}
\end{table}

\subsubsection{Results}\label{subsec:results}

We compare restricted and unrestricted federated IPW-MLE and IVW with the restricted and unrestricted benchmarks in Table \ref{tab:fed-ivw-main}. Additional simulation results with alternative sampling schemes and with federated MLE are presented Tables \ref{tab:varying-sampling-ratio}-\ref{tab:varying-d} in Appendix \ref{subsec:simulation-appendix}. The error of a federated estimator is defined as its difference from the benchmark.

There are four observations from Table \ref{tab:fed-ivw-main}. First and foremost, for both point and variance estimates, our restricted and unrestricted federated IPW-MLE have much lower errors than IVW, when compared to restricted and unrestricted benchmarks. Second, the restricted federated point estimator is closer to the restricted benchmark than the unrestricted benchmark. Analogously, the unrestricted federated point estimator is closer to the unrestricted benchmark. Third, the variance in the unrestricted benchmark and federated variance are larger than the restricted counterparts, implying the efficiency loss when flexible model specifications are used. Fourth, interestingly, the unrestricted federated variance is closer to variances in both restricted and unrestricted benchmarks. This is because federated variance tends to underestimate the true variance in finite samples (even though both are consistent). As the unrestricted federated variance tends to be larger, it partially corrects for the underestimation error. 

These observations are robust to alternative sampling schemes and to federated MLE as shown in Tables \ref{tab:varying-sampling-ratio}-\ref{tab:varying-d} in Appendix \ref{subsec:simulation-appendix}.\footnote{\cmtfinal{
    We could use alternative approaches to obtaining federated maximum likelihood estimator of treatment coefficient, such as by using a surrogate likelihood function that communicates gradients only \citep{jordan2018communication} or that communicates both gradients and Hessians \citep{duan2020learning} similar to our federated MLE. In the likelihood function, the heterogeneity in data sets can be adjusted through tilting the density ratio \citep{duan2022heterogeneity}; moreover, a regularization term can be included in high-dimensional settings \citep{wang2017efficient,li2021targeting}. These methods do not account for the treatment selection bias and are iterative, while federated IPW-MLE does and is noniterative. We expect the results of these methods to be conceptually similar to those of federated MLE.  }} As unrestricted federated IPW-MLE is more flexible and generally provides a better variance estimate, we use unrestricted federated IPW-MLE to federate MarketScan and Optum, as shown in Section \ref{subsec:empirical-federate-two} below.

		\subsection{Federation Across Two Medical Claim Data Sets}\label{subsec:empirical-federate-two}
		
		In this section, we seek to federate MarketScan and Optum to study the effect of alpha-blockers. As shown in Figure \ref{fig:fed-ms-optum}, the coefficient on alpha blockers is consistently negative on the individual cohorts of ARD patients and of pneumonia patients, implying a reduced risk of adverse outcomes for ARD and pneumonia patients who were exposed to alpha blockers.

\begin{figure}[t!]
\tcaptab{Federation Across MarketScan and Optum}

\begin{subfigure}{1.\textwidth}
		\centering		\includegraphics[width=1\columnwidth]{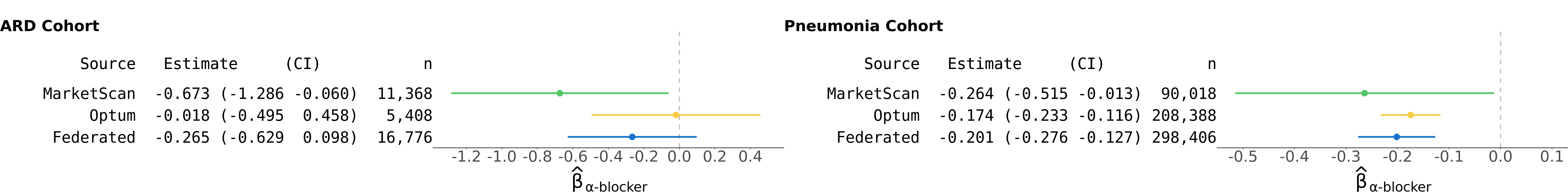}
			\end{subfigure}
			\label{fig:fed-ms-optum}
			\bnotefig{These figures show the estimated coefficient of alpha blockers and 95\% confidence interval on MarketScan and Optum, and federated coefficient and 95\% confidence interval from unrestricted federated IPW-MLE with ATE weighting. {\blue Note that for the pneumonia cohort, the confidence intervals for the federated estimator are wider than those on Optum. This can happen when the asymptotic variance is heterogeneous across data sets and the asymptotic variance on the small data is much larger than that on the large data. In this case, the federated variance obtained by sample size weighting can be larger than the variance on the large data.  See Appendix \ref{subsec:toy-example-sample-size-weighting} for a toy example.} See Figure \ref{fig:compare-ivw-ours-pna} in Appendix \ref{sec:additional-empirical} for ATT weighting; the results are close to those in these figures. 
			}
\end{figure}

    However, coefficients of some confounders, e.g., age, are of different magnitudes or signs in the outcome model across the two databases (though, none with statistical significance).
    This raises three potential concerns: model instability, model misspecification, and unobserved confounders across the two databases, which we ameliorate as follows. 
    
    First, model instability could be due to the different populations underlying these two databases, as shown in Figure \ref{fig:age-hist}, as well as the heterogeneous response of outcomes to the treatment and confounders. Unrestricted federated IPW-MLE with a flexible functional form for the combined data seems to be preferable in the presence of model instability. Second, model misspecification could exist if the response is indeed the same across two databases, but there exists a coefficient difference in the estimated outcome models. To protect against this possibility, we suggest using IPW-MLE due to its doubly robust properties (as opposed to MLE). Third, we have largely controlled for unobserved confounders in our approach to constructing cohorts, as discussed in Appendix \ref{subsec:study-definition}, and sensitivity analyses are conducted in \cite{koenecke2020alpha}.
    
    Figure \ref{fig:fed-ms-optum} shows the federated point estimates and confidence intervals from unrestricted federated IPW-MLE. As desired, the federated estimates of the effect of alpha blockers lie between the estimates on MarketScan and Optum for both ARD and pneumonia patients, and they approximate the average effect of alpha blockers on all ARD or pneumonia patients across two databases (recall the estimates from IVW may not lie between those on MarketScan and Optum as shown in Figure \ref{fig:compare-ivw-ours} and Figure \ref{fig:compare-ivw-ours-pna} in Appendix \ref{sec:additional-empirical}). 
    
    As a robustness check, we report the results from federated MLE in Figure \ref{fig:fed-ms-optum-mle}, and estimated treatment effects from federated IPW-MLE and AIPW in Figure \ref{fig:compare-ivw-ours-aipw} in Appendix \ref{sec:additional-empirical}.\footnote{\cmtfinal{Similar to Footnote 24, we could use alternative approaches to obtaining the federated estimator of treatment coefficient. The results would be conceptually similar to those of federated MLE in Figure \ref{fig:fed-ms-optum-mle}. Due to the treatment selection bias, the estimated treatment coefficient from alternative approaches would not have the interpretation of the average treatment coefficient on either the whole population or the treated population, while the estimated coefficient from IPW-MLE does. }  } Both the coefficient in the outcome model and estimated treatment effects of alpha blockers are negative and statistically significant, supporting our finding of an association between the exposure to alpha blockers and a reduced risk of progression to ventilation and death.

\section{Conclusion}\label{sec:conclusion}
This paper proposes three categories of federated inference methods based on MLE,  IPW-MLE, and AIPW, respectively. Our federated point estimators have the same asymptotic distributions as the corresponding estimators from combined, individual-level data. Our federated variance estimators are consistent. To achieve these properties, we show that the implementations of our federated methods should be adjusted based on conditions such as whether propensity and outcome models are stable across heterogeneous data sets. Finally, we apply our federated inference methods to study the effectiveness of alpha blockers on patient outcomes from two separate medical claims databases. 

To conclude, we would like to point out three interesting directions for future work. The first is to develop federated semiparametric or nonparametric estimation methods. The second is to develop communication-efficient, theoretically guaranteed federated causal inference methods in settings with high-dimensional nuisance parameters. The third is to develop these methods in settings with many data sets, while each data set may only have a small number of observations.

\end{onehalfspacing}

 \bibliographystyle{apalike} 
 { \small
\bibliography{reference}
}

\newpage
\begin{appendices}
{\small

\section{Supplementary Details and Results}\label{sec:additional-asymptotic}

{\blue 

\subsection{Regularity Conditions}\label{subsec:regularity-conditions}

\begin{assumption}[Regularity Conditions on Outcome and Propensity Models]\label{assumption:parametric}
\texttt{} 
	\begin{enumerate}
		\item Condition \ref{cond:parametric-outcome} holds. For any $k$, $\mathcal{X}_k$ is bounded. $f(y \mid \*x, w,  \bm\beta)$ is twice continuously differentiable in $\bm\beta$. $\bm\beta^{\sk\ast} \in \mathcal{S}_{\bm\beta}^\sk \subset \+R^{d_k+1}$ lies in the interior of a known compact set $\mathcal{S}_{\bm\beta}^\sk$, where $\bm\beta^{\sk\ast} $ is the unique solution that minimizes $-\+E[\log f(y \mid \*x, w,\bm\beta^{\sk\ast})]$. The information matrix $\mathcal{I}^\sk(\bm\beta)  = -\+E_{(\*x, w, y) \sim \mathbb{P}^\sk} \Big[ \frac{\partial^2 \log  f(y \mid \*x, w, \bm\beta) }{\partial \bm\beta \partial \bm\beta^\top } \Big] $ is positive definite,  full rank, and its condition number is bounded for all $\bm\beta$.
		\item Condition \ref{cond:parametric-propensity} holds. For any $k$, $\mathcal{X}_k$ is bounded. $e(\*x,\bm\gamma)$ is twice continuously differentiable in $\bm\gamma$. $\bm\gamma^{\sk\ast} \in \mathcal{S}_{\bm\gamma}^\sk \subset \+R^{d_k+1}$ lies in the interior of a known compact set $\mathcal{S}_{\bm\gamma}^\sk$, where $\bm\gamma^{\sk\ast}$ is the unique solution that minimizes $-\+E[\log  e(\*x, \bm\gamma^{\sk\ast})]$. The information matrix $\mathcal{I}^\sk(\bm\gamma)  = -\+E_{(\*x, w) \sim \mathbb{P}^\sk} \Big[ \frac{\partial^2 \log  e(\*x, \bm\gamma) }{\partial \bm\gamma \partial \bm\gamma^\top } \Big] $ is positive definite, full rank, and its condition number is bounded for all $\bm\gamma$.
		\item Regularity conditions in Assumptions \ref{assumption:parametric}.1 and  \ref{assumption:parametric}.2 hold for the outcome and propensity models on the combined, individual-level data.
	\end{enumerate}
\end{assumption}

If $f(y \mid \*x, w,  \bm\beta)$ contains the true structure $f_0(y \mid \*x, w, \bm\beta_0^\sk)$, then $\+E[\log f(y \mid \*x, w,\bm\beta_0^\sk)] = 0$ and $\bm\beta^{\sk\ast} = \bm\beta_0^\sk$. Similarly,  if $e(\*x, \bm\gamma)$ contains the true structure $e_0(\*x, \bm\gamma_0^\sk)$, then\\ $\+E[\log e(\*x,\bm\gamma_0^\sk)] = 0$ and $\bm\gamma^{\sk\ast} = \bm\gamma_0^\sk$. The same properties hold for the density functions on the combined, individual-level data, with parameters $\bm\beta^{\ast}$, $\bm\beta_0$, $\bm\gamma^{\ast}$, and $\bm\gamma_0$ defined analogously to $\bm\beta^{\sk\ast}$, $\bm\beta_0^\sk$, $\bm\gamma^{\sk\ast}$, and $\bm\gamma_0^\sk$.


\subsection{Treatment Effect Estimation Based on IPW-MLE}\label{subsec:ipw-mle-double-robustness}
After we estimate the parameters $\bm\beta$ in the likelihood function, we can use $\hat{\bm\beta}_\HT$  to estimate the conditional outcome models $\mu_{(w)}(\*X_i, \bm\beta) = \+E[Y_i \mid \*X_i, W_i = w]$ \footnote{Since the likelihood function can be parametrized by $\bm{\beta}$ and notice that $\+E[Y_i \mid \*x_i = \*x, W_i = w] = \int y f(y \mid \*X_i = x, W_i = w, \bm{\beta}) d y$, the conditional outcome models can also be parametrized by $\bm{\beta}$. } and $\tau_\ate$ \footnote{The estimator of $\tau_\att$ can be defined as $\hat{\tau}_\att= \frac{1}{\sum_{i=1}^n W_i} \sum_{i=1}^n  W_i  \cdot \left[\mu_{(1)}(\*X_i, \hat{\bm\beta}_\HT) -  \mu_{(0)}(\*X_i, \hat{\bm\beta}_\HT) \right].$}
\[\hat{\tau}_\ate= \frac{1}{n} \sum_{i=1}^n \left[\mu_{(1)}(\*X_i, \hat{\bm\beta}_\HT) -  \mu_{(0)}(\*X_i, \hat{\bm\beta}_\HT) \right].\]

$\hat{\tau}_\ate$ estimated from this approach enjoys the ``double robustness'' property \citep{wooldridge2007inverse,lumley2011complex}, meaning that $\tau_\ate$ is consistent even if one of outcome and propensity models, but not both, is misspecified. On one hand, if the outcome model is correctly specified, then $\hat{\bm\beta}_\HT$ is consistent. We can show that  $\frac{1}{n} \sum_i \mu_{(w)}(\*X_i, \hat{\bm\beta}_\HT)$ is a consistent estimator of 
$\+E[Y_i(w)]$, and $\hat{\tau}_\ate$ is consistent.\footnote{Note that $\+E[Y_i(w)] = \+E[\mu_{(w)}(\*X_i, \bm\beta)]$.}

On the other hand, if the outcome model is misspecified, and if the propensity model is correctly specified, then  $\hat{\bm\beta}_\HT$ is a consistent estimator of $\bm\beta^\ast$, where $\bm\beta^\ast$ is the unique solution that maximizes $\+E[\log f(Y_i \mid \*X_i, W_i, \bm\beta^\ast)]$. If the conditional outcome models satisfy $\+E[\mu_{(w)}(\*X_i, \bm\beta^\ast)] = \+E[Y_i(w)] $,\footnote{We can show that if $\mu_{(w)}(\*X_i, \bm\beta^\ast)$ is a linear or logistic function of $\*X_i$ and $w$ with an intercept term, then $\+E[\mu_{(w)}(\*X_i, \bm\beta^\ast)] = \+E[Y_i(w)] $. } then $\hat{\tau}_\ate$ is still consistent \citep{wooldridge2007inverse}.

Additionally, under suitable assumptions, $\hat{\tau}_\ate$ is asymptotically normal, 
\begin{align*}
	\sqrt{n} \big(\hat{\tau}_\ate -  \tau_\ate   \big) \xrightarrow{d}&\mathcal{N} \big(0, \+E[J(\*X_i, \bm\beta^\ast)]^\T \cdot \*V_{\bm\beta^\ast, \HT, \hat{e}}^\dagger \cdot  \+E[J(\*X_i, \bm\beta^\ast)] \big).
	\end{align*}
where $\*V_{\bm\beta^\ast, \HT, \hat{e}}^\dagger$ is defined in Lemma \ref{lemma:expression-ht-var} and $J(\*X_i, \bm\beta)$ is the gradient
\begin{align*}
    J(\*X_i, \bm\beta) =& \frac{\partial}{\partial \bm\beta} \left[\mu_{(1)}(\*X_i, \bm\beta) - \mu_{(0)}(\*X_i, \bm\beta) \right].
\end{align*}
For example, if the outcome model is logit with parameters $\bm\beta$, 
\begin{align*}
    \mu_{(w)}(\*X_i, \bm\beta) = \frac{1}{1 + \exp(- \tilde{\*X}^\T_{(w),i} \bm{\beta})}
\end{align*}
for $\tilde{\*X}_{(w),i} = [w, \*X^\T_i]^\T$, then the gradient is 
\[J(\*X_i, \bm\beta) =  \mu_{(1)}(\*X_i, \bm\beta) \big(1 - \mu_{(1)}(\*X_i, \bm\beta) \big) \cdot \tilde{\*X}_{(1),i} - \mu_{(0)}(\*X_i, \bm\beta) \big(1 - \mu_{(0)}(\*X_i, \bm\beta) \big) \cdot \tilde{\*X}_{(0),i}.\]

\subsection{Federated IPW-MLE for ATE}
We can construct a federated estimator for average treatment effects based on IPW-MLE. Specifically, we first use federated IPW-MLE to obtain the federated parameters $\hat{\bm{\beta}}^\pool_\HT$ in the outcome model on the combined data. Next we use $\hat{\bm{\beta}}^\pool_\HT$ to estimate ATE on each data set. Let the estimator on data set $k$ be $\hat{\tau}^\sk_\ate$. Finally we use sample size weighting to combine $\hat{\tau}^\sk_\ate$ together to obtain the federated ATE, $\hat{\tau}^\pool_\ate$. 

For the asymptotic variance of federated ATE, we can first use $\hat{\bm{\beta}}^\pool_\HT$ to estimate $\+E[J(\*X_i, \bm\beta^\ast)]^\T $ and $ \*V_{\bm\beta^\ast, \HT, \hat{e}}^\dagger$ on each data set $k$, and then use sample size weighting to combine these estimates together to obtain the federated variance on the combined data.

}

\subsection{Lemma for AIPW}\label{subsec:lemma-aipw}

Our federated AIPW estimators in Appendices \ref{subsec:pool-stable-aipw} and \ref{subsec:pool-aipw-model-shift} are based on the asymptotic linear property of the AIPW estimator
\citep{robins1994estimation,tsiatis2007comment}. For completeness, we state this property in the following lemma. 

\begin{lemma}[Adapted from \cite{tsiatis2007comment} and \cite{chernozhukov2017double}]\label{lemma:expression-aipw-var}
	Suppose at least one condition holds: (a) $\mu_{(w)}(\*x)$ is correctly specified and consistently estimated for $w \in \{0,1\}$, or (b) $e(\*x)$ is correctly specified and consistently estimated. Then the AIPW estimator $\hat{\tau}_\ate$ for ATE satisfies 
	\begin{align}\label{eqn:aipw-asymptotic-normality}
	\sqrt{n} (\hat{\tau}_\ate - \tau_\ate) = \frac{1}{\sqrt{n}} \sum_{i = 1 }^n [\hat{\phi}(\*X_i, W_i, Y_i ) - \tau_\ate] =\frac{1}{\sqrt{n}} \sum_{i = 1 }^n \phi(\*X_i, W_i, Y_i ) + o_p(1) 
	\xrightarrow{d} \mathcal{N} \big( 0, \*V_\tau \big),
	\end{align}
	for the influence function $\phi(\*x, w, y)$ that satisfies $\+E[\phi(\*x, w, y) ] = 0$ and  $\*V_\tau = \+E[\phi(\*x, w, y)^2 ]$ and is defined as
	\[\phi(\*x,w,y) =  \mu_{(1)}(\*x)- \mu_{(0)}(\*x)+ \frac{w}{e(\*x)} ( y -\mu_{(1)}(\*x) ) - \frac{(1 - w )}{1 - e(\*x)} (y - \mu_{(0)}(\*x) )  - \tau_0\]
	The AIPW estimator $\hat{\tau}_\att$ for ATT also satisfies \eqref{eqn:aipw-asymptotic-normality} with $\phi(\*x, w, y)$ defined as 
	\[\phi(\*x,w,y ) =   w \big(y - \mu_{(1)}(\*x) \big)- \frac{e(\*x)(1 - w)}{1 - e(\*x)} \big(y -\mu_{(0)}(\*x)  \big) - \tau_0.\]
\end{lemma}

We can see from Lemma \ref{lemma:expression-aipw-var} that the score $\hat{\phi}(\*x,w,y)$ in the definition of $\hat{\tau}_\ate$ is an estimator of $\tau_\ate+ \phi(\*x,w,y)$ (recall Section \ref{subsec:aipw-estimate}, and similarly for $\hat{\tau}_\att$). Lemma \ref{lemma:expression-aipw-var}  formally states the doubly robust property mentioned in Appendix \ref{subsec:aipw-estimate}:
$\hat{\tau}_\aipw$ continues to be consistent and asymptotically normal if either the propensity model is misspecified or the outcome model is misspecified, but not both.


\subsection{IVW has the minimum variance}\label{subsec:ivw-min-var}
Let $\hat{z}^\concat$ be an estimator for the combined data and $\hat{z}^\sk$ be an estimator on data set $k$. The following discussion holds for $\hat{z}$ to be any of $\hat{\tau}_\ate$, $\hat{\tau}_\att$, $\hat{\*V}_{\tau_\ate}$ and $\hat{\*V}_{\tau_\att}$.

Let $\hat z^\concat = \sum_{k = 1}^D \omega_k \hat z^\sk$ with $\sum_{k = 1}^D \omega_k = 1$. Since $\hat z^\sk$ for all $k$ are estimated from different populations, they are independent and 
\[\mathrm{Var}(\hat z^\concat ) = \sum_{k = 1}^D \omega_k^2 \mathrm{Var}(\hat z^\sk ). \]
To solve the $\omega_k$ that minimizes $\mathrm{Var}(\hat z^\concat )$ under the constraint $\sum_{k = 1}^D \omega_k = 1$, we introduce a Lagrange multiplier $\lambda$, and we seek to solve $\omega_k$ and $\lambda$ from the following Lagrange function
\[\mathcal{L}(\bm{\omega}, \lambda) = \sum_{k = 1}^D \omega_k^2 \mathrm{Var}(\hat z^\sk ) - \lambda \left(\sum_{k = 1}^D \omega_k - 1 \right) \]
Setting the derivative of $\mathcal{L}(\bm{\omega}, \lambda)$  with respect to $\omega_k$ to zero, we have $\omega^\ast_k = \lambda/(2\mathrm{Var}(\hat z^\sk ))$. Given that $\sum_{k = 1}^D \omega^\ast_k = 1$, the solution $(\lambda^\ast,\bm{\omega}^\ast)$ that minimizes $\mathcal{L}(\bm{\omega}, \lambda)$ is 
\[\lambda^\ast = \frac{2}{\sum_{j = 1}^D 1/\mathrm{Var}(\hat z^{(j)} ) } \qquad  \omega_k^\ast = \frac{1/\mathrm{Var}(\hat z^\sk )}{\sum_{ = 1}^D 1/\mathrm{Var}(\hat z^{(j)} ) } \quad \forall k.\]
In other words, $\omega_k^\ast$ is the same as IVW.
\subsection{Supplementary Results}

When the outcome model is unstable, if we continue using the same federation formulas as those for stable models in Section \ref{subsec:pool-stable-mle}, Theorem \ref{thm:pool-mle} continues to hold for some special cases, but $\hat{\bm\beta}^\pool_\mle$ converges to a different limit from that in Theorem \ref{thm:pool-mle}.

\begin{proposition}[Restricted Federated MLE for Correctly-Specified Unstable Outcome Models]\label{proposition:model-shift-mle-weighted}
	Suppose Assumption \ref{assumption:parametric}.1 hold, Condition \ref{cond:stable-outcome} holds, and $\norm{  \mathcal{I}^\concat(\bm\beta)^\I \mathcal{I}^\sk(\bm\beta)  }_2 \leq M$ for some $M < \infty$. Furthermore, suppose $\dot{\*d}^\sk_y(\bm\beta)  - \mathcal{I}^\sk(\bm\beta) \cdot \bm\beta$ and $\mathcal{I}^\sk(\bm\beta)$ do not depend on $\bm\beta$ for all $k$,  where $\dot{\*d}^\sk_y(\bm\beta)  = \+E_{(\*x, w, y) \sim \mathbb{P}^\sk }   \Big[ \frac{\partial \log  f(y|\*x, w;  \bm\beta) }{\partial \bm\beta  } \Big]$. As $n_1, \cdots, n_D \rightarrow \infty$, we have
	\begin{align}
	n_\npool^{1/2} (\hat{\*V}_{\bm\beta}^\concat)^{-1/2} (\hat{\bm\beta}^\pool_\mle - \bm\beta^\dagger) \xrightarrow{d}& \mathcal{N} (0, \*I_d), \label{eqn:proposition-model-shift-mle}
	\end{align}
	where $\bm\beta^\dagger$ minimizes the 
	Kullback-Leibler Information Criterion between $f_0(y|\*x, w,  \bm\beta^\dagger)$ and the mixture of  $f_0(y|\*x, w,  \bm\beta_0^\sk)$ on the combined data. If we replace $\hat{\bm\beta}^\pool_\mle$ by $\hat{\bm\beta}^\concat_\mle$ and/or replace  $\hat{\*V}_{\bm\beta}^\concat$ by $\hat{\*V}_{\bm\beta}^\pool$, then \eqref{eqn:proposition-model-shift-mle} continues to hold.
\end{proposition}

If the outcome model is linear with i.i.d. Gaussian noise and variance $\sigma^2_e$, then $\mathcal{I}^\sk(\bm\beta)  = \*X^\T \*X/\sigma_e^2 $ and $\dot{\*d}^\sk_y(\bm\beta)  - \mathcal{I}^\sk(\bm\beta) \cdot \bm\beta = -  \*Y^\T \*X/\sigma_e^2$ do not depend on $\bm\beta$, satisfying the assumptions in Proposition \ref{proposition:model-shift-mle-weighted}. In this case, $\bm\beta^\dagger$ is a weighted average of $(\bm\beta_0^{(1)},\bm\beta_0^{(2)}, \cdots, \bm\beta_0^{(D)})$ and satisfies $\sum_{k = 1}^D p_k \+E_{\*x \sim \mathbb{P}^\sk} [\*x] \cdot (\bm\beta_0^\sk - \bm\beta_0^\ast) = 0$. 

\subsection{Practical Considerations}

Regarding the variance estimator of IPW-MLE, if $Y_{i}$ is binary, $\+E[Y_i|\*X_i]$ follows a logit model, and the true propensity score is used, then we can estimate $\*D^\sk_{\bm\beta_0, \varpi}$ by
\begin{align*}
\hat{\*D}_{\bm\beta, \varpi}^\sk =& \frac{1}{n_k } \sum_{i = 1 }^{n_k}  \Big( \frac{W_i}{(\hat{e}^\pool_i )^2}   + \frac{1 - W_i}{(1 - \hat{e}^\pool_i)^2 }\Big)  \hat \varepsilon_i^2 \*X_i \*X_i^\top,
\end{align*}
where $\varepsilon_i$ is unit $i$'s residual. Some commonly used packages, such as \textsf{syvglm}  in \textsf{R} \citep{lumley2011complex}, use working residuals for $\hat \varepsilon_i$ (i.e., $\hat \varepsilon_i = \frac{Y_i - \hat p_i}{\hat p_i (1 - \hat p_i) }$  and $ \hat p_i = \frac{\exp(\*X_i \hat{\bm\beta}_\HT^\pool)}{1 + \exp(\*X_i \hat{\bm\beta}_\HT^\pool)}$).

\section{Supplementary Empirical Analyses}\label{sec:additional-empirical}
\subsection{A Toy Example for Inverse Variance Weighting to Combine Coefficients}\label{subsec:toy-example-ivw}
In this section, we present a simplified example for the federated treatment coefficient from inverse variance weighting lying outside the interval between treatment coefficients on two data sets. Suppose we only have treatment and age in the outcome model, and the coefficients and inverse variance matrices on two data sets \footnote{These numbers are identical to those in the inverse propensity-weighted logistic regression on MarketScan and Optum ARD cohorts.} are:
\begin{align*}
    & \hat{\bm\beta}_\ms = \begin{bmatrix} \hat{\beta}_{\ms, w} \\ \hat{\beta}_{\ms, \mathrm{age}} \end{bmatrix} = \begin{bmatrix} -0.67 \\
2.03 \end{bmatrix},  \quad \hat{\*V}^{-1}_\ms  = \begin{bmatrix} 51.6 &  -28.6 \\ -28.6	& 474.02 \end{bmatrix},  \\ 
&\hat{\bm\beta}_\optum = \begin{bmatrix} \hat{\beta}_{\optum, w} \\ \hat{\beta}_{\optum, \mathrm{age}} \end{bmatrix} = \begin{bmatrix} -0.02 \\ -0.15\end{bmatrix}, \quad \hat{\*V}^{-1}_\optum  = \begin{bmatrix} 55.34 & 	14.61 \\ 14.61	& 187.98 \end{bmatrix}.
\end{align*}
Then the federated coefficients based on inverse variance weighting are 
\begin{align*}
    \hat{\bm\beta}_{\ivw} = \big(\hat{\*V}^{-1}_\ms + \hat{\*V}^{-1}_\optum \big)^{-1} \big(\hat{\*V}^{-1}_\ms \hat{\bm\beta}_\ms + \hat{\*V}^{-1}_\optum \hat{\bm\beta}_\optum \big) = \begin{bmatrix} -0.71 \\ 1.42 \end{bmatrix}
\end{align*}
The federated treatment coefficient is $-0.71$, which is smaller than $\hat{\beta}_{\ms, w} $ and $\hat{\beta}_{\optum, w}$.


{\blue 
\subsection{A Toy Example for Sample Size Weighting to Combine Variances}\label{subsec:toy-example-sample-size-weighting}
In this section, we present a toy example for the federated confidence intervals to be wider than the confidence intervals of an individual data set (or equivalently, the federated variance to be larger than the variance of an individual data set). This toy example is based on the point estimates, confidence intervals and sample sizes of the pneumonia cohort in Figure \ref{fig:fed-ms-optum}. Let the sample size on MarketScan and Optum be $n_\ms = 90,018$ and $n_\optum = 208,388$. Let the estimated variance (scaled by sample size) and estimated asymptotic variance on MarketScan and Optum be
\begin{align*}
    \hat{V}_{\ms,\mathrm{sc}} =  \left(\frac{-0.013-(-0.264)}{1.96}\right)^2 = 0.0164 & \qquad \hat{V}_{\ms} =  n_\ms \hat{V}_{\ms,\mathrm{sc}}  = 1476.27 \\
    \hat{V}_{\optum,\mathrm{sc}} = \left(\frac{-0.116-(-0.174)}{1.96}\right)^2 = 0.00088 & \qquad \hat{V}_{\optum} =  n_\optum \hat{V}_{\optum,\mathrm{sc}} = 182.48.
\end{align*}
The federated variance estimator that weighs $\hat{V}_{\ms}$ and $\hat{V}_{\optum}$ by sample size weighting is 
\begin{align*}
    & \hat{V}^\pool = \frac{n_\ms}{n_\ms+n_\optum} \hat{V}_{\ms} + \frac{n_\optum}{n_\ms+n_\optum} \hat{V}_{\optum} = 572.77 \\ & \hat{V}^\pool_{\mathrm{sc}} = \frac{\hat{V}^\pool}{n_\ms+n_\optum} = 0.0019 > \hat{V}_{\optum,\mathrm{sc}}= 0.00088
\end{align*}
Then the federated variance $\hat{V}^\pool_{\mathrm{sc}}$ is larger than the estimated variance $\hat{V}_{\optum,\mathrm{sc}}$. Even though the the federated variance estimator of IPW-MLE is more than complicated than this toy example, the general intuition is the same. 

}

\subsection{Study Definitions}\label{subsec:study-definition}
We follow the study definitions in \cite{koenecke2020alpha}.

\paragraph{Participants}\label{paragraph:participant} We study two cohorts of patients who were diagnostically coded in U.S. hospitals with acute respiratory distress
(ARD) from each of the MarketScan and Optum databases. We further study two cohorts of patients diagnostically coded in U.S. hospitals with pneumonia from each of the MarketScan and Optum databases. 

We limit the study to older men because alpha blockers are widely used as a treatment in the U.S. for benign prostatic hyperplasia (BPH), a common condition in older men that is clinically unrelated to the respiratory system. More specifically, we focus on men over the age of 45 so that a large portion of the exposed group faces similar risks of poor outcomes from respiratory conditions as the unexposed group, thus mitigating confounding by indication.\footnote{Note that this limits our analysis' validity to older men due to their being the dominant population historically being prescribed alpha blockers. However, we recognize the importance of studying other demographics, such as women and younger men, in clinical studies \citep{holdcroft2007, mcmurray1991}; extrapolating our results to these demographics would require additional assumptions as noted in \citep{powell2021ten}.} In addition, we enforce a maximum age of 85 years to reflect the ongoing clinical trials investigating prazosin (an alpha blocker) and its effects on COVID-19 patients.\footnote{See \url{https://clinicaltrials.gov/ct2/show/NCT04365257}.}

		\begin{figure}[ht!]
\tcaptab{Histograms of Patient Age in MarketScan and Optum}
			\centering
			\begin{subfigure}{0.5\textwidth}
				\centering
				\includegraphics[width=0.7\linewidth]{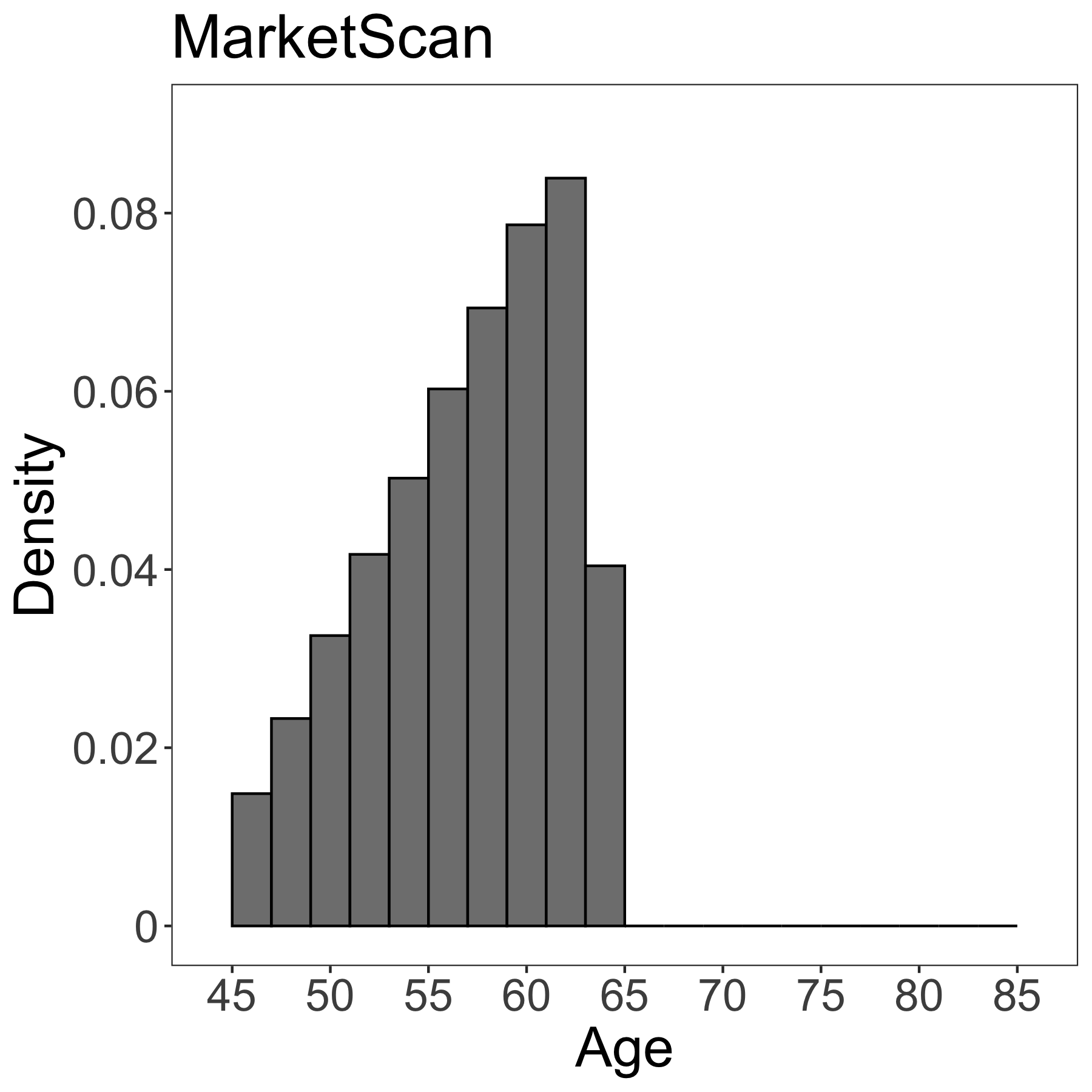}
			\end{subfigure}%
	\begin{subfigure}{0.5\textwidth}
				\centering
				\includegraphics[width=0.7\linewidth]{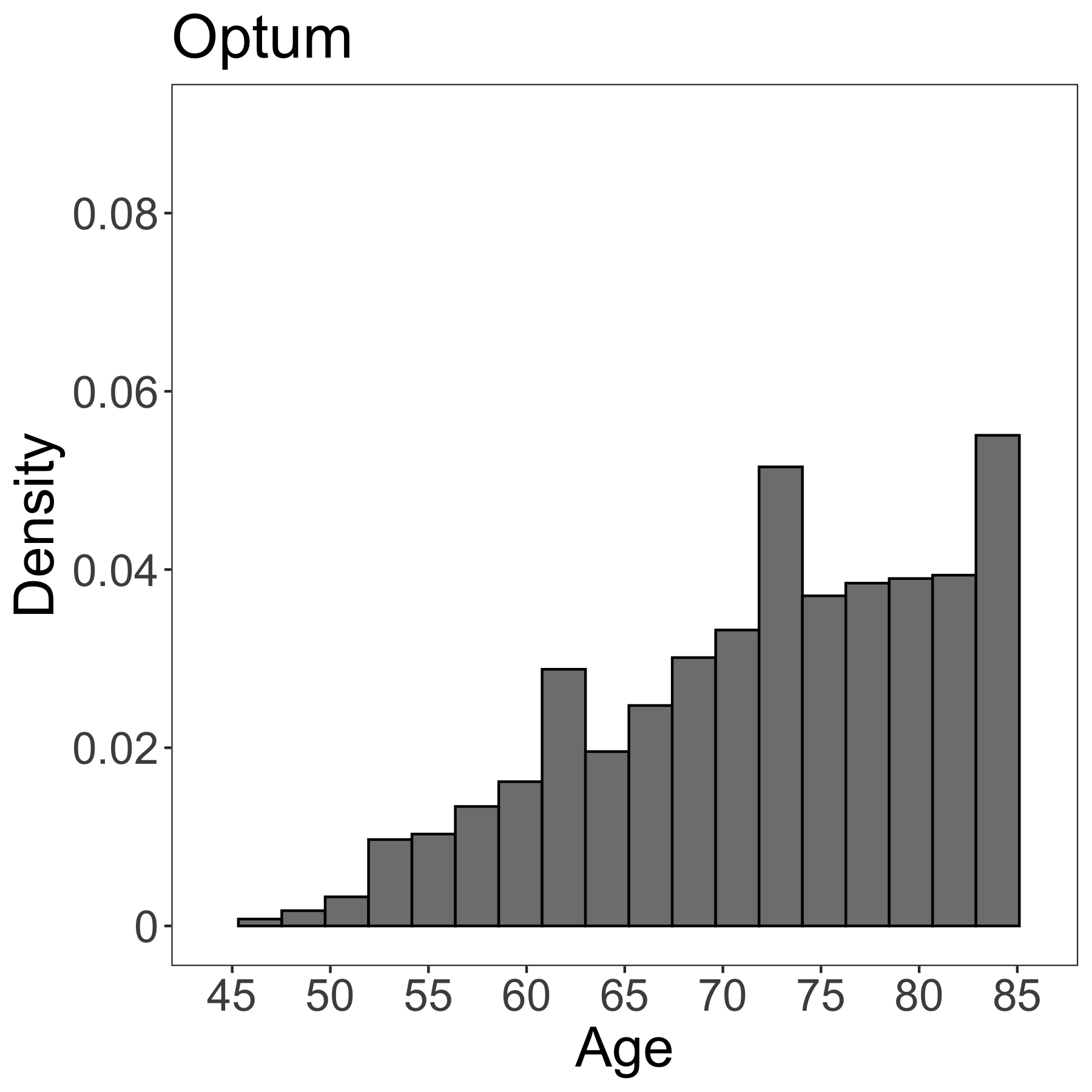}
			\end{subfigure}
			\bnotefig{We restrict all patients in both MarketScan and Optum databases to be over the age of 45. While patient data from the MarketScan database only include patients younger than age 65, a majority of the patients in the Optum database are over 65 years old.}
			\label{fig:age-hist}
		\end{figure}

After the restrictions on sex and age,  we obtain a cohort of 12,463 ARD inpatients and a cohort of 103,681 pneumonia inpatients from the MarketScan database (denoted as $\mathcal{C}_{\ms,\mathrm{ARD}}$ and $\mathcal{C}_{\ms,\mathrm{PNA}}$, respectively), and a cohort of 6,084 ARD inpatients and a cohort of  234,993 pneumonia inpatients from the Optum database (denoted as $\mathcal{C}_{\optum,\mathrm{ARD}}$ and $\mathcal{C}_{\optum,\mathrm{PNA}}$, respectively).

The demographics of patients in the MarketScan and Optum databases differ in two aspects. First, Optum includes older patients as MarketScan only includes patients up to age 65 due to Medicare exclusions (see Figure \ref{fig:age-hist} for the distribution of patient age on MarketScan and Optum). Second, Optum has more recent patient records from the fiscal year 2004 to 2019, while MarketScan only has patient records from the fiscal year 2004 to 2016.



\paragraph{Potential Confounders $\*X_i$} $\*X_i$ consists of age, fiscal year, and health-related confounders. Health-related confounders include total weeks with inpatient
admissions in the prior year, total outpatient visits in the prior year, total days as an inpatient in the
prior year, total weeks with inpatient admissions in the prior two months, and comorbidities identified
from healthcare encounters in the prior year: hypertension, ischemic heart disease, acute myocardial infarction, heart failure, chronic obstructive pulmonary disease, diabetes mellitus, and cancer.

\subsection{Additional Results for Federation Across Two Medical Claim Data Sets}

\begin{figure}[ht!]
		\centering
		\tcapfig{Coefficient of the Exposure to Alpha Blockers}
		\includegraphics[width=0.6\linewidth]{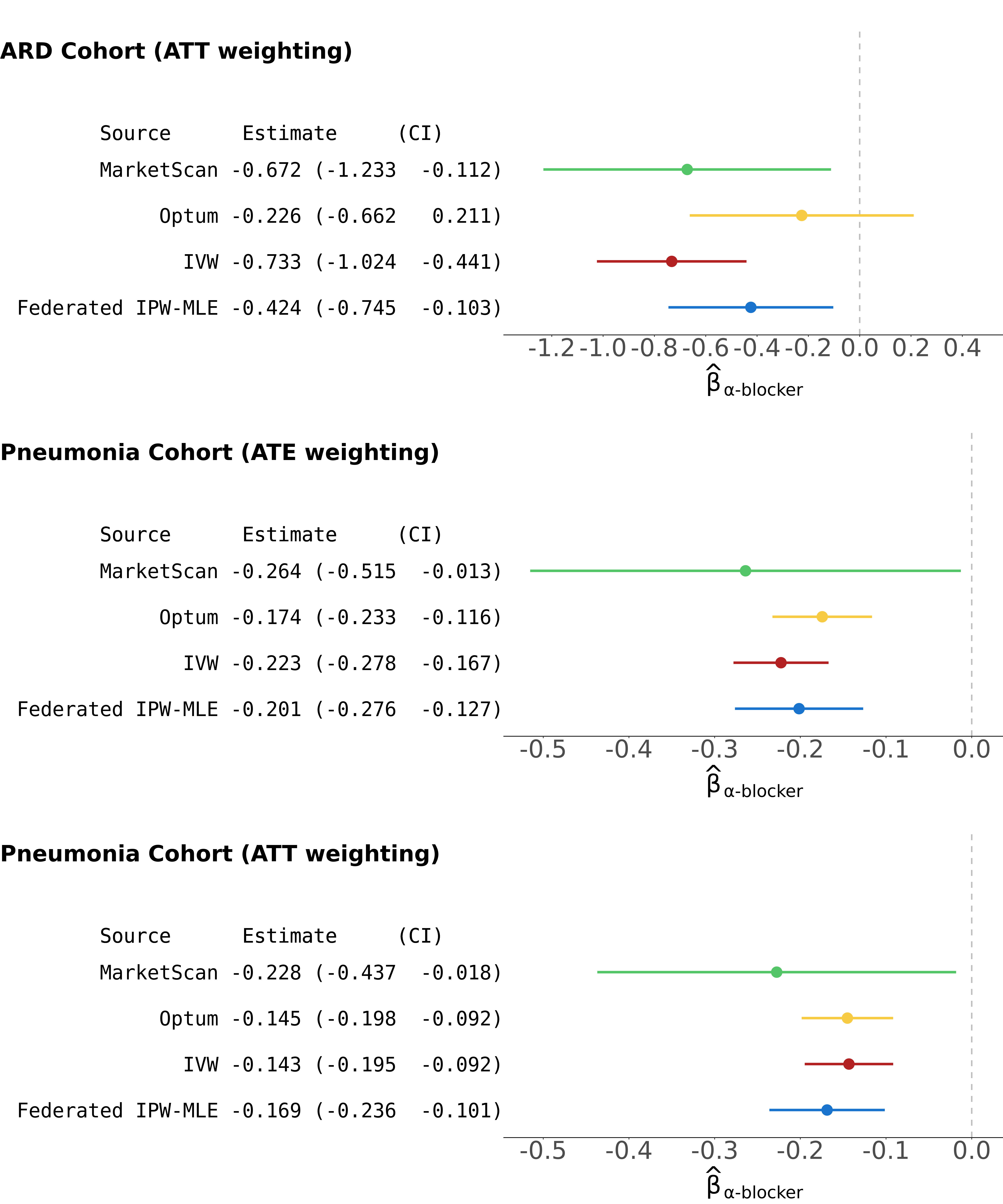}
		\label{fig:compare-ivw-ours-pna}
		\bnotefig{These figures show the estimated coefficient of alpha blockers and 95\% confidence interval on MarketScan and Optum, and federated coefficient and 95\% confidence interval from IVW and unrestricted federated IPW-MLE. These figures complement Figure \ref{fig:compare-ivw-ours} with ATE and ATT weighting on ARD and pneumonia cohorts. The federated coefficient from IVW lies outside the interval between treatment coefficients on two data sets only for the ARD cohort, whose sample size is much smaller than that of the pneumonia cohort.}
\end{figure}

		\begin{figure}[ht!]
\tcaptab{Coefficient of Age}\label{fig:coef-age}
			\centering
				\centering	\includegraphics[width=0.6\linewidth]{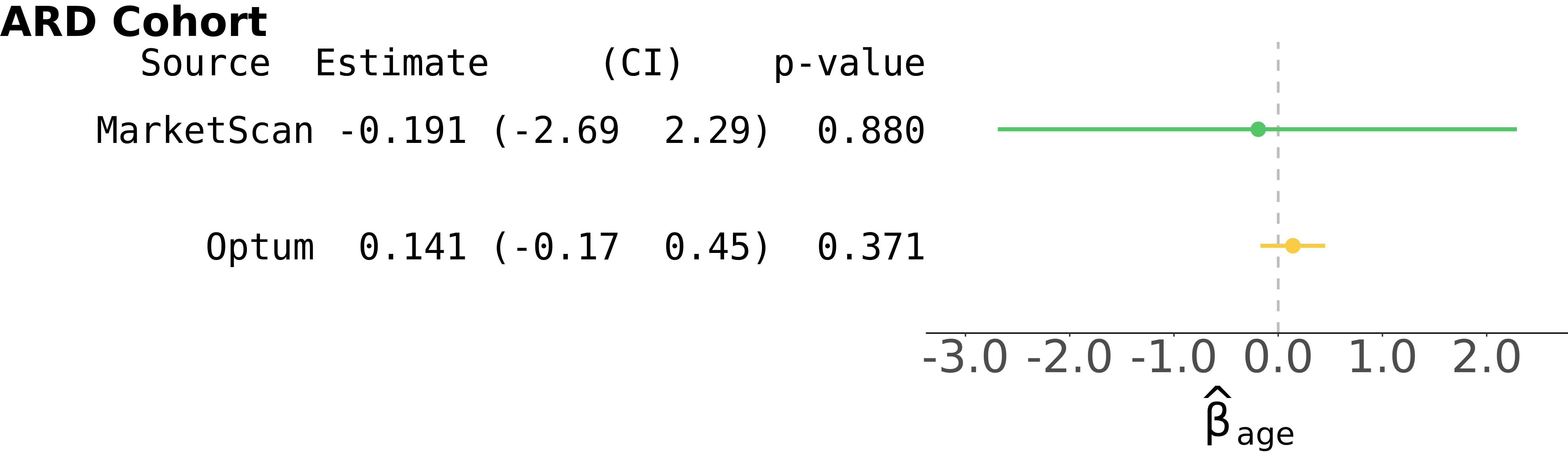}
\bnotefig{Coefficient of age has opposite signs in the logit model on two data sets.}
		\end{figure}



\begin{figure}
\tcaptab{Federation Across MarketScan and Optum (Unrestricted Federated MLE)}

		\centering
		\includegraphics[width=0.7\linewidth]{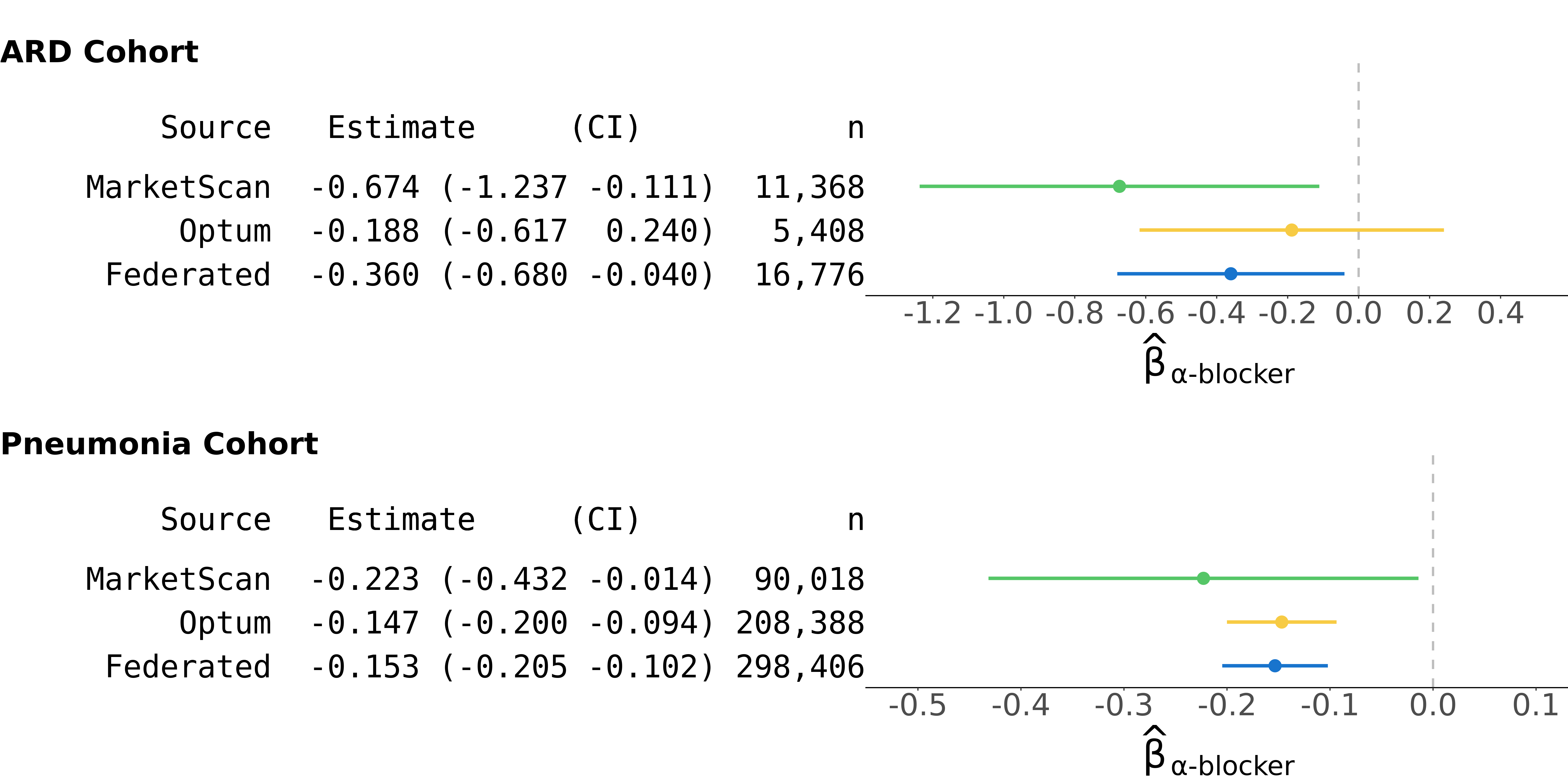}
			\label{fig:fed-ms-optum-mle}
\end{figure}

			
		
	
			\begin{figure}
		\centering
		\tcapfig{Federated ATE Across MarketScan and Optum}
		\begin{subfigure}{1\textwidth}
				\centering	\includegraphics[width=0.7\linewidth]{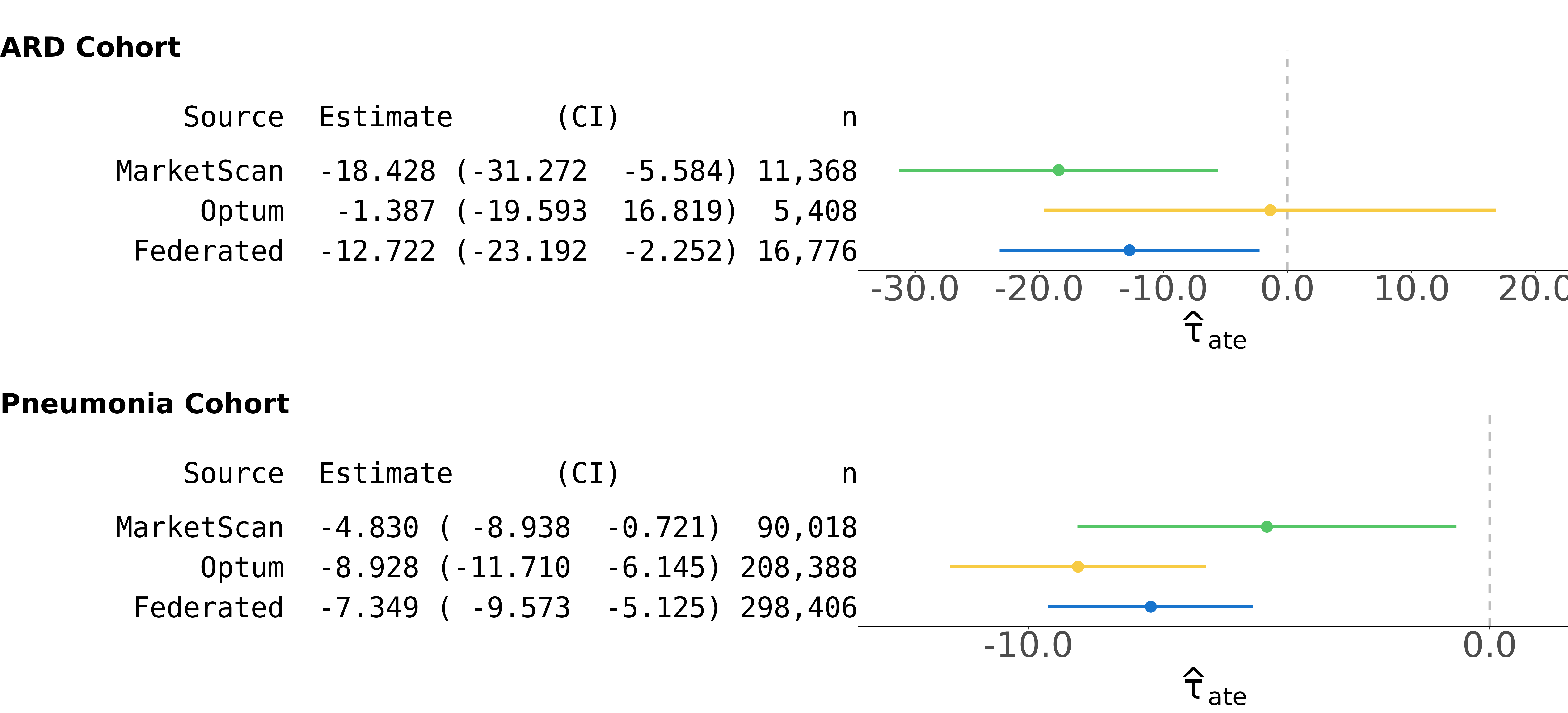}	
			\caption{Restricted AIPW {(inverse variance weighting)}}
			\end{subfigure}
		\begin{subfigure}{1\textwidth}
				\centering	\includegraphics[width=0.7\linewidth]{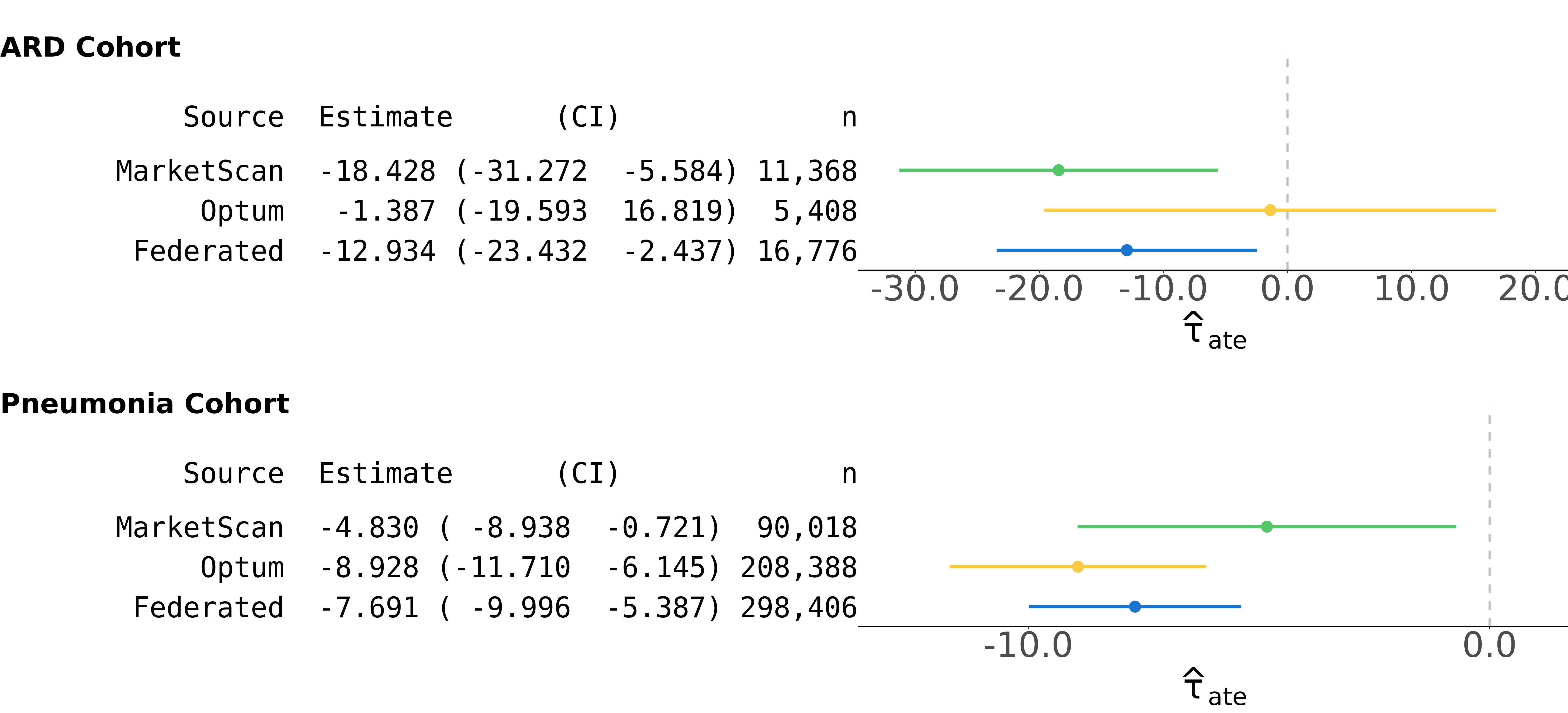}	
			\caption{Unrestricted AIPW (sample size weighting)}
			\end{subfigure}
			
			\begin{subfigure}{1\textwidth}
				\centering	\includegraphics[width=0.7\linewidth]{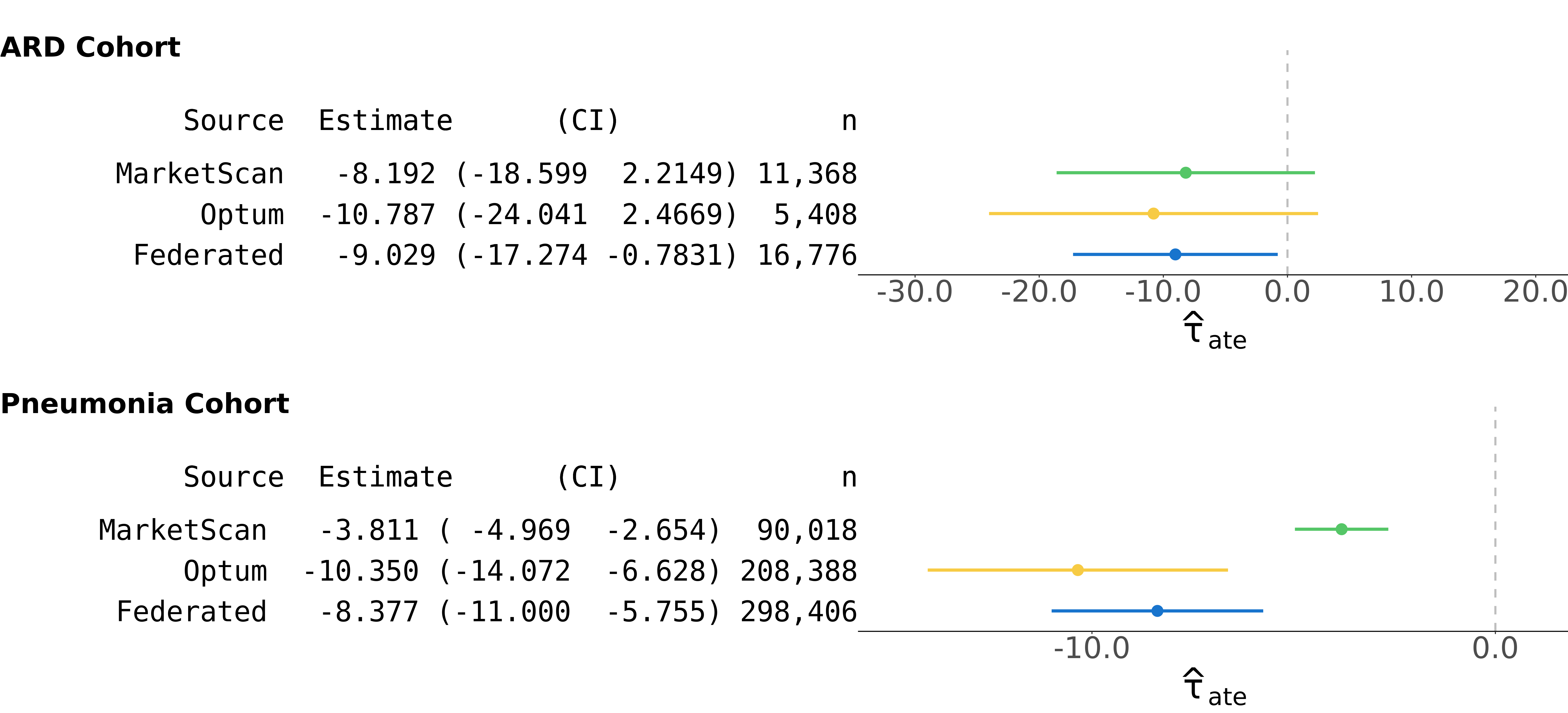}
			\caption{Unrestricted IPW-MLE (age and year dummies as unstable covariates)}
			\end{subfigure}
		\label{fig:compare-ivw-ours-aipw}
\end{figure}

\clearpage
\newpage
\subsection{Additional Simulation Results on One Medical Claims Data Set}\label{subsec:simulation-appendix}
The simulations in this section are based on various schemes of sampling from sub-cohorts that are partitioned from one patient cohort by age only. Suppose there are $D$ sub-cohorts. Then sub-cohort $j$, denoted as $\mathcal{C}_j$, has the records of patients whose age is between $(j-1)/D$ and $j/D$ percentiles of the full cohort. we consider alternative approaches to construct subsamples. The results are presented in Tables \ref{tab:varying-sampling-ratio}-\ref{tab:varying-d}, and are consistent with the results in Section \ref{subsec:empirical-sampling}. 

\paragraph{Varying Sampling Ratios of Sub-cohorts}
We construct $D = 2$ subsamples of equal size. For $\mathcal{S}_j$, $x$\% are sampled from $\mathcal{C}_j$ with replacement, and $(100-x)$\% are sampled from $\mathcal{C}_{3-j}$ with replacement, where $x \in \{50, 70, 90\}$ and $j \in \{1,2\}$. When $x = 50$, the age structure in $\mathcal{S}_1$ and $\mathcal{S}_2$ are similar; for other $x$, $\mathcal{S}_1$ has more young patients than $\mathcal{S}_2$. See Table \ref{tab:varying-sampling-ratio} for the results.


\paragraph{Varying Subsample Sizes}
We follow the same sampling schemes as \textbf{Varying Sampling Ratios of Sub-cohorts} with $x = 80$, but subsamples have unequal sizes. See Table \ref{tab:varying-sample-size} for the results.

\paragraph{Varying Number of Subsamples} We construct $D$ subsamples of equal size for $D \in \{2,3,4\}$. For $\mathcal{S}_j$, 70\% are drawn from $\mathcal{C}_j$ with replacement, and 30/(D-1)\% are drawn from $\mathcal{C}_k$ with replacement for $k \neq j$. See Table \ref{tab:varying-d} for the results.

\begin{table}[ht!]
	\centering
	\tcaptab{Comparison Between Restricted/Unrestricted Federated IPW-MLE and IVW with Corresponding Restricted/Unrestricted Benchmarks (ATT Weighting)}
	\begin{subtable}[t]{.53\textwidth}
	\centering
	\caption{Restricted Benchmarks $\hat\beta^{\bm{\s}}_{w,\BM}$, $\hat{V}^{\bm{\s}}_{w,\BM}$}
	{\footnotesize
	
	\begin{tabular}{l|r|r|r|r}
		\toprule
	&	\multicolumn{1}{c|}{$\hat\beta^{\bm{\s}}_{w,\BM}$}     & \multicolumn{1}{c|}{$\hat{\beta}_{w,\ivw}$} & \multicolumn{1}{c|}{$\hat\beta^{\bm{\s}.\pool}_{w,\HT}$} & \multicolumn{1}{c}{$\hat\beta^{\bm{\uns}.\pool}_{w,\HT}$} \\
	&	\textbf{mean}  & \textbf{MAE} &  \textbf{MAE} & \textbf{MAE}  \\ 
		\midrule
  ARD & \textsl{-0.7283} & 1.0965 & 0.0497 & 0.0506 \\ 
  PNA & \textsl{-0.2136} & 0.5260 & 0.0292 & 0.0367 \\ 
\midrule
&	\multicolumn{1}{c|}{$\hat{V}^{\bm{\s}}_{w,\BM}$}    & \multicolumn{1}{c|}{$\hat{V}_{w,\ivw}$}  & \multicolumn{1}{c|}{$\hat{V}^{\bm{\s}.\pool}_{w,\HT}$} & \multicolumn{1}{c}{$\hat{V}^{\bm{\uns}.\pool}_{w,\HT}$} \\
&	\textbf{mean}  & \textbf{MAE} &  \textbf{MAE} & \textbf{MAE}  \\ 
\midrule
  ARD & \textsl{0.0953} & 0.0691 & 0.0349 & 0.0323 \\ 
  PNA & \textsl{0.0468} & 0.0171 & 0.0084 & 0.0066 \\ 
		\bottomrule
	\end{tabular}
	}
	 \end{subtable}%
   \begin{subtable}[t]{0.53\textwidth}
  \centering
\caption{Unrestricted Benchmarks $\hat\beta^{\bm{\uns}}_{w,\BM}$, $\hat{V}^{\bm{\uns}}_{w,\BM}$}
{\footnotesize
\centering
	\begin{tabular}{l|r|r|r|r}
		\toprule
	&	\multicolumn{1}{c|}{$\hat\beta^{\bm{\uns}}_{w,\BM}$}     & \multicolumn{1}{c|}{$\hat{\beta}_{w,\ivw}$} & \multicolumn{1}{c|}{$\hat\beta^{\bm{\s}.\pool}_{w,\HT}$} & \multicolumn{1}{c}{$\hat\beta^{\bm{\uns}.\pool}_{w,\HT}$} \\
  &	\textbf{mean}  & \textbf{MAE} &  \textbf{MAE} & \textbf{MAE}  \\ 
		\midrule
  ARD & \textsl{-0.7223} & 1.1026 & 0.0483 & 0.0474 \\ 
  PNA & \textsl{-0.2142} & 0.5254 & 0.0294 & 0.0358 \\ 
\midrule
&	\multicolumn{1}{c|}{$\hat{V}^{\bm{\uns}}_{w,\BM}$}    & \multicolumn{1}{c|}{$\hat{V}_{w,\ivw}$}  & \multicolumn{1}{c|}{$\hat{V}^{\bm{\s}.\pool}_{w,\HT}$} & \multicolumn{1}{c}{$\hat{V}^{\bm{\uns}.\pool}_{w,\HT}$} \\
 &	\textbf{mean} & \textbf{MAE} &  \textbf{MAE} & \textbf{MAE}  \\ 
\midrule
  ARD & \textsl{0.0951} & 0.0689 & 0.0347 & 0.0321 \\ 
  PNA & \textsl{0.0467} & 0.0171 & 0.0084 & 0.0065 \\ 
		\bottomrule
	\end{tabular}
			}
	\end{subtable}
	\bnotetab{Subsamples are simulated from the MarketScan ARD cohort, and from the MarketScan pneumonia (PNA) cohort with $D=2$. For subsamples drawn from ARD cohort, $n_1 = n_2 = 6,000$; for subsamples drawn from PNA cohort, $n_1 = n_2 = 10,000$. These tables complement Table \ref{tab:fed-ivw-main} and follow the same sampling scheme as Table \ref{tab:fed-ivw-main}.}
	\label{tab:fed-ivw-att-weighting}
\end{table}

\begin{table}[h!]
	\centering
	\tcaptab{Varying Sampling Ratios of Sub-cohorts}
	\begin{subtable}[t]{.53\textwidth}
	\centering
	\caption{MLE: Restricted Benchmarks}
	{\footnotesize
	\begin{tabular}{l|r|r|r|r}
		\toprule
	&	\multicolumn{1}{c|}{$\hat\beta^{\bm{\s}}_{w,\BM}$}     & \multicolumn{1}{c|}{$\hat{\beta}_{w,\ivw}$} & \multicolumn{1}{c|}{$\hat\beta^{\bm{\s}.\pool}_{w,\mle}$} & \multicolumn{1}{c}{$\hat\beta^{\bm{\uns}.\pool}_{w,\mle}$} \\
	&	\textbf{mean} & \textbf{MAE} &  \textbf{MAE} & \textbf{MAE}  \\ 
		\midrule
50\%/50\% & \textsl{-0.2077} & 0.0504 & 0.0246 & 0.0264 \\ 
  70\%/30\% & \textsl{-0.1883} & 0.0586 & 0.0289 & 0.0306 \\ 
  90\%/10\% & \textsl{-0.2294} & 0.0503 & 0.0262 & 0.0268 \\ 
  \midrule
&  \multicolumn{1}{c|}{$\hat{V}^{\bm{\s}}_{w,\BM}$}    & \multicolumn{1}{c|}{$\hat{V}_{w,\ivw}$}  & \multicolumn{1}{c|}{$\hat{V}^{\bm{\s}.\pool}_{w,\mle}$} & \multicolumn{1}{c}{$\hat{V}^{\bm{\uns}.\pool}_{w,\mle}$} \\
 &	\textbf{mean} & \textbf{MAE} &  \textbf{MAE} & \textbf{MAE}  \\ 
\midrule
  50\%/50\% & \textsl{0.0515} & 0.0081 & 0.0043 & 0.0040 \\ 
  70\%/30\% & \textsl{0.0508} & 0.0086 & 0.0047 & 0.0045 \\ 
  90\%/10\% & \textsl{0.0528} & 0.0082 & 0.0049 & 0.0046 \\ 
  \bottomrule
	\end{tabular}
	}
	 \end{subtable}%
   \begin{subtable}[t]{0.53\textwidth}
  \centering
\caption{MLE: Unrestricted Benchmarks}
{\footnotesize
\centering
	\begin{tabular}{l|r|r|r|r}
		\toprule
	&	\multicolumn{1}{c|}{$\hat\beta^{\bm{\uns}}_{w,\BM}$}     & \multicolumn{1}{c|}{$\hat{\beta}_{w,\ivw}$} & \multicolumn{1}{c|}{$\hat\beta^{\bm{\s}.\pool}_{w,\mle}$} & \multicolumn{1}{c}{$\hat\beta^{\bm{\uns}.\pool}_{w,\mle}$} \\
 &	\textbf{mean}  & \textbf{MAE} &  \textbf{MAE} & \textbf{MAE}  \\ 
		\midrule
50\%/50\% & \textsl{-0.2082} & 0.0512 & 0.0251 & 0.0266 \\ 
  70\%/30\% & \textsl{-0.1887} & 0.0591 & 0.0298 & 0.0305 \\ 
  90\%/10\% & \textsl{-0.2300} & 0.0503 & 0.0267 & 0.0273 \\ 
\midrule
&	\multicolumn{1}{c|}{$\hat{V}^{\bm{\uns}}_{w,\BM}$}    & \multicolumn{1}{c|}{$\hat{V}_{w,\ivw}$}  & \multicolumn{1}{c|}{$\hat{V}^{\bm{\s}.\pool}_{w,\mle}$} & \multicolumn{1}{c}{$\hat{V}^{\bm{\uns}.\pool}_{w,\mle}$} \\
 &	\textbf{mean}  & \textbf{MAE} &  \textbf{MAE} & \textbf{MAE}  \\ 
\midrule
  50\%/50\% & \textsl{0.0516} & 0.0081 & 0.0043 & 0.0041 \\ 
  70\%/30\% & \textsl{0.0508} & 0.0087 & 0.0048 & 0.0045 \\ 
  90\%/10\% & \textsl{0.0529} & 0.0082 & 0.0049 & 0.0047 \\ 
		\bottomrule
	\end{tabular}
			}
	\end{subtable}
	\vspace{0.5cm}
	\begin{subtable}[t]{.53\textwidth}
	\centering
	\caption{IPW-MLE: Restricted Benchmarks}
	{\footnotesize
	
	\begin{tabular}{l|r|r|r|r}
		\toprule
	&	\multicolumn{1}{c|}{$\hat\beta^{\bm{\s}}_{w,\BM}$}     & \multicolumn{1}{c|}{$\hat{\beta}_{w,\ivw}$} & \multicolumn{1}{c|}{$\hat\beta^{\bm{\s}.\pool}_{w,\HT}$} & \multicolumn{1}{c}{$\hat\beta^{\bm{\uns}.\pool}_{w,\HT}$} \\
	&	\textbf{mean}  & \textbf{MAE} &  \textbf{MAE} & \textbf{MAE}  \\ 
		\midrule
50\%/50\% & \textsl{-0.2793} & 0.7466 & 0.0322 & 0.0205 \\ 
  70\%/30\% & \textsl{-0.2630} & 0.7721 & 0.0383 & 0.0262 \\ 
  90\%/10\% & \textsl{-0.3029} & 0.8316 & 0.0342 & 0.0289 \\ 
 \midrule
&	\multicolumn{1}{c|}{$\hat{V}^{\bm{\s}}_{w,\BM}$}    & \multicolumn{1}{c|}{$\hat{V}_{w,\ivw}$}  & \multicolumn{1}{c|}{$\hat{V}^{\bm{\s}.\pool}_{w,\HT}$} & \multicolumn{1}{c}{$\hat{V}^{\bm{\uns}.\pool}_{w,\HT}$} \\
&	\textbf{mean}  & \textbf{MAE} &  \textbf{MAE} & \textbf{MAE}  \\ 
\midrule
  50\%/50\% & \textsl{0.0697} & 0.0449 & 0.0167 & 0.0138 \\ 
  70\%/30\% & \textsl{0.0705} & 0.0467 & 0.0176 & 0.0148 \\ 
  90\%/10\% & \textsl{0.0720} & 0.0492 & 0.0177 & 0.0153 \\ 
		\bottomrule
	\end{tabular}
	}
	 \end{subtable}%
   \begin{subtable}[t]{0.53\textwidth}
  \centering
\caption{IPW-MLE: Unrestricted Benchmarks}
{\footnotesize
\centering
	\begin{tabular}{l|r|r|r|r}
		\toprule
	&	\multicolumn{1}{c|}{$\hat\beta^{\bm{\uns}}_{w,\BM}$}     & \multicolumn{1}{c|}{$\hat{\beta}_{w,\ivw}$} & \multicolumn{1}{c|}{$\hat\beta^{\bm{\s}.\pool}_{w,\HT}$} & \multicolumn{1}{c}{$\hat\beta^{\bm{\uns}.\pool}_{w,\HT}$} \\
  &	\textbf{mean}  & \textbf{MAE} &  \textbf{MAE} & \textbf{MAE}  \\ 
		\midrule
50\%/50\% & \textsl{-0.2746} & 0.7513 & 0.0344 & 0.0117 \\ 
  70\%/30\% & \textsl{-0.2587} & 0.7763 & 0.0373 & 0.0152 \\ 
  90\%/10\% & \textsl{-0.2993} & 0.8353 & 0.0333 & 0.0178 \\

\midrule
&	\multicolumn{1}{c|}{$\hat{V}^{\bm{\uns}}_{w,\BM}$}    & \multicolumn{1}{c|}{$\hat{V}_{w,\ivw}$}  & \multicolumn{1}{c|}{$\hat{V}^{\bm{\s}.\pool}_{w,\HT}$} & \multicolumn{1}{c}{$\hat{V}^{\bm{\uns}.\pool}_{w,\HT}$} \\
 &	\textbf{mean} & \textbf{MAE} &  \textbf{MAE} & \textbf{MAE}  \\ 
\midrule
  50\%/50\% & \textsl{0.0690} & 0.0441 & 0.0160 & 0.0130 \\ 
  70\%/30\% & \textsl{0.0690} & 0.0452 & 0.0160 & 0.0133 \\ 
  90\%/10\% & \textsl{0.0711} & 0.0484 & 0.0169 & 0.0145 \\ 
		\bottomrule
	\end{tabular}
			}
	\end{subtable}
	\bnotetab{Subsamples are sampled from the MarketScan pneumonia cohort with $D=2$ and $n_1 = n_2 = 10,000$. We use ATE weighting in IPW-MLE. The benchmark means and MAE are calculated based on 50 iterations.}
	\label{tab:varying-sampling-ratio}
\end{table}

\begin{table}[h!]
	\centering
	\tcaptab{Varying Subsample Sizes}
	\begin{subtable}[t]{.53\textwidth}
	\centering
	\caption{MLE: Restricted Benchmarks $\hat\beta^{\bm{\s}}_{w,\BM}$, $\hat{V}^{\bm{\s}}_{w,\BM}$}
	{\footnotesize
	\begin{tabular}{l|r|r|r|r}
		\toprule
	&	\multicolumn{1}{c|}{$\hat\beta^{\bm{\s}}_{w,\BM}$}     & \multicolumn{1}{c|}{$\hat{\beta}_{w,\ivw}$} & \multicolumn{1}{c|}{$\hat\beta^{\bm{\s}.\pool}_{w,\mle}$} & \multicolumn{1}{c}{$\hat\beta^{\bm{\uns}.\pool}_{w,\mle}$} \\
	&	\textbf{mean} & \textbf{MAE} &  \textbf{MAE} & \textbf{MAE}  \\ 
		\midrule
  20k10k & \textsl{-0.2379} & 0.0348 & 0.0162 & 0.0152 \\ 
  40k10k & \textsl{-0.2688} & 0.0259 & 0.0102 & 0.0092 \\ 
  \midrule
&  \multicolumn{1}{c|}{$\hat{V}^{\bm{\s}}_{w,\BM}$}    & \multicolumn{1}{c|}{$\hat{V}_{w,\ivw}$}  & \multicolumn{1}{c|}{$\hat{V}^{\bm{\s}.\pool}_{w,\mle}$} & \multicolumn{1}{c}{$\hat{V}^{\bm{\uns}.\pool}_{w,\mle}$} \\
 &	\textbf{mean} & \textbf{MAE} &  \textbf{MAE} & \textbf{MAE}  \\ 
\midrule
  20k10k & \textsl{0.0372} & 0.0043 & 0.0025 & 0.0022 \\ 
  40k10k & \textsl{0.0243} & 0.0018 & 0.0011 & 0.0010 \\ 
  \bottomrule
	\end{tabular}
	}
	 \end{subtable}%
   \begin{subtable}[t]{0.53\textwidth}
  \centering
\caption{MLE: Unrestricted Benchmarks $\hat\beta^{\bm{\uns}}_{w,\BM}$, $\hat{V}^{\bm{\uns}}_{w,\BM}$}
{\footnotesize
\centering
	\begin{tabular}{l|r|r|r|r}
		\toprule
	&	\multicolumn{1}{c|}{$\hat\beta^{\bm{\uns}}_{w,\BM}$}     & \multicolumn{1}{c|}{$\hat{\beta}_{w,\ivw}$} & \multicolumn{1}{c|}{$\hat\beta^{\bm{\s}.\pool}_{w,\mle}$} & \multicolumn{1}{c}{$\hat\beta^{\bm{\uns}.\pool}_{w,\mle}$} \\
 &	\textbf{mean}  & \textbf{MAE} &  \textbf{MAE} & \textbf{MAE}  \\ 
		\midrule
  20k10k & \textsl{-0.2377} & 0.0346 & 0.0165 & 0.0147 \\ 
  40k10k & \textsl{-0.2688} & 0.0260 & 0.0104 & 0.0091 \\ 
\midrule
&	\multicolumn{1}{c|}{$\hat{V}^{\bm{\uns}}_{w,\BM}$}    & \multicolumn{1}{c|}{$\hat{V}_{w,\ivw}$}  & \multicolumn{1}{c|}{$\hat{V}^{\bm{\s}.\pool}_{w,\mle}$} & \multicolumn{1}{c}{$\hat{V}^{\bm{\uns}.\pool}_{w,\mle}$} \\
 &	\textbf{mean}  & \textbf{MAE} &  \textbf{MAE} & \textbf{MAE}  \\ 
\midrule
  20k10k & \textsl{0.0372} & 0.0043 & 0.0025 & 0.0022 \\ 
  40k10k & \textsl{0.0243} & 0.0018 & 0.0011 & 0.0010 \\ 
		\bottomrule
	\end{tabular}
			}
	\end{subtable}
	\vspace{0.5cm}
	\begin{subtable}[t]{.53\textwidth}
	\centering
	\caption{IPW-MLE: Restricted Benchmarks $\hat\beta^{\bm{\s}}_{w,\BM}$, $\hat{V}^{\bm{\s}}_{w,\BM}$}
	{\footnotesize
	
	\begin{tabular}{l|r|r|r|r}
		\toprule
	&	\multicolumn{1}{c|}{$\hat\beta^{\bm{\s}}_{w,\BM}$}     & \multicolumn{1}{c|}{$\hat{\beta}_{w,\ivw}$} & \multicolumn{1}{c|}{$\hat\beta^{\bm{\s}.\pool}_{w,\HT}$} & \multicolumn{1}{c}{$\hat\beta^{\bm{\uns}.\pool}_{w,\HT}$} \\
	&	\textbf{mean}  & \textbf{MAE} &  \textbf{MAE} & \textbf{MAE}  \\ 
		\midrule
  20k10k & \textsl{-0.3444} & 0.6105 & 0.0527 & 0.0568 \\ 
  40k10k & \textsl{-0.3590} & 0.4971 & 0.0898 & 0.0921 \\ 

\midrule
&	\multicolumn{1}{c|}{$\hat{V}^{\bm{\s}}_{w,\BM}$}    & \multicolumn{1}{c|}{$\hat{V}_{w,\ivw}$}  & \multicolumn{1}{c|}{$\hat{V}^{\bm{\s}.\pool}_{w,\HT}$} & \multicolumn{1}{c}{$\hat{V}^{\bm{\uns}.\pool}_{w,\HT}$} \\
&	\textbf{mean}  & \textbf{MAE} &  \textbf{MAE} & \textbf{MAE}  \\ 
\midrule
  20k10k & \textsl{0.0508} & 0.0291 & 0.0105 & 0.0091 \\ 
  40k10k & \textsl{0.0342} & 0.0175 & 0.0048 & 0.0045 \\ 
		\bottomrule
	\end{tabular}
	}
	 \end{subtable}%
   \begin{subtable}[t]{0.53\textwidth}
  \centering
\caption{IPW-MLE: Unrestricted Benchmarks $\hat\beta^{\bm{\uns}}_{w,\BM}$, $\hat{V}^{\bm{\uns}}_{w,\BM}$}
{\footnotesize
\centering
	\begin{tabular}{l|r|r|r|r}
		\toprule
	&	\multicolumn{1}{c|}{$\hat\beta^{\bm{\uns}}_{w,\BM}$}     & \multicolumn{1}{c|}{$\hat{\beta}_{w,\ivw}$} & \multicolumn{1}{c|}{$\hat\beta^{\bm{\s}.\pool}_{w,\HT}$} & \multicolumn{1}{c}{$\hat\beta^{\bm{\uns}.\pool}_{w,\HT}$} \\
  &	\textbf{mean}  & \textbf{MAE} &  \textbf{MAE} & \textbf{MAE}  \\ 
		\midrule
  20k10k & \textsl{-0.3457} & 0.6092 & 0.0547 & 0.0558 \\ 
  40k10k & \textsl{-0.3598} & 0.4963 & 0.0909 & 0.0922 \\

\midrule
&	\multicolumn{1}{c|}{$\hat{V}^{\bm{\uns}}_{w,\BM}$}    & \multicolumn{1}{c|}{$\hat{V}_{w,\ivw}$}  & \multicolumn{1}{c|}{$\hat{V}^{\bm{\s}.\pool}_{w,\HT}$} & \multicolumn{1}{c}{$\hat{V}^{\bm{\uns}.\pool}_{w,\HT}$} \\
 &	\textbf{mean} & \textbf{MAE} &  \textbf{MAE} & \textbf{MAE}  \\ 
\midrule
  20k10k & \textsl{0.0506} & 0.0288 & 0.0102 & 0.0089 \\ 
  40k10k & \textsl{0.0342} & 0.0175 & 0.0049 & 0.0046 \\ 
		\bottomrule
	\end{tabular}
			}
	\end{subtable}
	\bnotetab{Subsamples are sampled from the MarketScan pneumonia cohort with $D=2$ and varying values of $n_1$ and $n_2$. In the first column of these tables, ``$x$k$y$k'' denotes $n_1 = 1000x$ and $n_2 = 1000y$ for $x \in \{20,40\}$ and $y = 10$.  We use ATE weighting in IPW-MLE. The benchmark means and MAE are calculated based on 50 iterations.}
	\label{tab:varying-sample-size}
\end{table}

\begin{table}[h!]
	\centering
	\tcaptab{Varying Number of Subsamples}
	\begin{subtable}[t]{.53\textwidth}
	\centering
	\caption{MLE: Restricted Benchmarks}
	{\footnotesize
	\begin{tabular}{l|r|r|r|r}
		\toprule
	&	\multicolumn{1}{c|}{$\hat\beta^{\bm{\s}}_{w,\BM}$}     & \multicolumn{1}{c|}{$\hat{\beta}_{w,\ivw}$} & \multicolumn{1}{c|}{$\hat\beta^{\bm{\s}.\pool}_{w,\mle}$} & \multicolumn{1}{c}{$\hat\beta^{\bm{\uns}.\pool}_{w,\mle}$} \\
	&	\textbf{mean} & \textbf{MAE} &  \textbf{MAE} & \textbf{MAE}  \\ 
		\midrule
$D=2$ & \textsl{-0.2100} & 0.0312 & 0.0168 & 0.0162 \\ 
  $D=3$ & \textsl{-0.2096} & 0.0303 & 0.0146 & 0.0140 \\ 
  $D=4$ & \textsl{-0.2482} & 0.0388 & 0.0270 & 0.0252 \\ 
  \midrule
&  \multicolumn{1}{c|}{$\hat{V}^{\bm{\s}}_{w,\BM}$}    & \multicolumn{1}{c|}{$\hat{V}_{w,\ivw}$}  & \multicolumn{1}{c|}{$\hat{V}^{\bm{\s}.\pool}_{w,\mle}$} & \multicolumn{1}{c}{$\hat{V}^{\bm{\uns}.\pool}_{w,\mle}$} \\
 &	\textbf{mean} & \textbf{MAE} &  \textbf{MAE} & \textbf{MAE}  \\ 
\midrule
  $D=2$ & \textsl{0.0342} & 0.0039 & 0.0020 & 0.0019 \\ 
  $D=3$ & \textsl{0.0228} & 0.0031 & 0.0016 & 0.0014 \\ 
  $D=4$ & \textsl{0.0176} & 0.0032 & 0.0017 & 0.0015 \\ 
  \bottomrule
	\end{tabular}
	}
	 \end{subtable}%
   \begin{subtable}[t]{0.53\textwidth}
  \centering
\caption{MLE: Unrestricted Benchmarks}
{\footnotesize
\centering
	\begin{tabular}{l|r|r|r|r}
		\toprule
	&	\multicolumn{1}{c|}{$\hat\beta^{\bm{\uns}}_{w,\BM}$}     & \multicolumn{1}{c|}{$\hat{\beta}_{w,\ivw}$} & \multicolumn{1}{c|}{$\hat\beta^{\bm{\s}.\pool}_{w,\mle}$} & \multicolumn{1}{c}{$\hat\beta^{\bm{\uns}.\pool}_{w,\mle}$} \\
 &	\textbf{mean}  & \textbf{MAE} &  \textbf{MAE} & \textbf{MAE}  \\ 
		\midrule
$D=2$ & \textsl{-0.2100} & 0.0318 & 0.0172 & 0.0162 \\ 
  $D=3$ & \textsl{-0.2098} & 0.0305 & 0.0149 & 0.0140 \\ 
  $D=4$ & \textsl{-0.2483} & 0.0391 & 0.0271 & 0.0253 \\ 

\midrule
&	\multicolumn{1}{c|}{$\hat{V}^{\bm{\uns}}_{w,\BM}$}    & \multicolumn{1}{c|}{$\hat{V}_{w,\ivw}$}  & \multicolumn{1}{c|}{$\hat{V}^{\bm{\s}.\pool}_{w,\mle}$} & \multicolumn{1}{c}{$\hat{V}^{\bm{\uns}.\pool}_{w,\mle}$} \\
 &	\textbf{mean}  & \textbf{MAE} &  \textbf{MAE} & \textbf{MAE}  \\ 
\midrule
  $D=2$ & \textsl{0.0342} & 0.0039 & 0.0020 & 0.0019 \\ 
  $D=3$ & \textsl{0.0228} & 0.0031 & 0.0016 & 0.0014 \\ 
  $D=4$ & \textsl{0.0176} & 0.0032 & 0.0017 & 0.0015 \\ 
		\bottomrule
	\end{tabular}
			}
	\end{subtable}
	\vspace{0.5cm}
	\begin{subtable}[t]{.53\textwidth}
	\centering
	\caption{IPW-MLE: Restricted Benchmarks}
	{\footnotesize
	
	\begin{tabular}{l|r|r|r|r}
		\toprule
	&	\multicolumn{1}{c|}{$\hat\beta^{\bm{\s}}_{w,\BM}$}     & \multicolumn{1}{c|}{$\hat{\beta}_{w,\ivw}$} & \multicolumn{1}{c|}{$\hat\beta^{\bm{\s}.\pool}_{w,\HT}$} & \multicolumn{1}{c}{$\hat\beta^{\bm{\uns}.\pool}_{w,\HT}$} \\
	&	\textbf{mean}  & \textbf{MAE} &  \textbf{MAE} & \textbf{MAE}  \\ 
		\midrule
$D=2$ & \textsl{-0.2757} & 0.5992 & 0.0199 & 0.0128 \\ 
  $D=3$ & \textsl{-0.2461} & 0.8752 & 0.0230 & 0.0197 \\ 
  $D=4$ & \textsl{-0.2961} & 0.9790 & 0.0308 & 0.0195 \\

\midrule
&	\multicolumn{1}{c|}{$\hat{V}^{\bm{\s}}_{w,\BM}$}    & \multicolumn{1}{c|}{$\hat{V}_{w,\ivw}$}  & \multicolumn{1}{c|}{$\hat{V}^{\bm{\s}.\pool}_{w,\HT}$} & \multicolumn{1}{c}{$\hat{V}^{\bm{\uns}.\pool}_{w,\HT}$} \\
&	\textbf{mean}  & \textbf{MAE} &  \textbf{MAE} & \textbf{MAE}  \\ 
\midrule
  $D=2$ & \textsl{0.0471} & 0.0264 & 0.0082 & 0.0067 \\ 
  $D=3$ & \textsl{0.0327} & 0.0229 & 0.0079 & 0.0067 \\ 
  $D=4$ & \textsl{0.0248} & 0.0184 & 0.0070 & 0.0058 \\  
		\bottomrule
	\end{tabular}
	}
	 \end{subtable}%
   \begin{subtable}[t]{0.53\textwidth}
  \centering
\caption{IPW-MLE: Unrestricted Benchmarks}
{\footnotesize
\centering
	\begin{tabular}{l|r|r|r|r}
		\toprule
	&	\multicolumn{1}{c|}{$\hat\beta^{\bm{\uns}}_{w,\BM}$}     & \multicolumn{1}{c|}{$\hat{\beta}_{w,\ivw}$} & \multicolumn{1}{c|}{$\hat\beta^{\bm{\s}.\pool}_{w,\HT}$} & \multicolumn{1}{c}{$\hat\beta^{\bm{\uns}.\pool}_{w,\HT}$} \\
  &	\textbf{mean}  & \textbf{MAE} &  \textbf{MAE} & \textbf{MAE}  \\ 
		\midrule
$D=2$ & \textsl{-0.2759} & 0.5990 & 0.0216 & 0.0059 \\ 
  $D=3$ & \textsl{-0.2477} & 0.8735 & 0.0243 & 0.0093 \\ 
  $D=4$ & \textsl{-0.2967} & 0.9784 & 0.0322 & 0.0125 \\

\midrule
&	\multicolumn{1}{c|}{$\hat{V}^{\bm{\uns}}_{w,\BM}$}    & \multicolumn{1}{c|}{$\hat{V}_{w,\ivw}$}  & \multicolumn{1}{c|}{$\hat{V}^{\bm{\s}.\pool}_{w,\HT}$} & \multicolumn{1}{c}{$\hat{V}^{\bm{\uns}.\pool}_{w,\HT}$} \\
 &	\textbf{mean} & \textbf{MAE} &  \textbf{MAE} & \textbf{MAE}  \\ 
\midrule
  $D=2$ & \textsl{0.0466} & 0.0259 & 0.0077 & 0.0062 \\ 
  $D=3$ & \textsl{0.0322} & 0.0224 & 0.0073 & 0.0061 \\ 
  $D=4$ & \textsl{0.0245} & 0.0181 & 0.0067 & 0.0055 \\ 
		\bottomrule
	\end{tabular}
			}
	\end{subtable}
	\bnotetab{Subsamples are sampled from the MarketScan pneumonia cohort with $D \in \{2,3,4\}$ and $n_j = 15,000$ for all $j \in \{1, \cdots, D\}$. We use ATE weighting in IPW-MLE. The benchmark means and MAE are calculated based on 50 iterations.}
	\label{tab:varying-d}
\end{table}

\clearpage
\newpage

{\blue
\subsection{Simulation Results on Model Efficiency Comparison}\label{efficiency-comparison}

We randomly sample 20,000 units (without replacement) from the Optum pneumonia patient cohort as our fixed benchmark combined data set. In each iteration, we then randomly partition this 20,000 units into $D = 2$ subsamples of size 10,000. 
We specify various sets of unrestricted covariates $= \{\emptyset, \{\text{age}\}, \{\text{age, health-related confounders}\}, \{\text{all covariates}\}\}$ and compare the empirical standard deviation of the federated estimates against restricted benchmarks (all covariates are set as restricted) under each model specification. The results are presented in Table 10.

\begin{table}[h!]
	\centering
	\tcaptab{Varying Unrestricted Model Specification for Sub-cohorts}
	\begin{subtable}[t]{.53\textwidth}
	\centering
	\caption{MLE: Restricted Benchmarks}
	{\footnotesize
	\begin{tabular}{l|r|r}
	\toprule
	\textbf{unrestricted} & \multicolumn{1}{c|}{$\hat\beta^{\bm{\uns}.\pool}_{w,\mle}$} & \multicolumn{1}{c}{$\hat{V}^{\bm{\uns}.\pool}_{w,\mle}$} \\
	\textbf{covariates} & \textbf{SD} & \textbf{SD}  \\ 
	\midrule
$\emptyset$ & 3.82 & 0.04 \\ 
\text{age} & 3.73 & 0.05 \\ 
\text{age, health-related} & 3.92 & 0.04 \\ 
\text{all covariates} & 4.21 & 0.04 \\ 
  \bottomrule
	\end{tabular}
	}
	\end{subtable}%
	\begin{subtable}[t]{0.53\textwidth}
    \centering
	\caption{IPW-MLE: Restricted Benchmarks}
	{\footnotesize
	\begin{tabular}{l|r|r}
		\toprule
	\textbf{unrestricted} & \multicolumn{1}{c|}{$\hat\beta^{\bm{\uns}.\pool}_{w,\HT}$} & \multicolumn{1}{c}{$\hat{V}^{\bm{\uns}.\pool}_{w,\HT}$} \\
	\textbf{covariates} & \textbf{SD} &  \textbf{SD}  \\ 
	\midrule
$\emptyset$ & 4.34 & 0.09 \\ 
\text{age} & 4.45 & 0.10 \\ 
\text{age, health-related} & 6.74 & 0.13 \\ 
\text{all covariates} & 10.68 & 0.17 \\ 
  \bottomrule
	\end{tabular}
	}
	\end{subtable}
	\bnotetab{From a fixed 20,000-unit sample obtained from Optum pneumonia cohort, subsamples are randomly partitioned into $D=2$ and $n_1 = n_2 = 10,000$. We use ATE weighting in IPW-MLE. The empirical standard deviation (SD) are calculated based on 50 iterations. The SD values in the table are multiplied by 1,000.}
	\label{tab:efficiency-comparison}
\end{table}
 }
\clearpage

\section{Simulations}\label{sec:simulation}
\subsection{Simulations for Finite-Sample Properties}
In this subsection, we demonstrate the finite sample properties of our asymptotic results for the federated MLE, federated IPW-MLE, and federated AIPW, and confirm our theoretical distribution results. To conserve space, we present the finite-sample results for the case in which the propensity and outcome models are stable, estimated, and correctly specified. The results for other cases are similar and available upon request. In our data generating process, $\*X_i \stackrel{\mathrm{i.i.d.}}{\sim} \mathrm{unif}(-1,1) $ is a scalar, and $Y_i$ is a binary response variable that follows
\begin{align*}
\frac{\pr(Y_i = 1\mid \*X_i, W_i)}{\pr(Y_i = 0\mid \*X_i, W_i)} =& \exp(\beta_c + \beta_w W_i + \beta_x \*X_i)  \\
\frac{\pr(W_i = 1 \mid \*X_i)}{\pr(W_i = 0 \mid \*X_i)} =& \exp(\gamma_c + \gamma_x \*X_i),
\end{align*}
where $\bm\beta_0 = [\beta_c, \beta_w, \beta_x] = [-0.2, -0.3, 0.5]$ and $\bm\gamma_0 =  [\gamma_c, \gamma_x] = [0.1, 0.2]$. We generate $n_\npool$ observations and randomly split these $n_\npool$ observations into $D$ equally-sized data sets, in which $n_\npool$ is selected at $500$, $1000$, $2000$, and $5000$, and $D$ varies from 1 to 5. Note that $D=1$ implies that we can simply apply the conventional MLE, IPW-MLE, and AIPW estimators without pooling $\bm\beta$ and $\tau_\ate$. The results for $D = 1$ serve as the benchmark to compare the results with other $D$. When $D$ varies from 2 to 5, we apply our estimation and federated methods from Section \ref{sec:estimation} to obtain the federated MLE, federated IPW-MLE, and federated AIPW estimators for $\bm\beta$ and $\tau_\ate$ and their federated variances. We calculate the standardized federated MLE estimator using $n_\npool^{1/2} (\hat{\*V}_{\bm\beta}^\pool)^{-1/2} (\hat{\bm\beta}^\pool_\mle - \bm\beta^\ast)$ based on Theorem \ref{thm:pool-mle}. Similarly, we calculate the standardized federated IPW-MLE and federated AIPW based on Theorems \ref{theorem:ht-est-prop} and \ref{theorem:pool-aipw}. 

Figure \ref{fig:histogram-mle-ht-aipw} shows the histograms of standardized federated MLE and federated IPW-MLE for the treatment coefficient $\beta_w$, as well as federated AIPW for $\tau_\ate$, for various $D$ with $n_\npool = 500$ based on 2,000 replications of the above procedure. The histograms match the
standard normal density function very well. Additionally, Table \ref{tab:simulation-mle-ht-aipw} reports the mean and standard error of the standardized federated MLE, federated IPW-MLE, and federated AIPW estimators for other $n_\npool$. Figure \ref{fig:histogram-mle-ht-aipw} and Table \ref{tab:simulation-mle-ht-aipw} show that federated estimators across data sets are very close to those estimated from the combined, individual-level data. Moreover, they support the validity of our asymptotic results in
finite samples even when $n_\npool$ is as low as $500$. A sample size of a few hundred observations for good finite sample properties can be satisfied in many empirical medical applications, such as our medical claims data in Section \ref{sec:empirical}.


		\begin{figure}[t!]
\tcaptab{Histograms of Standardized MLE, IPW-MLE, and AIPW}
			\centering
			\begin{subfigure}{0.8\textwidth}
				\centering
				\includegraphics[width=1\linewidth]{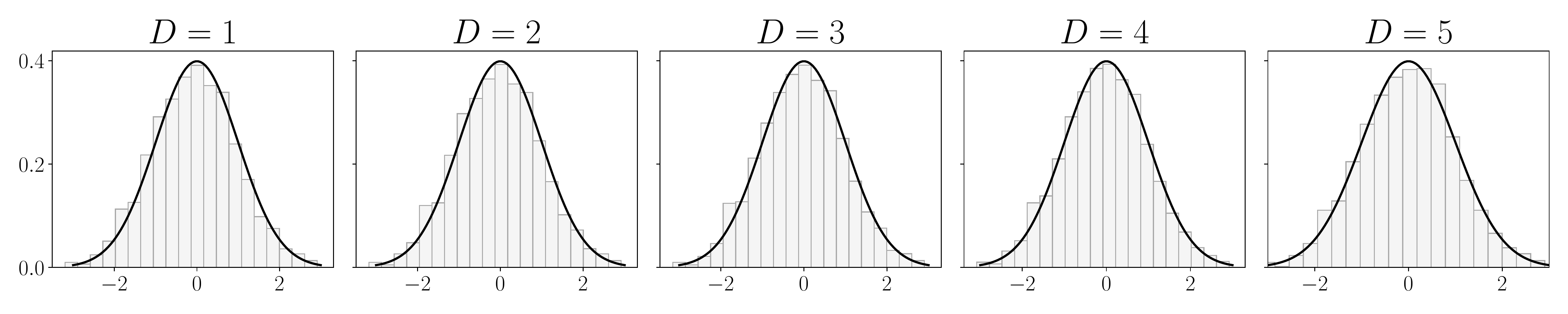}
			\caption{Federated MLE }\label{fig:mle-hist}
			\end{subfigure}
			\begin{subfigure}{0.8\textwidth}
				\centering
			\includegraphics[width=1\linewidth]{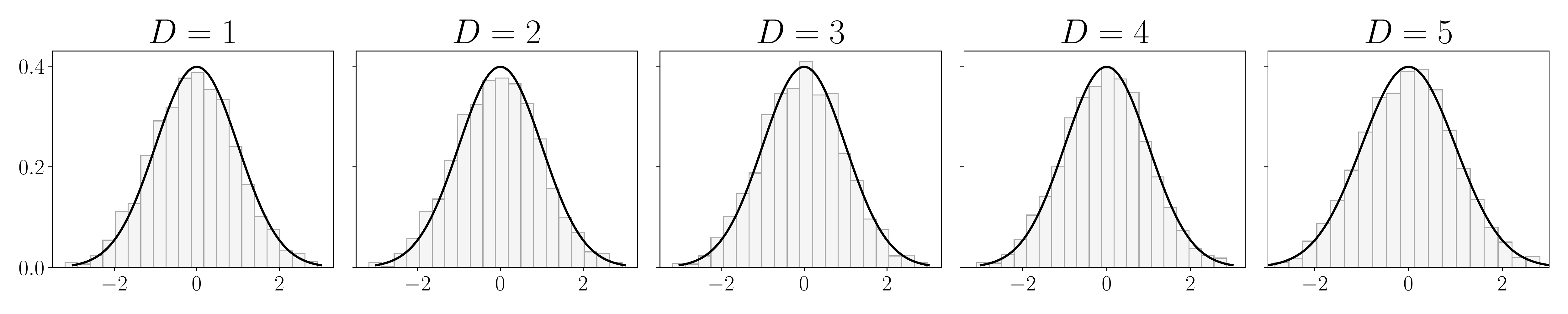}
				\caption{Federated IPW-MLE}\label{fig:ht-hist}
			\end{subfigure}
			\begin{subfigure}{0.8\textwidth}
				\centering
			\includegraphics[width=1\linewidth]{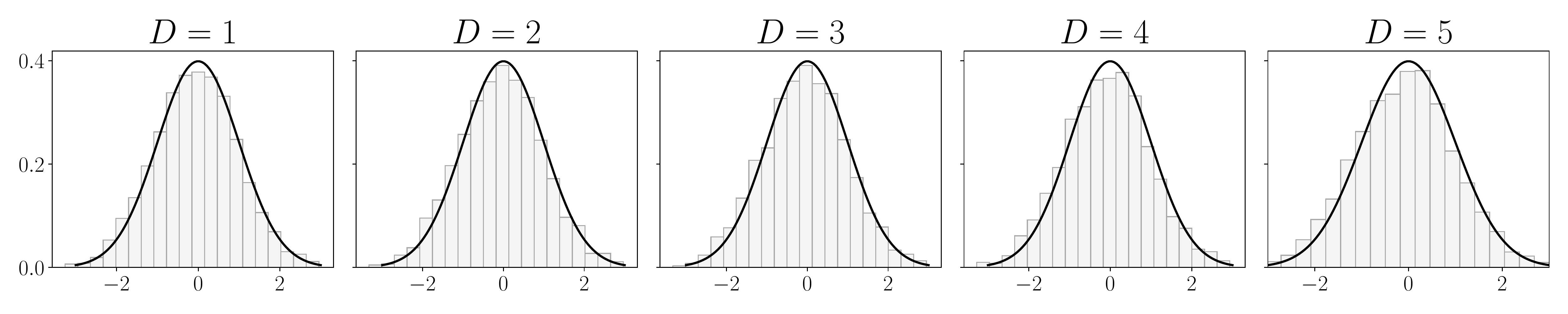}
				\caption{Federated AIPW  Estimators}\label{fig:aipw-hist}
			\end{subfigure}
			\bnotefig{These figures show the histograms of estimated MLE, IPW-MLE, and AIPW estimators normalized by their estimated standard deviations, where $n_\npool = 500$. $D$ is selected from 1 to 5, where $D = 1$ is the benchmark and implies that we estimate population parameters from the combined data. The normal density function is superimposed on the histograms. The results are based on 2,000 simulation replications. 
			}
			\label{fig:histogram-mle-ht-aipw}
		\end{figure}

\begin{table}[h]
	\centering
	\tcaptab{Simulations: Standardized Federated Maximum Likelihood Estimators}
	\begin{subtable}[t]{1\textwidth}
	\centering
	\begin{tabular}{r|rrrrrrrrrr}
		\toprule
		D & 1 &  & 2 &  & 3 &  & 4 &  & 5 &  \\ 
		\diagbox[width=3em]{$n$}{} & Mean & Std.  & Mean & Std. & Mean & Std.  & Mean & Std. & Mean & Std.   \\ 
		\midrule
		500 & -0.060 & 1.005 & -0.049 & 1.000 & -0.038 & 0.995 & -0.027 & 0.987 & -0.014 & 0.984 \\ 
		1000 & -0.011 & 0.997 & -0.004 & 0.994 & 0.004 & 0.991 & 0.010 & 0.988 & 0.018 & 0.986 \\ 
		2000 & -0.035 & 0.999 & -0.029 & 0.998 & -0.025 & 0.997 & -0.020 & 0.995 & -0.013 & 0.993 \\ 
		5000 & -0.015 & 1.019 & -0.012 & 1.019 & -0.008 & 1.018 & -0.005 & 1.017 & -0.002 & 1.017 \\
		\bottomrule
	\end{tabular}
	\caption{Federated MLE}
	\end{subtable}
	\hspace{1cm}
	\begin{subtable}[t]{1\textwidth}
	\centering
	\begin{tabular}{r|rrrrrrrrrr}
		\toprule
		D & 1 &  & 2 &  & 3 &  & 4 &  & 5 &  \\ 
		\diagbox[width=3em]{$n$}{} & Mean & Std.  & Mean & Std. & Mean & Std.  & Mean & Std. & Mean & Std.   \\ 
		\midrule
		500 & -0.058 & 1.004 & -0.047 & 1.000 & -0.036 & 0.994 & -0.025 & 0.986 & -0.012 & 0.983 \\ 
		1000 & -0.012 & 0.996 & -0.005 & 0.994 & 0.005 & 0.990 & 0.011 & 0.989 & 0.017 & 0.983 \\ 
		2000 & -0.035 & 1.000 & -0.030 & 0.999 & -0.024 & 0.997 & -0.019 & 0.998 & -0.013 & 0.995 \\ 
		5000 & -0.014 & 1.020 & -0.011 & 1.019 & -0.008 & 1.018 & -0.005 & 1.018 & -0.001 & 1.018 \\ 
		\bottomrule
	\end{tabular}
	\caption{Federated IPW-MLE}
	\end{subtable}
	\hspace{1cm}
	\begin{subtable}[t]{1\textwidth}
	\centering
	\begin{tabular}{r|rrrrrrrrrr}
		\toprule
		D & 1 &  & 2 &  & 3 &  & 4 &  & 5 &  \\ 
		\diagbox[width=3em]{$n$}{} & Mean & Std.  & Mean & Std. & Mean & Std.  & Mean & Std. & Mean & Std.   \\ 
		\midrule
		500 & -0.053 & 1.009 & -0.060 & 1.014 & -0.061 & 1.025 & -0.071 & 1.036 & -0.083 & 1.044 \\ 
		1000 & -0.004 & 0.999 & -0.008 & 1.003 & -0.013 & 1.007 & -0.015 & 1.014 & -0.019 & 1.011 \\ 
		2000 & -0.025 & 1.000 & -0.029 & 1.002 & -0.030 & 1.001 & -0.034 & 1.007 & -0.038 & 1.008 \\ 
		5000 & -0.002 & 1.020 & -0.005 & 1.020 & -0.007 & 1.022 & -0.009 & 1.022 & -0.009 & 1.023 \\ 
		\bottomrule
	\end{tabular}
	\caption{Federated AIPW}
	\end{subtable}

	\bnotetab{This table reports the mean and standard error of the standardized federated MLE and federated IPW-MLE for the treatment coefficient $\beta_w$, as well as the standardized federated AIPW for ATE $\tau_\ate$ across 2,000 simulation replications. $n_\npool$ is selected at $500$, $1000$, $2000$, and $5000$. $D$ is selected from 1 to 5, where $D = 1$ is the benchmark and implies that we estimate population parameters from the combined data. The results for the federated estimators ($D = 2, 3, 4, 5$)  are very close to the benchmarks ($D = 1$), implying the validity of our federated procedures for MLE, IPW-MLE, and AIPW. Moreover, the mean is close to 0, and the standard error is close to 1, verifying that our federated estimators have good finite sample properties. }
	\label{tab:simulation-mle-ht-aipw}
\end{table}

\clearpage

{\blue 
\subsection{Simulation for Double Robustness Property of Federated AIPW}
In this subsection, we demonstrate the double robustness property of our federated AIPW estimator under different settings of model specification. To conserve space, we present the results for the case in which the propensity and outcome models are stable, estimated, and correctly specified. We examine the performance of federated AIPW in terms of the Mean Absolute Error (MAE) with respect to the ground truth $\tau_{\ate}$ across simulations. We additionally compare the federated AIPW estimator with the two commonly used alternatives: outcome regression (OM) and inverse propensity weighting (IPW) estimators which do not have the double robustness property. In our data generating process, $\*X_i = (X_{i 1}, X_{i 2}, X_{i 3})^\T \in \mathbb{R}^3$ are i.i.d. samples where each $X_{i j} {\sim} \mathrm{unif}(-1,1) $ is a scalar for $j \in \{1, 2, 3\}$. $W_i$ is a binary treatment variable that follows
\begin{align*}
\frac{\pr(W_i = 1 \mid \*X_i)}{\pr(W_i = 0 \mid \*X_i)} =& \exp(\gamma_c + \gamma_x^\T \*X_i),
\end{align*}
where $\gamma_c = 0.1$ and $\gamma_x = [0.2, 0.3, 0.4]$. $Y_i$ is a binary response variable that follows
\begin{align*}
\frac{\pr(Y_i = 1\mid \*X_i, W_i)}{\pr(Y_i = 0\mid \*X_i, W_i)} =& \exp(\beta_c + \beta_w W_i + \beta_x^\T \*X_i),
\end{align*}
where $\beta_c = -0.2$, $\beta_w = -0.3$, $\beta_x = [0.5, 0.7, -0.6]$. 
We generate $n_\npool =20,000$ observations and randomly split these observations into $D=2$ equally-sized data sets. We evaluate the performance of the federated estimators under four settings: both outcome and propensity models are correctly specified (Setting 1); propensity model is correctly specified, but outcome model is misspecified (Setting 2); outcome model is correctly specified, but propensity model is misspecified (Setting 3); both outcome and propensity models are misspecified (Setting 4). In our simulation, misspecified models are set to include linear terms of the first two covariates ($X_{i1}$ and $X_{i2}$) and thus fail to capture the relationship of outcome (treatment) and $X_{i3}$. 

To compare federated AIPW, OM and IPW, we use different methods to estimate $\tau_\ate^\sk$, but we use the same method (i.e., IVW for the stable case) to federate the estimated $\tau_\ate^\sk$. In federated AIPW, we use our approach in Section \ref{subsec:federate-aipw}.


In federated OM, we estimate $\tau_\ate^\sk$ by 
\[\hat{\tau}_{\ate, \text{OM}}^\sk = \frac{1}{n_k} \sum_{i=1}^{n_k}\left(\hat{\mu}^\pool_{(1)}\left(\*X_i\right)-\hat{\mu}^\pool_{(0)}\left(\*X_i\right)\right),\]
where $\hat{\mu}^\pool_{(1)}\left(\*X_i\right)$ and $\hat{\mu}^\pool_{(0)}\left(\*X_i\right)$ are the federated MLE of $\mathbb{E}(Y_i\mid \*X_i, W_i= 1)$ and $\mathbb{E}(Y_i\mid \*X_i, W_i= 0)$ respectively.

In federated IPW, we estimate $\tau_\ate^\sk$ by
\[\hat{\tau}_{\ate, \text{IPW}}^\sk = \frac{1}{n_k} \sum_{i=1}^{n_k}\left(\frac{W_i Y_i}{\hat{e}^\pool\left(\*X_i\right)}-\frac{\left(1-W_i\right) Y_i}{1-\hat{e}^\pool\left(\*X_i\right)}\right)\]
where $\hat{e}^\pool\left(\*X_i\right)$ is the federated MLE of the propensity score $\pr(W_i = 1 \mid \*X_i)$. 


Table \ref{tab:double-robust-aipw} reports the MAE of federated AIPW, OM, and IPW for $D=2$ and $n_\npool = 20,000$ based on 50 replications from data generating process described above. If the outcome model is misspecified, while the propensity model is correctly specified (setting 2), the MAE of federated OM is substantially larger than that of federated AIPW. If the propensity model is misspecified, while the outcome model is correctly specified (setting 3), the MAE of federated IPW is substantially larger than that of federated AIPW. These results illustrate the double robustness property of federated AIPW. 

\begin{table}[h]
	\centering
	\tcaptab{Simulations: Federated AIPW, OM and IPW Estimators}
	\begin{subtable}[t]{1\textwidth}
	\centering
	\begin{tabular}{r|rrr}
		\toprule
		  & AIPW & OM & IPW \\ 
		  & \textbf{MAE} & \textbf{MAE} & \textbf{MAE}  \\ 
		
		\midrule
		Setting 1 & 3.672 & 3.660 & 5.953 \\ 
		Setting 2 & 3.682 & 17.815 & 5.953 \\
		Setting 3 & 3.675 & 3.660 & 6.961 \\
		Setting 4 & 17.845 & 17.815 & 6.961 \\
		
		\bottomrule
	\end{tabular}
	\end{subtable}
	\bnotetab{This table reports the MAE (values in the table are multiplied by $1,000$) of the estimated $\tau_\ate$ using federated AIPW, OM and IPW estimators across 50 simulation replications. For all simulations, $n_\npool = 20,000$ and $D = 2$. In the table, in setting 1, both outcome and propensity models are correctly specified; in setting 2, propensity model is correctly specified, but outcome model is misspecified; in setting 3, outcome model is correctly specified, but propensity model is misspecified; in setting 4, both outcome and propensity models are misspecified. The low MAE of federated AIPW estimator in the first three settings demonstrates its double robustness property as compared to the other two estimators.}
	\label{tab:double-robust-aipw}
\end{table}
}

\clearpage
\newpage

\section{Proofs}\label{sec:proof}

Let $\dot{\bm\ell}_n(\bm\beta) \coloneqq \frac{\partial \bm\ell_n(\bm\beta)}{\partial \bm\beta } $ and $ \ddot{\bm\ell}_n(\bm\beta) \coloneqq \frac{\partial^2 \bm\ell_n(\bm\beta)}{\partial \bm\beta \partial \bm\beta^\top } $ be the gradient and Hessian of the likelihood function.
Moreover, let $\bm\ell^\sk_{n_k}(\bm\beta)$, $\dot{\bm\ell}^\sk_{n_k}(\bm\beta)$, $\ddot{\bm\ell}^\sk_{n_k}(\bm\beta)$ and $\hat{\bm\beta}_\mle^\sk$ be the likelihood function, gradient, Hessian, and estimator on data set $k$. 

\subsection{Misspecified Maximum Likelihood Estimator}\label{subsec:misspecified-mle}
If the outcome model in the maximum likelihood estimator is misspecified, under suitable regularity conditions, the maximum likelihood estimator is still consistent and asymptotic normal \citep{white1982maximum}, i.e., 
\begin{align}
\sqrt{n} (\hat{\bm\beta}_\mle - {\bm\beta}^\ast ) \xrightarrow{d} \mathcal{N} \big(0, \*A_{\bm\beta^\ast}^\I \*B_{\bm\beta^\ast} \*A_{\bm\beta^\ast}^\I \big), \label{eqn:misspecified-mle}
\end{align}
where $\bm\beta^\ast$ minimizes the  Kullback-Leibler Information Criterion,
\begin{align*}
\int \log \Big( \frac{g(y \mid \*x, w)}{ f(y \mid \*x, w, \bm\beta)  }   \Big) d G(\*x, w, y).
\end{align*}
$G(\*x, w, y) $ is the cumulative density function of $(\*x, w, y) $. $g(y \mid \*x, w)$ is the population density function of $y$ given $(\*x, w) $. $\*A_{\bm\beta^\ast}$ and $\*B_{\bm\beta^\ast}$ are $\*A_{\bm\beta}$ and $\*B_{\bm\beta}$ evaluated at $\bm{\beta}^\ast$ for the definitions of $\*A_{\bm\beta}$ and $\*B_{\bm\beta}$ provided in Table \ref{tab:def-matrices}.

\subsection{Proof of Proposition \ref{proposition:adjust-data-efficiency}}

\begin{proof}[Proof of Proposition \ref{proposition:adjust-data-efficiency}]
	We adjust the covariates corresponding to $\bm\beta_{\mathrm{uns}}^\sk$ by data set. For example, in generalized linear models, we can partition the treatment and covariates into two groups, $\tilde{\*X}_i = (W_i, \*X_i) = (\tilde{\*X}_{\mathrm{s}}, \tilde{\*X}_{\mathrm{uns}})$, and include the interaction terms between $\tilde{\*X}_{\mathrm{uns}}$  and $Z_k$ in the pooled outcome model, where $Z_k$ is an binary variable indicating whether an observation is in data set $k$.

	If $Y_i$ follows a GLM, it means the conditional distribution of  $Y_i$ on $\*X_i$ and $W_i$ is in the exponential family and the log-likelihood function can be simplified to 
	\begin{align}
	\bm\ell_{n}(\bm\beta) = \sum_{i = 1 }^n \frac{Y_i \bm{\theta}_i - b(\bm{\theta}_i)}{\phi} + c(Y_i, \phi),
	\end{align}
	for a dispersion parameter $\phi$, a natural parameter $\bm\theta$, and functions $b(\bm\theta)$, and $c(Y, \phi)$.\footnote{By slight abuse of notation, $\bm\ell_{n}(\bm\beta)$ is the shorthand for $\bm\ell_{n}(\bm\beta;\phi)$, and likewise for $\dot{\bm\ell}_{n}(\bm\beta)$, $\ddot{\bm\ell}_{n}(\bm\beta)$, and $\mathcal{I}(\bm\beta)$ are similar. } Additionally, with link function $g$, we have $\+E[Y_i] = \mu_i = b^\prime(\bm{\theta}_i)$, $\tilde{\*X}_i^\T \bm\beta = g(\mu_i)$
	and 
	$\tilde{\*X}_i = (\*X_i, W_i)$. Let $h(\tilde{\*X}_i^\T \bm\beta) \coloneqq \bm{\theta}_i = (b^\prime)^\I \circ g^\I(\tilde{\*X}_i^\T \bm\beta)$. Therefore, we have $\dot{\bm\ell}_{n}(\bm\beta) = \sum_{i = 1 }^n \frac{Y_i - \mu_i}{\phi}  h^\prime(\tilde{\*X}_i^\T \bm\beta) \tilde{\*X}_i$ and 
	\[ \quad \+E[\ddot{\bm\ell}_n(\bm\beta) ] = - \frac{1}{\phi} \sum_{i = 1 }^n b^{\prime\prime}(\bm{\theta}_i) [h^\prime (\tilde{\*X}_i^\T \bm\beta) ]^2 \tilde{\*X}_i \tilde{\*X}_i^\T = -    \sum_{i = 1 }^n \underbrace{\frac{h^\prime (\tilde{\*X}_i^\T \bm\beta) }{g^\prime(\mu_i) \phi} }_{\xi_i} \tilde{\*X}_i \tilde{\*X}_i^\T =- \tilde{\*X}^\T \Xi \tilde{\*X},  \]
	where $\Xi = \diag(\xi_1, \cdots, \xi_n)$. We have $\mathcal{I}(\bm\beta) = \tilde{\*X}^\T \Xi \tilde{\*X}$ and $\Var(\hat{\bm\beta}) = (\tilde{\*X}^\T \Xi \tilde{\*X})^\I$. 
	
	Now we consider two data sets with parameters $\bm\beta^{(1)}$ and $\bm\beta^{(2)}$. 

	Suppose $\bm\beta^{(1)} = \bm\beta^{(2)}$ but we use a richer model for the pooled data that adjusts covariates by data sets, $(\tilde{\*X}_{i, \mathrm{s}}, \tilde{\*X}_{i, \mathrm{uns}} \cdot Z_1, \tilde{\*X}_{i, \mathrm{uns}} \cdot Z_2  )$, with coefficients $(\bm\beta_{\mathrm{s}}, \bm\beta^{(1)}_{\mathrm{uns}}, \bm\beta^{(2)}_{\mathrm{uns}})$,  where $Z_1$ and $Z_2$ are binary variables indicating whether an observation is in data sets 1 and 2, respectively.  We show that using this richer model gives us a less efficient estimate of $\bm\beta_{\mathrm{s}}$, where the estimator is denoted as $\bm\beta_{\mathrm{s}}^\sep$. The corresponding estimate of $\bm\beta$ from the simple model is denoted as $\bm\beta_{\mathrm{s}}^\joint$. 
	
	Next we show 
	\[ \Var(\hat{\bm\beta}_{\mathrm{s}}^\sep) \succcurlyeq \Var(\hat{\bm\beta}_{\mathrm{s}}^\joint).  \]
	Let $\tilde{\*X}_{\mathrm{s}}^\sj \in \+R^{n_j \times s_0}$ and $\tilde{\*X}_{\mathrm{uns}}^\sj \in \+R^{n_j \times (d_j-s_0)}$ be the covariate matrices of shared parameters and dataset-specific parameters on data set $j$. With algebra, we can show that 
	\begin{align*}  
	&\Var(\hat{\bm\beta}_{w}^\sep)^\I  = (\tilde{\*X}_{\mathrm{s}}^{(1)})^\T (\Xi^{(1)})^\I \tilde{\*X}_{\mathrm{s}}^{(1)} \\ & \quad - \big((\tilde{\*X}_{\mathrm{s}}^{(1)})^\T (\Xi^{(1)})^\I \tilde{\*X}_{\mathrm{uns}}^{(1)}  \big) \cdot \big(  (\tilde{\*X}_{\mathrm{uns}}^{(1)})^\T (\Xi^{(1)})^\I \tilde{\*X}_{\mathrm{uns}}^{(1)} \big)^\I \cdot \big(  (\tilde{\*X}_{\mathrm{uns}}^{(1)})^\T (\Xi^{(1)})^\I \tilde{\*X}_{\mathrm{s}}^{(1)}   \big) \\
	& \quad + (\tilde{\*X}_{\mathrm{s}}^{(2)})^\T (\Xi^{(2)})^\I \tilde{\*X}_{\mathrm{s}}^{(2)} - \big((\tilde{\*X}_{\mathrm{s}}^{(2)})^\T (\Xi^{(2)})^\I \tilde{\*X}_{\mathrm{uns}}^{(2)}    \big)  \cdot \big(  (\tilde{\*X}_{\mathrm{uns}}^{(2)})^\T (\Xi^{(2)})^\I \tilde{\*X}_{\mathrm{uns}}^{(2)}  \big)^\I \cdot \big( (\tilde{\*X}_{\mathrm{uns}}^{(2)})^\T (\Xi^{(2)})^\I \tilde{\*X}_{\mathrm{s}}^{(2)}  \big)\\
	&\Var(\hat{\bm\beta}_{w}^\joint)^\I = \big((\tilde{\*X}_{\mathrm{s}}^{(1)})^\T (\Xi^{(1)})^\I \tilde{\*X}_{\mathrm{s}}^{(1)} \\ & \quad + (\tilde{\*X}_{\mathrm{s}}^{(2)})^\T (\Xi^{(2)})^\I \tilde{\*X}_{\mathrm{s}}^{(2)}  \big) - \big((\tilde{\*X}_{\mathrm{s}}^{(1)})^\T (\Xi^{(1)})^\I \tilde{\*X}_{\mathrm{uns}}^{(1)} + (\tilde{\*X}_{\mathrm{s}}^{(2)})^\T (\Xi^{(2)})^\I \tilde{\*X}_{\mathrm{uns}}^{(2)}    \big) \\
	& \quad \cdot \big(  (\tilde{\*X}_{\mathrm{uns}}^{(1)})^\T (\Xi^{(1)})^\I \tilde{\*X}_{\mathrm{uns}}^{(1)} + (\tilde{\*X}_{\mathrm{uns}}^{(2)})^\T (\Xi^{(2)})^\I \tilde{\*X}_{\mathrm{uns}}^{(2)}  \big)^\I \cdot \big(  (\tilde{\*X}_{\mathrm{uns}}^{(1)})^\T (\Xi^{(1)})^\I \tilde{\*X}_{\mathrm{s}}^{(1)} + (\tilde{\*X}_{\mathrm{uns}}^{(2)})^\T (\Xi^{(2)})^\I \tilde{\*X}_{\mathrm{s}}^{(2)}  \big).
	\end{align*}
	In order to show $ \Var(\hat{\bm\beta}_{\mathrm{s}}^\sep) \succcurlyeq \Var(\hat{\bm\beta}_{\mathrm{s}}^\joint)$, it is equivalent to show $ \Var(\hat{\bm\beta}_{\mathrm{s}}^\sep)^\I \preccurlyeq \Var(\hat{\bm\beta}_{\mathrm{s}}^\joint)^\I$ and therefore equivalent to show for any vector $v \in R^{s_0}$, $v^\T \Var(\hat{\bm\beta}_{\mathrm{s}}^\sep)^\I  v \leq  v^\T \Var(\hat{\bm\beta}_{\mathrm{s}}^\joint)^\I v $. Let $a_1 = (\Xi^{(1)})^{-1/2} \tilde{\*X}_{\mathrm{s}}^{(1)} \cdot v $, $a_2 = (\Xi^{(2)})^{-1/2} \tilde{\*X}_{\mathrm{s}}^{(2)} \cdot v $, $\*M_1 = (\Xi^{(1)})^{-1/2} \tilde{\*X}_{\mathrm{uns}}^{(1)}  $ and $\*M_2 = (\Xi^{(2)})^{-1/2} \tilde{\*X}_{\mathrm{uns}}^{(2)}  $. With algebra, we have
	\begin{align}
	& v^\T \Var(\hat{\bm\beta}_{\mathrm{s}}^\sep)^\I  v \leq  v^\T \Var(\hat{\bm\beta}_{\mathrm{s}}^\joint)^\I v \label{ineq1}  \\
	\nonumber \Leftrightarrow& (a_1)^\T a_1 - (a_1)^\T \*M_1 \big((\*M_1)^\T \*M_1 \big)^\I (\*M_1)^\T  a_1 + (a_2)^\T a_2 - (a_2)^\T \*M_2 \big((\*M_2)^\T \*M_2 \big)^\I (\*M_2)^\T  a_2 \\
	\nonumber & \leq (a_1)^\T a_1  + (a_2)^\T a_2 -  \big((a_1)^\T \*M_1  + (a_2)^\T \*M_2  \big) \big((\*M_1)^\T \*M_1 + (\*M_2)^\T \*M_2\big)^\I \big((\*M_1)^\T  a_1  + (\*M_2)^\T  a_2\big).
	\end{align}
	Consider the SVD of $\*M_1 = \*U_1 \*D_1 \*V_1^\T \in \+R^{n_1 \times p}$ and $\*M_2 = \*U_2 \*D_2 \*V_2^\T \in \+R^{n_2 \times p}$, where $\*V_1^\I = \*V_1^\T$ and $\*V_2^\I = \*V_2^\T$ following $p \ll n_1$ and $p \ll n_2$. We can simplify the inequality \eqref{ineq1} to 
	\begin{align*}
	& a_1^\T \*U_1 \*U_1^\T  a_1 + a_2^\T \*U_2 \*U_2^\T  a_2   \\ 
	\geq&  \big(a_1^\T \*U_1 \*D_1 \*V_1^\T + a_2^\T \*U_2 \*D_2 \*V_2^\T\big) \big(\*V_1 \*D_1^2 \*V_1^\T + \*V_2 \*D_2^2 \*V_2^\T\big)^\I \big(\*V_1 \*D_1 \*U^\T_1 a_1 + \*V_2 \*D_2 \*U^\T_2 a_2 \big).
	\end{align*}
	Let $\bm\Omega= \*D_1^\I \*V_1^\I \*V_2 \*D_2$. We can write the terms in the above in equality as functions of $\bm\Omega$: 
	\begin{align*}
	\*D_1 \*V_1^\T \big(\*V_1 \*D_1^2 \*V_1^\T + \*V_2 \*D_2^2 \*V_2^\T\big)^\I \*V_1 \*D_1   =& \big( \*I + \*D_1^\I \*V_1^\I \*V_2 \*D_2^2 \*V_2^\T (\*V_1^\T)^\I \*D_1^\I \big)^\I= \big(\*I + \bm\Omega \bm\Omega^\T \big)^\I \\
	\*D_2 \*V_2^\T \big(\*V_1 \*D_1^2 \*V_1^\T + \*V_2 \*D_2^2 \*V_2^\T\big)^\I \*V_2 \*D_2   =& \big( \*I + \*D_2^\I \*V_2^\I \*V_1 \*D_1^2 \*V_1^\T (\*V_2^\T)^\I \*D_2^\I \big)^\I= \big(\*I + \*\Omega^\I  (\*\Omega^\I)^\T \big)^\I \\
	\*D_1 \*V_1^\T \big(\*V_1 \*D_1^2 \*V_1^\T + \*V_2 \*D_2^2 \*V_2^\T\big)^\I \*V_2 \*D_2 =& (\bm\Omega^\I + \bm\Omega^\T)^\I.
	\end{align*}
	Let $\tilde{a}_1 = \*U_1^\T  a_1 $ and $\tilde{a}_2 = \*U_2^\T  a_2 $,. We can further simplify Inequality \eqref{ineq1} to 
	\begin{align*}
	\tilde a_1^\T \tilde  a_1 + \tilde a_2^\T \tilde  a_2   
	\geq \tilde a_1^\T  \big(\*I + \bm\Omega \bm\Omega^\T \big)^\I  \tilde  a_1   + \tilde a_2^\T \big(\*I + \bm\Omega^\I  (\bm\Omega^\I)^\T \big)^\I  \tilde  a_2 + 2 a_1 (\bm\Omega^\I + \bm\Omega^\T)^\I a_2. 
	\end{align*}
	Consider the SVD of $\bm\Omega = \*U \*D \*V^\T$. We can further simplify Inequality \eqref{ineq1} to 
	\begin{align*}
	\tilde a_1^\T \*U \*D^2 \big(\*I + \*D^2\big)^\I \*U^\T \tilde  a_1 + \tilde a_2^\T \*V \*D^{-2} \big(I + \*D^{-2}\big)^\I  \*V^\I \tilde  a_2   \geq 2 \tilde{a}_1^\T \*U \big(\*D + \*D^\I \big)^\I \*V^\T \tilde{a}_2
	\end{align*}
	Denote each element in $\*U^\T \tilde{a}_1$ as $\bar{a}_{1,i}$ and each element in $\*V^\I \tilde  a_2$ as $\bar{a}_{2,i}$. We can further simplify Inequality \eqref{ineq1} to 
	\begin{align*}
	\sum_{i} \frac{\bar{a}_{1,i}^2 d_i^2}{1+d_i^2} + \sum_{i} \frac{\bar{a}_{2,i}^2 }{1+d_i^2} \geq 2 \sum_{i} \frac{\bar{a}_{1,i} \bar{a}_{2,i} d_i}{1+d_i^2}
	\end{align*}
	We can see that this inequality holds from the Cauchy-Schwarz inequality, and therefore Inequality \eqref{ineq1} holds. If there are more data sets,  $\Var(\hat{\bm\beta}_{\mathrm{s}}^\sep) \succcurlyeq \Var(\hat{\bm\beta}_{\mathrm{s}}^\joint)$ still holds by induction.
	
\end{proof}

\subsection{Proof of Results for Federated MLE in Section \ref{subsec:mle-results}}

\subsubsection{Proof of Theorem \ref{thm:pool-mle} (Correctly Specified and Stable Outcome Models)}
If outcome models are correctly specified, then $\bm\beta^\ast = \bm\beta_0$ and the information matrix equality holds, implying that $\*V_{\bm{\beta}} = \mathcal{I}(\bm{\beta})^\I$. For the proof in this part, we use $\bm\beta_0$ to denote the limit of (federated) MLE.
\begin{proof}[Proof of Theorem \ref{thm:pool-mle} (Correctly Specified and Stable Outcome Models)]

	
    Our proof of Theorem \ref{thm:pool-mle} consists of showing the following four equations:
    \begin{enumerate}
        \item $n_\npool^{1/2} (\hat{\*V}_{\bm\beta}^{\concat})^{-1/2} (\hat{\bm\beta}^\concat_\mle - \bm\beta_0) \xrightarrow{d} \mathcal{N} (0, \*I_d)$
        \item $n_\npool^{1/2} (\hat{\*V}_{\bm\beta}^{\pool})^{-1/2} (\hat{\bm\beta}^\pool_\mle - \bm\beta_0) \xrightarrow{d} \mathcal{N} (0, \*I_d)$
        \item $n_\npool^{1/2} (\hat{\*V}_{\bm\beta}^{\concat})^{-1/2} (\hat{\bm\beta}^\pool_\mle - \bm\beta_0) \xrightarrow{d} \mathcal{N} (0, \*I_d)$
        \item $n_\npool^{1/2} (\hat{\*V}_{\bm\beta}^{\pool})^{-1/2} (\hat{\bm\beta}^\concat_\mle - \bm\beta_0) \xrightarrow{d} \mathcal{N} (0, \*I_d)$
    \end{enumerate}
    
    We need to consider two cases. The first case is the information matrix $\mathcal{I}^\sk(\bm\beta) $ being the same for all $k$. The second case is $\mathcal{I}^\sk(\bm\beta) $ varying with $k$.

	The first step is to show $n_\npool^{1/2} (\hat{\*V}_{\bm\beta}^{\concat})^{-1/2} (\hat{\bm\beta}^\concat_\mle - \bm\beta_0) \xrightarrow{d} \mathcal{N} (0, \*I_d)$.
	MLE is consistent and asymptotic normal (see Chapter 4.2.3 in \cite{amemiya1985advanced}):
	\begin{align*}
	\hat{\bm\beta}_\mle^\sk  \xrightarrow{p}& \bm\beta_0 \\
	\sqrt{n_k} \big(\hat{\bm\beta}_\mle^\sk -  \bm\beta_0   \big) \xrightarrow{d}&\mathcal{N} \big(0, \mathcal{I}^\sk(\bm\beta_0)^\I \big),	
	\end{align*}
	where $\mathcal{I}^\sk(\bm\beta)  = - \+E_{(\*x, w, y) \sim \mathbb{P}^\sk} \Big[ \frac{\partial^2 \log  f(y_i \mid \*x_i, w_i,  \bm\beta) }{\partial \bm\beta \partial \bm\beta^\top } \Big] $. From the law of large numbers and the consistency of $\hat{\bm\beta}_\mle^\sk$, we have $- \frac{1}{n_k}  \hat{\*H}_{\bm\beta}^\sk \xrightarrow{p}  \mathcal{I}^\sk(\bm\beta_0) $. Hence, from Slutsky's theorem,  we have for each individual data set $k$,
	\begin{align*}
	\sqrt{n_k} \big(- \hat{\*H}_{\bm\beta}^\sk/n_k \big)^{-1/2} \big(\hat{\bm\beta}_\mle^\sk -  \bm\beta_0   \big) \xrightarrow{d}&\mathcal{N} \big(0,  \*I_d\big),
	\end{align*}
	where $\hat{\*H}_{\bm\beta}^\sk =  \ddot{\bm\ell}^\sk_{n_k}(\hat{\bm\beta}_\mle^\sk) $. Similarly for the combined, individual-level data, we have 
	 $\hat{\bm\beta}_\mle^\concat \xrightarrow{p} \bm\beta_0$ and $ \hat{\*V}_{\bm\beta}^{\concat} =  \big( -  \frac{1}{n_\npool} \sum_{k = 1}^D \ddot{\bm\ell}^\sk_{n_k}(\hat{\bm\beta}_\mle^\concat) \big)^\I \xrightarrow{p}\mathcal{I}^\concat(\bm\beta_0)^\I $. Then, we have
	 \begin{align*}
	     n_\npool^{1/2} (\hat{\*V}_{\bm\beta}^{\concat})^{-1/2} (\hat{\bm\beta}^\concat_\mle - \bm\beta_0) \xrightarrow{d} \mathcal{N} (0, \*I_d),
	 \end{align*}
	 which is our first equation.

    The second step is to show 
    	 \begin{align}\label{eqn:theorem1-step2}
	 & n_\npool^{1/2}  (\hat{\*V}_{\bm\beta}^\pool)^{-1/2} (\hat{\bm\beta}^\pool_\mle - \bm\beta_0) 
	  \xrightarrow{d}  \mathcal{N} \big( 0, \*I_d \big).
	 \end{align}

	Let us first consider the case where the information matrix $\mathcal{I}^\sk(\bm\beta) $ is the same for all data sets (and then follow with the case where $\mathcal{I}^\sk(\bm\beta) $ differs across data sets). Let $\mathcal{I}(\bm\beta) = \mathcal{I}^\sk(\bm\beta)$ for all $k$. In this case, $\mathcal{I}(\bm\beta) = \mathcal{I}^\concat(\bm\beta)$. Using the property that for all $k$,  $- \frac{1}{n_k} \hat{\*H}_{\bm\beta}^\sk \xrightarrow{p} \mathcal{I}(\bm\beta_0)$, we have
	 \begin{align}\label{eqn:hessian-equality}
	     \Big( \sum_{k = 1}^{D} \hat{\*H}_{\bm\beta}^\sk  \Big)^\I  \cdot \hat{\*H}_{\bm\beta}^\sj \cdot \frac{\sum_{k = 1}^D n_k}{n_j} \xrightarrow{p} \*I_d,
	 \end{align}
	 and we can use this property to show the consistency of $\hat{\bm\beta}^\pool_\mle$. Let $\hat{p}_{n,j} = \frac{n_j}{\sum_{k = 1}^D n_k}$. We have
	 \begin{align}
	 \nonumber \norm{\hat{\bm\beta}^\pool_\mle - \bm\beta_0}_2  =& \norm{\Big( \sum_{k = 1}^{D} \hat{\*H}_{\bm\beta}^\sk  \Big)^\I \Big( \sum_{k = 1}^{D} \hat{\*H}_{\bm\beta}^\sk \big(\hat{\bm\beta}_\mle^\sk  - \bm\beta_0 \big)   \Big) }_2 \\
	 \nonumber =& \norm{\sum_{j = 1}^{D} \hat{p}_{n,j} \cdot  \Big[  \Big( \sum_{k = 1}^{D} \hat{\*H}_{\bm\beta}^\sk  \Big)^\I  \cdot \hat{\*H}_{\bm\beta}^\sj \cdot \frac{1}{\hat{p}_{n,j}}  \cdot \big(\hat{\bm\beta}_\mle^\sj  - \bm\beta_0 \big)  \Big]}_2 \\
	 \leq&  \sum_{j = 1}^{D} \hat{p}_{n,j} \cdot  \underbrace{ \norm{ \Big[  \Big( \sum_{k = 1}^{D} \hat{\*H}_{\bm\beta}^\sk  \Big)^\I \cdot \hat{\*H}_{\bm\beta}^\sj \cdot \frac{1}{\hat{p}_{n,j}} \cdot  \big(\hat{\bm\beta}_\mle^\sj  - \bm\beta_0 \big)   \Big]}_2}_{o_p(1)} \label{eqn:pool-beta-mle-est-err} = o_p(1),
	 \end{align}
	 where we use the properties that $\hat{\bm\beta}_\mle^\sj   \xrightarrow{p}  \bm\beta_0 $, $0 < \hat{p}_{n,j} < 1$ and $D$ is finite.
	 
	 Since observations between data sets are asymptotically independent, we have $\Big(n_1^{1/2}  \big(\hat{\bm\beta}_\mle^{(1)} -  \bm\beta_0   \big) , n_2^{1/2} \big(\hat{\bm\beta}_\mle^{(2)} -  \bm\beta_0   \big), \cdots, n_D^{1/2} \big(\hat{\bm\beta}_\mle^{(D)} -  \bm\beta_0   \big)    \Big)$ jointly converge to a normal distribution, and for any $j \neq k$, $n_j^{1/2}  \big(\hat{\bm\beta}_\mle^{(j)} -  \bm\beta_0   \big)$ and $n_k^{1/2}  \big(\hat{\bm\beta}_\mle^{(k)} -  \bm\beta_0   \big)$ are independent. Using $n_\npool = \sum_{k = 1}^D n_k$, we can decompose $n_\npool^{1/2} \Big( \hat{\bm\beta}^\pool_\mle  -  \bm\beta_0 \Big) $ as 
	 	\begin{align*}
	n_\npool^{1/2} \Big( \hat{\bm\beta}^\pool_\mle  -  \bm\beta_0 \Big)  
	 =&  \sum_{j = 1}^{D}  \hat{p}_{n,j}^{1/2}  \underbrace{\Big[  \Big( \sum_{k = 1}^{D} \hat{\*H}_{\bm\beta}^\sk  \Big)^\I \cdot
	 \hat{\*H}_{\bm\beta}^\sj \cdot \frac{1}{\hat{p}_{n,j}} \cdot n_j^{1/2}   \big(\hat{\bm\beta}_\mle^\sj   - \bm\beta_0 \big) \Big]  }_{\coloneqq \bm\xi_{n_j}^\sj } .
	 \end{align*}
	 For the term $\bm\xi_{n_j}^\sj$ in the bracket, from Eq. \eqref{eqn:hessian-equality} and Slutsky's theorem, we have
	 \[ \bm\xi_{n_j}^\sj \xrightarrow{d}  \bm\xi^j \stackrel{d}{=} \mathcal{N}(0,  \mathcal{I}(\bm\beta_0)^\I). \] 
    As the multiplier $\hat{p}_{n,j}^{1/2}$ converges to $p^{1/2}_j$ as $n_k \rightarrow \infty$ for all $k$, from Slutsky's theorem and the delta method, we have 
  %
	  \[n_\npool^{1/2} \Big( \hat{\bm\beta}^\pool_\mle  -  \bm\beta_0 \Big)   \xrightarrow{d} \mathcal{N}(0,  \mathcal{I}(\bm\beta_0)^\I). \]

	 Next, let us consider the case, where $\mathcal{I}^\sk(\bm\beta) $ varies with the data set.
  Using the property that $\frac{1}{n_k} \hat{\*H}_{\bm\beta}^\sk \xrightarrow{p} - \mathcal{I}^\sk(\bm\beta_0)  $ and the definition $\mathcal{I}^\concat(\bm\beta_0) = \sum_{k = 1}^D p_k \mathcal{I}^\sk(\bm\beta_0)$, we have 
	\begin{align*}
	\frac{1}{\sum_{k = 1}^D n_k}  \sum_{j = 1}^D \hat{\*H}_{\bm\beta}^\sj = - \sum_{j = 1}^D  \underbrace{\frac{n_j}{\sum_{k = 1}^D n_k} }_{\hat{p}_{n,j} }  \mathcal{I}^\sj(\bm\beta_0) + o_p(1) = - \mathcal{I}^\concat(\bm\beta_0) + o_p(1).
	\end{align*}
    Since $\norm{ \mathcal{I}^\concat(\bm\beta_0)^\I \cdot \mathcal{I}^\sj(\bm\beta_0)}_2 \leq M$, we have
	\begin{align*}
	 \Big( \sum_{k = 1}^{D} \hat{\*H}_{\bm\beta}^\sk  \Big)^\I   \hat{\*H}_{\bm\beta}^\sj \cdot \frac{1}{\hat{p}_{n,j}} \xrightarrow{p} \mathcal{I}^\concat(\bm\beta_0)^\I \cdot \mathcal{I}^\sj(\bm\beta_0).
	\end{align*}
	and we can show the consistency of $\hat{\bm\beta}^\pool_\mle$ following the same procedures as Inequality \eqref{eqn:pool-beta-mle-est-err} using the property that $\norm{ \mathcal{I}^\concat(\bm\beta_0)^\I \cdot \mathcal{I}^\sj(\bm\beta_0)}_2 \leq M$. For the asymptotic normality of $\hat{\bm\beta}^\pool_\mle$, 
	since $\frac{n_j}{\sum_{k = 1}^D n_k}$ converges to some constant for all $j$, using Slutsky's Theorem and the delta method, we have
    \begin{align*}
	&n_\npool^{1/2} \Big( \hat{\bm\beta}^\pool_\mle  -  \bm\beta_0 \Big)  
	=  \sum_{j = 1}^{D} \hat{p}_{n,j}^{1/2} \Big[  \Big( \sum_{k = 1}^{D} \hat{\*H}_{\bm\beta}^\sk  \Big)^\I   \hat{\*H}_{\bm\beta}^\sj \cdot\frac{1}{\hat{p}_{n,j}} \cdot n_j^{1/2} \big(\hat{\bm\beta}_\mle^\sj   - \bm\beta_0 \big) \Big] \\
	\xrightarrow{d}& \mathcal{N} \Big( 0,\sum_{j = 1}^D p_j \cdot \mathcal{I}^\concat(\bm\beta_0)^\I \cdot \mathcal{I}^\sj(\bm\beta_0)  \cdot \mathcal{I}^\sj(\bm\beta_0)^\I   \cdot \mathcal{I}^\sj(\bm\beta_0)  \cdot \mathcal{I}^\concat(\bm\beta_0)^\I  \Big) 
	\stackrel{d}{=} \mathcal{N}  \big( 0, \mathcal{I}^\concat(\bm\beta_0)^\I   \big).
	\end{align*}
	Using the property $\big( \hat{\*V}_{\bm\beta}^\pool \big)^\I  =  \sum_{j = 1 }^D \hat{p}_{n,j} \cdot \big(  \hat{\*V}^\sj_{\bm\beta}  \big)^\I =  \sum_{j = 1 }^D p_j \mathcal{I}^\sj(\bm\beta_0) + o_p(1) = \mathcal{I}^\concat(\bm\beta_0)  + o_p(1)$, \eqref{eqn:theorem1-step2} continues to hold, and we finish showing the second step for the case where $\mathcal{I}^\sk(\bm\beta_0)$ varies with $k$.
	
	The third step is to show $n_\npool^{1/2} (\hat{\*V}_{\bm\beta}^{\concat})^{-1/2} (\hat{\bm\beta}^\pool_\mle - \bm\beta_0) \xrightarrow{d} \mathcal{N} (0, \*I_d)$. From the second step, we have shown that  $n_\npool^{1/2} \big( \hat{\bm\beta}^\pool_\mle  -  \bm\beta_0 \big) \xrightarrow{d} \mathcal{N}  \big( 0, \mathcal{I}^\concat(\bm\beta_0)^\I   \big) $ holds regardless of whether $\mathcal{I}^\sk(\bm\beta_0)$ varies with $k$. From the first step, we have $\frac{1}{n_\npool} \hat{\*V}_{\bm\beta}^{\concat} \xrightarrow{p} \mathcal{I}^\concat(\bm\beta_0)^\I $. By Slutsky's theorem, 
	\begin{align*}
	    n_\npool^{1/2} (\hat{\*V}_{\bm\beta}^{\concat})^{-1/2} (\hat{\bm\beta}^\pool_\mle - \bm\beta_0) \xrightarrow{d} \mathcal{N} (0, \*I_d)
	\end{align*}
	which completes the proof of the third step. 
	
	The last step is to show $n_\npool^{1/2} (\hat{\*V}_{\bm\beta}^{\pool})^{-1/2} (\hat{\bm\beta}^\concat_\mle - \bm\beta_0) \xrightarrow{d} \mathcal{N} (0, \*I_d)$. We have shown that $\big( \hat{\*V}_{\bm\beta}^\pool \big)^\I  = \mathcal{I}^\concat(\bm\beta_0)  + o_p(1)$ in the second step. Using this property, together with the first step, we have 
	\[n_\npool^{1/2} (\hat{\*V}_{\bm\beta}^{\pool})^{-1/2} (\hat{\bm\beta}^\concat_\mle - \bm\beta_0) \xrightarrow{d} \mathcal{N} (0, \*I_d). \]
	
	This recovers all four steps and therefore concludes the proof of Theorem \ref{thm:pool-mle}.

\end{proof}

\subsubsection{Proof of Theorem \ref{thm:pool-mle} (Misspecified and Stable Outcome Models)}
\begin{proof}[Proof of Theorem \ref{thm:pool-mle} (Misspecified and Stable Outcome Models)]
    The proof for the misspecified outcome models is the same as Theorem \ref{thm:pool-mle}, but with the limit $\bm\beta_0$ replaced by $\bm\beta^\ast$ and with $\mathcal{I}^\sk(\bm\beta)$ replaced by $(\*A^\sk_{\bm\beta})^\I \*B^\sk_{\bm\beta} (\*A^\sk_{\bm\beta})^\I$, where the definitions of $\*A^\sk_{\bm\beta}$ and $\*B^\sk_{\bm\beta}$ can be found in Table \ref{tab:def-matrices}. 
\end{proof}

\subsubsection{Proof of Theorem \ref{thm:pool-mle} (Unstable Outcome Models)}
This proof works for both correctly specified and misspecified outcome models. If outcome models are correctly specified, then $\bm\beta^\ast = \bm\beta_0$.
\begin{proof}[Proof of Theorem \ref{thm:pool-mle} (Unstable Outcome Models)]
Since the nonzero blocks $\*A^\sk_{\bm\beta,\mathrm{s},\mathrm{s}}$, $\*A^\sk_{\bm\beta,\mathrm{s},\mathrm{uns}}$, and $\*A^\sk_{ \bm\beta,\mathrm{uns},\mathrm{uns}}$ in $\*A^{\pad,\sk} $ can be consistently estimated, $\*A^{\pad,\sk} $ can be consistently estimated for all $k$. 
Hence, our pooling procedure provides a consistent estimator for $\*A^{\concat} $ (and similarly for $\*B^{\concat} $), where $\*A^{\concat} $ is defined as 
$\*A^{\concat} = \sum_{k=1}^D p_k \*A^{\pad,\sk}$ (and $\*B^{\concat} $ is defined similarly). In the case where the outcome model is correctly specified, $\*A^{\concat} =\*B^{\concat}  $. 

Let $\hat{\bm\beta}^{\concat}_\mle$ be the estimator that maximizes the likelihood function $\bm\ell^\concat_{n_\npool}(\bm\beta^\concat) $ for the combined, individual-level data, where the true parameter is $\bm\beta^\ast$. We have
\begin{align*}
    n_\npool^{1/2} \big(\hat{\bm\beta}^{\concat}_\mle - \bm\beta^\ast \big) \xrightarrow{d} \mathcal{N} \big(0, (\*A^\concat)^\I \*B^\concat  (\*A^\concat)^\I \big).
\end{align*}

Recall that $\sqrt{n_k} \big(\hat{\bm\beta}_\mle^\sk -  \bm\beta^{\sk\ast} \big) \xrightarrow{d}\mathcal{N} \big(0, (\*A^\sk)^\I \*B^\sk (\*A^\sk)^\I \big)$ and $- \hat{\*H}^\sk_{\bm\beta} /n_k \xrightarrow{p} \*A^\sk$. From Slutsky's theorem, we have $n_k^{-1/2} \hat{\*H}^\sk_{\bm\beta} \big(\hat{\bm\beta}_\mle^\sk -  \bm\beta^{\sk\ast}   \big) \xrightarrow{d}\mathcal{N} \big(0, \*B^\sk\big)$, and then we have
\[ n_k^{-1/2} \hat{\*H}^{\pad,\sk}_{\bm\beta} \big(\hat{\bm\beta}_\mle^{\pad,\sk} -  \bm\beta^{\pad,\sk\ast}   \big) \xrightarrow{d}\mathcal{N} \big(0, \*B^{\pad,\sk}\big), \]
using the property that $-\hat{\*H}^{\pad,\sk}_{\bm\beta}/n_k \xrightarrow{p} \*A^{\pad,\sk}$. 
Moreover,  we have 
\[ n_\npool \cdot n_k^{-1/2} \Big(\sum_{j = 1}^D \hat{\*H}^{\pad,\sj}_{\bm\beta} \Big)^\I  \hat{\*H}^{\pad,\sk}_{\bm\beta} \big(\hat{\bm\beta}_\mle^{\pad,\sk} -  \bm\beta^{\pad,\sk\ast} \big) \xrightarrow{d}\mathcal{N} \big(0,  (\*A^\concat)^\I \*B^{\pad,\sk} (\*A^\concat)^\I \big), \]
which follows from $-\frac{1}{n_\npool} \sum_{j = 1}^D \hat{\*H}^{\pad,\sj}_{\bm\beta}= -\sum_{j = 1}^D \frac{n_j}{n_\npool} \hat{\*H}^{\pad,\sj}_{\bm\beta}/n_j  \xrightarrow{p} \sum_{j = 1}^D p_j \*A^{\pad,\sj} = \*A^\concat $.

Note that we have the equality that $\hat{\*H}^{\pad,\sk}_{\bm\beta} \bm\beta^{\pad,\sk\ast}   = \hat{\*H}^{\pad,\sk}_{\bm\beta} \bm\beta^\ast$. This equality follows from the fact that for all the nozero entries in $\bm\beta^\ast - \bm\beta^{\pad,\sk\ast}$, the corresponding columns in $\hat{\*H}^{\pad,\sk}_{\bm\beta}$ are 0. Then, we can decompose $\bm\beta^\ast$ as 
\begin{align*}
    \bm\beta^{\concat}_0 = \sum_{k = 1}^D \bigg(\sum_{j = 1}^D \hat{\*H}^{\pad,\sj}_{\bm\beta} \bigg)^\I \hat{\*H}^{\pad,\sk}_{\bm\beta}   \bm\beta^{\pad,\sk\ast} + o_p(1).
\end{align*}
Now we are ready to show the asymptotic distribution of $\hat{\bm\beta}^\pool_\mle$: 
\begin{align*}
& n_\npool^{1/2}\Big( \hat{\bm\beta}^\pool_\mle - \bm\beta^\ast  \Big)
= \sum_{k = 1}^D  \frac{n_k^{1/2}}{n_\npool^{1/2}} \frac{n_\npool}{n_k^{1/2}} \bigg(\sum_{j = 1}^D \hat{\*H}^{\pad,\sj}_{\bm\beta} \bigg)^\I  \hat{\*H}^{\pad,\sk}_{\bm\beta} \big(\hat{\bm\beta}_\mle^{\pad,\sk} -  \bm\beta^{\pad,\sk}_0 \big) \\
\xrightarrow{d}& \mathcal{N} \Big(0,  (\*A^\concat)^\I \Big( \sum_{k = 1}^D  p_k \*B^{\pad,\sk} \Big) (\*A^\concat)^\I  \Big)  \stackrel{d}{=} \mathcal{N} \Big(0,  (\*A^\concat)^\I \*B^\concat (\*A^\concat)^\I \Big)  + o_p(1).
\end{align*}
Hence, we have  $n_\npool^{1/2}\Big( \hat{\bm\beta}^\pool_\mle - \bm\beta^\ast  \Big) \stackrel{d}{=}  n_\npool^{1/2}\Big( \hat{\bm\beta}^\concat_\mle - \bm\beta^\ast  \Big)$. Our federation procedures provide consistent estimators for $\*A^{\concat} $ and $\*B^{\concat} $. Then, we follow the same procedures and can show that the four steps in the proof of Theorem \ref{thm:pool-mle} continue to hold (even with a misspecified outcome model). 
\end{proof}

\subsubsection{Proof of Proposition \ref{proposition:model-shift-mle-weighted}}
\begin{proof}[Proof of Proposition \ref{proposition:model-shift-mle-weighted}]
For each data set $k$, if the outcome model is correctly specified, then the MLE estimator satisfies
	\begin{align*}
\sqrt{n_k} \big(\hat{\bm\beta}_\mle^\sk -  \bm\beta_0^\sk   \big) \xrightarrow{d}&\mathcal{N} \Big(0,  \mathcal{I}^\sk(\bm\beta_0)^\I \Big).	
\end{align*}
In this proof, let $\*H^\sk(\bm\beta) = \sum_{i = 1 }^{n_k}  \frac{\partial^2}{\partial \bm\beta \partial \bm\beta^\T} \log f(Y^\sk_i \mid  \*X^\sk_i, W^\sk_i, {\bm\beta})  $.
From the mean value theorem, on each data set $k$, we have 
\begin{align*}
\Big(\frac{1}{n_k} \ddot{\bm\ell}^\sk_{n_k}(\hat{\bm\beta}_\mle^\sk) \Big) \big(\hat{\bm\beta}_\mle^\sk - \bm\beta_0^\sk  \big) = - \frac{1}{n_k} \dot{\bm\ell}^\sk_{n_k}({\bm\beta}_0^\sk) + o_p \Big( \frac{1}{\sqrt{n_k}} \Big),
\end{align*}
and the above equation holds with $\ddot{\bm\ell}^\sk_{n_k}(\hat{\bm\beta}_\mle^\sk) $ replaced by $\ddot{\bm\ell}^\sk_{n_k}({\bm\beta}_0^\sk)$. Since $ \hat{\*H}_{\bm\beta}^\sk =  \ddot{\bm\ell}^\sk_{n_k}(\hat{\bm\beta}_\mle^\sk) $ for all $k$, we have
\begin{align*}
&\frac{1}{n_\npool}  \sum_{k = 1}^D  \Big(  \ddot{\bm\ell}^\sk_{n_k}(\hat{\bm\beta}_\mle^\sk)  \big(\hat{\bm\beta}_\mle^\sk - \bm\beta^\dagger  \big)  \Big)  \\ =& - \frac{1}{n_\npool} \sum_{k = 1}^D \bigg( \dot{\bm\ell}^\sk_{n_k}( \bm\beta_0^\sk) - \ddot{\bm\ell}^\sk_{n_k}( \bm\beta_0^\sk)  \big({\bm\beta}_0^\sk - \bm\beta^\dagger  \big) \bigg) + o_p \big( n_\npool^{-1/2} \big) \\
=& - \frac{1}{n_\npool} \sum_{k = 1}^D \big( \dot{\bm\ell}^\sk_{n_k}( \bm\beta_0^\sk) -   \ddot{\bm\ell}^\sk_{n_k}( \bm\beta_0^\sk)  \cdot {\bm\beta}_0^\sk  \big)  - \Big(\frac{1}{n_\npool}  \sum_{j = 1}^D  \ddot{\bm\ell}^\sj_{n_j} ( {\bm\beta}^\dagger) \Big)  \bm\beta^\dagger  + o_p \big( n_\npool^{-1/2} \big) \\
=& - \frac{1}{n_\npool} \sum_{k = 1}^D \big( \dot{\bm\ell}^\sk_{n_k}( \bm\beta^\dagger) -   \ddot{\bm\ell}^\sk_{n_k}( \bm\beta^\dagger)  \cdot {\bm\beta}^\ast  \big)  - \Big(\frac{1}{n_\npool}  \sum_{j = 1}^D  \ddot{\bm\ell}^\sj_{n_j} ( {\bm\beta}^\ast) \Big)  \bm\beta^\ast  + o_p \big( n_\npool^{-1/2} \big) \\
=& - \frac{1}{n_\npool} \sum_{k = 1}^D \dot{\bm\ell}^\sk_{n_k}( \bm\beta^\ast)   + o_p \big( n_\npool^{-1/2} \big), 
\end{align*}
where the first equality follows from that $p_k = \lim n_k/n_\npool$ is bounded away from 0 and 1,  the second equality follows from the assumption that $\mathcal{I}^\sj(\bm\beta)$ not depending on $\bm\beta$ (recall $\ddot{\bm\ell}^\sj_{n_j}({\bm\beta}) /n_j \xrightarrow{p} \mathcal{I}^\sj(\bm\beta)$), and  the third equality follows from the assumption that $\dot{\*d}^\sj_y(\bm\beta)  - \mathcal{I}^\sj(\bm\beta) \cdot \bm\beta$  not depending on $\bm\beta$ (recall $\dot{\bm\ell}^\sj_{n_j}({\bm\beta}) /n_j \xrightarrow{p} \dot{\*d}^\sj_y(\bm\beta)$). Hence we have
\begin{align*}
& n_\npool^{1/2}\Big( \hat{\bm\beta}^\pool_\mle   - \bm\beta^\ast \Big) = n_\npool^{1/2} \Big(\sum_{k = 1}^{D} \hat{\*H}_{\bm\beta}^\sk \Big)^\I   \sum_{j = 1}^{D}  \Big[   \hat{\*H}_{\bm\beta}^\sj \big(\hat{\bm\beta}_\mle^\sj   - \bm\beta^\dagger \big) \Big] 
\\  =& -  \Big(\frac{1}{n_\npool}  \sum_{k = 1}^D \ddot{\bm\ell}^\sk_{n_k}( \bm\beta^\dagger)  \Big)^\I \frac{1}{n_\npool^{1/2}} \sum_{k = 1}^D \dot{\bm\ell}^\sk_{n_k}( \bm\beta^\dagger)   + o_p (1) 
\xrightarrow{d}  \mathcal{N} \big(0, \*A^\concat(\bm\beta^\dagger)^\I \*B^\concat(\bm\beta^\dagger) \*A^\concat(\bm\beta^\dagger)^\I \big)
\end{align*}
following \eqref{eqn:misspecified-mle} in Appendix \ref{subsec:misspecified-mle}, $\*A^\concat(\bm\beta^\dagger) = \sum_{k = 1}^D p_k  \mathcal{I}^\sk(\bm\beta^\dagger ) $ and $\*B^\concat(\bm\beta^\dagger) = \sum_{k = 1}^D p_k \+E \big[ \dot{\bm\ell}^\sk_{n_k}( \bm\beta^\dagger)  \dot{\bm\ell}^\sk_{n_k}( \bm\beta^\dagger)^\T   \big]$.
We then complete the proof of Proposition \ref{proposition:model-shift-mle-weighted}. 
\end{proof}

\subsection{Proof of Results for Federated IPW-MLE in Section \ref{subsec:ht-results}}

\subsubsection{Proof of Lemma \ref{lemma:expression-ht-var}}
\begin{proof}[Proof of Lemma \ref{lemma:expression-ht-var}]
	Suppose the propensity model is the same across all data sets. Let us first show the asymptotic distribution for $	\hat{\bm\beta}_\HT$ when the propensity is estimated. We parameterize the propensity score as $\pr(W_i \mid \*X_i) = e(\*X_i, {\bm\gamma})$, and the corresponding maximum likelihood estimator is denoted as $\hat{\bm\gamma}$. Furthermore, we denote the likelihood of $W_i$ given $\*X_i$ and ${\bm\gamma}$  as $ \pr(W_i \mid \*X_i, \bm\gamma)$, and then we have $e(\*X_i, {\bm\gamma}) =  \pr(W_i = 1 \mid \*X_i, \bm\gamma)$.
	
	It is possible for $e(\*x, {\bm\gamma})$ to be misspecified. In this case, under regularity conditions in \cite{white1982maximum}, $\hat{\bm\gamma}_\mle$ is consistent and asymptotically normal: 
	\begin{align}
	\sqrt{n} (\hat{\bm\gamma}_\mle - {\bm\gamma}^\ast ) \xrightarrow{d} \mathcal{N} \big(0, \*V_{\bm\gamma^\ast} \big),
	\end{align}
	where $\bm\gamma^\ast$ minimizes the Kullback-Leibler Information Criterion between the true model and the parameterized model $e(\*X_i, {\bm\gamma}^\ast)$,  and $\*V_{\bm\gamma^\ast}  = \*A_{\bm\gamma^\ast}^\I \*B_{\bm\gamma^\ast} \*A_{\bm\gamma^\ast}^\I$. $\*A_{\bm\gamma^\ast}$ is $\*A_{\bm\gamma}$ evaluated at $\bm\gamma^\ast $ with the definition of $\*A_{\bm\gamma}$ provided in Table \ref{tab:def-matrices}, and likewise for $\*B_{\bm\gamma^\ast} $.

	Note that $	\hat{\bm\beta}_\HT$ satisfies the first order condition of the objective function \eqref{eqn:obj-ht}. With probability approaching one, we have the mean value expansion of the first order condition (or score)  at $\bm\beta_0$ of:  
	\[0 = \frac{1}{\sqrt{n}} \sum_{i = 1 }^n \varpi_{i,\hat{e}} \*g(\*X_i, W_i, \bm\beta^\ast)  + \bigg(\frac{1}{n} \sum_{i = 1 }^n \varpi_{i,\hat{e}} \ddot{\*H}(\*X_i, W_i, \tilde{\bm\beta})    \bigg) \sqrt{n} \big( \hat{\bm\beta}_\HT - \bm\beta^\ast \big),  \]
	where  $\*g_i \coloneqq \*g(\*X_i, W_i, \bm\beta^\ast)  =  \frac{\partial}{\partial \bm\beta} \log f(Y_i \mid  \*X_i, W_i, {\bm\beta}^\ast)$, $\ddot{\*H}(\*X_i, W_i, \tilde{\bm\beta})  =  \frac{\partial^2}{\partial \bm\beta \partial \bm\beta^\T} \log f(Y_i \mid  \*X_i, W_i, \tilde{\bm\beta}) $ with $\tilde{\bm\beta}$ lying between $\hat{\bm\beta}_\HT$ and $\bm\beta^\ast$, and $\varpi_{i, \hat{e}} = \frac{W_i}{\hat{e}(\*X_i)} + \frac{1 - W_i}{1 - \hat{e}(\*X_i)}$ for ATE weighting or $\varpi_{i, \hat{e}} = W_i + \frac{\hat{e}(\*X_i)}{1 - \hat{e}(\*X_i)} (1 - W_i)$ for ATT weighting.

	By the uniform weak law of large numbers, we have 
	\[  \sqrt{n} \big( \hat{\bm\beta}_\HT - \bm\beta^\ast \big) = -  \*A_{\bm{\beta}_0,\varpi}^\I \bigg(  \frac{1}{\sqrt{n}} \sum_{i = 1 }^n \varpi_{i,\hat{e}} \*g_i \bigg) + o_p(1), \]
	where $\*A_{\bm{\beta}_0,\varpi} = \frac{1}{n} \sum_{i = 1 }^n \varpi_{i,\hat{e}} \ddot{\*H}(\*X_i, W_i, \bm\beta^\ast)$. 
	The next step is to use the mean value expansion on $\frac{1}{\sqrt{n}} \sum_{i = 1 }^n \varpi_{i,\hat{e}} \*g_i $ at $\bm\gamma^\ast$; we have
	\[\frac{1}{\sqrt{n}} \sum_{i = 1 }^n \varpi_{i,\hat{e}}  \*g_i = \frac{1}{\sqrt{n}} \sum_{i = 1 }^n \underbrace{\varpi_{i,e_{\bm\gamma^\ast}} \*g_i}_{\coloneqq \*k_i}   + \+E\Big[ \*g_i \Big( \frac{\partial \varpi_{i, e_{\bm\gamma}}}{\partial  \bm\gamma} \Big|_{\bm\gamma= \bm\gamma^\ast}\Big)^\T  \Big] \sqrt{n} (\hat{\bm\gamma}_\mle - \bm\gamma^\ast)  + o_p(1),\]
	where $\frac{\partial \varpi_{i, e_{\bm\gamma}}}{\partial  \bm\gamma} \Big|_{\bm\gamma= \bm\gamma^\ast}$ is the first order derivative of $\varpi_{i, e_{\bm\gamma}}$ with respect to $\bm\gamma$ evaluated at $\bm\gamma^\ast$. In order to show the asympototic distribution of $\hat{\bm{\beta}}_{\HT}$, we need to show the asymptotic distribution of $\frac{1}{\sqrt{n}} \sum_{i = 1 }^n \varpi_{i,\hat{e}}  \*g_i$. We analyze the leading terms in the above equation one by one. 
	
	Let us first consider the ATE weighting. In this case, $\varpi_{i, e} = \frac{W_i}{e(\*X_i, \bm\gamma)} + \frac{1 - W_i}{1 - e(\*X_i, \bm\gamma)}$ and 
	\[ \frac{\partial \varpi_{i, e_{\bm\gamma}}}{\partial  \bm\gamma} \Big|_{\bm\gamma= \bm\gamma^\ast} = - \frac{W_i}{(e_i^\ast)^2}  \frac{\partial e(\*X_i, \bm\gamma^\ast)}{\partial  \bm\gamma} - \frac{1 - W_i}{(1-e_i^\ast)^2}   \frac{\partial (1 - e(\*X_i, \bm\gamma^\ast))}{\partial  \bm\gamma},  \]
	where $e^\ast_i = e(\*X_i, \bm\gamma^\ast)$. Under Asssumption \ref{assumption:parametric} and the asymptotic distribution \eqref{eqn:misspecified-mle} in Appendix \ref{subsec:misspecified-mle}, we have 
	\[ \sqrt{n} \big( \hat{\bm\gamma}_\mle - \bm\gamma^\ast  \big) = \*A_{\bm\gamma^\ast}^\I \cdot \frac{1}{\sqrt{n}} \sum_{i = 1 }^n \*d_i + o_p(1), \]
	where $\*A_{\bm\gamma^\ast}$ is $\*A_{\bm\gamma}$ evaluated at $\bm{\gamma}^\ast$, the definition of $\*A_{\bm\gamma}$ can be found in Table \ref{tab:def-matrices}, and $\*d_i$ is defined as 
	\[ \*d_i = \frac{W_i}{e_i^\ast} \frac{\partial e(\*X_i, \bm\gamma^\ast)}{\partial \bm\gamma} -  \frac{1 - W_i}{1 - e_i^\ast} \frac{\partial e(\*X_i, \bm\gamma^\ast)}{\partial \bm\gamma},   \]
	which is the first order derviative (or score) of the binary response (treatment variable $W_i$) evaluated at $\bm\gamma^\ast$. If $e(\*X_i, \bm\gamma)$ is correctly specified, we have $ \*A_{\bm\gamma^\ast} =  \+E[ \*d_i \*d_i^\T]$. Using $W_i(1 - W_i) = 0$, we have 
	$W_i \Big( \frac{\partial \varpi_{i, e_{\bm\gamma}}}{\partial  \bm\gamma} \Big|_{\bm\gamma= \bm\gamma^\ast}\Big) = - \frac{W_i}{e_i^\ast} \*d_i$ and $(1 - W_i )\Big( \frac{\partial \varpi_{i, e_{\bm\gamma}}}{\partial  \bm\gamma} \Big|_{\bm\gamma= \bm\gamma^\ast}\Big) = - \frac{1 -W_i}{1 - e_i^\ast} \*d_i$. Therefore,
	\[ \+E\Big[ \*g_i \Big( \frac{\partial \varpi_{i, e_{\bm\gamma}}}{\partial  \bm\gamma} \Big|_{\bm\gamma= \bm\gamma^\ast}\Big)^\T  \Big]   = - \+E \Big[ \underbrace{\Big( \frac{W_i}{e_i^\ast}   + \frac{1 - W_i}{1 - e_i^\ast}\Big) \*g_i}_{\*k_i}  \*d_i^\T    \Big]. \]
	Collecting terms together, we have shown
	\begin{align}\label{eqn;asymptotic-expansion-for-beta}
	\sqrt{n} \big( \hat{\bm\beta}_\HT - \bm\beta^\ast \big) =  - \*A_{\bm{\beta}^\ast, \varpi}^\I  \Bigg(  \frac{1}{\sqrt{n}} \sum_{i = 1 }^n \*k_i - \+E \big[\*k_i \*d_i^\T \big] \*A_{\bm\gamma^\ast}^\I \cdot \frac{1}{\sqrt{n}} \sum_{i = 1 }^n  \*d_i  \Bigg)+ o_p(1).
	\end{align}
	Since the standard unconfoundedness assumption holds (stated in Section \ref{subsec:model-setup}), the randomness of $\*k_i$ comes from the residual in $Y_i$, and the randomness of $\*d_i$ comes from the residual in $W_i$, and we have $\*k_i$ uncorrelated with $\*d_j$ for any $i$ and $j$ (including the case where $i$ and $j$ are the same). In addition, observations are i.i.d., $\*k_i$ is uncorrelated with $\*k_j$, and $\*d_i$ is uncorrelated with $\*d_j$ for $i \neq j$. Then, we have 
	\begin{align*}
	\*V_{\bm\beta^\ast, \HT, \hat{e}}^\dagger   = \*A_{\bm{\beta}^\ast, \varpi}^\I  \Big(  \underbrace{\+E \big[\*k_i \*k_i^\T \big] }_{\*D_{\bm{\beta}^\ast, \varpi}}   - \underbrace{\+E \big[\*k_i \*d_i^\T \big] }_{\*C_{\bm{\beta}^\ast, \varpi}}  \*V_{\bm\gamma}  \underbrace{\+E \big[\*d_i \*k_i^\T \big]}_{\*C^\T_{\bm{\beta}^\ast, \varpi}}  \Big) \*A_{\bm{\beta}^\ast, \varpi}^\I, 
	\end{align*}
	where $\*V_{\bm\gamma}  = \*A_{\bm\gamma^\ast}^\I \*B_{\bm\gamma^\ast} \*A_{\bm\gamma^\ast}^\I $. If $e(\*X_i, \bm\gamma)$ is correctly specified, we have  $\*V_{\bm\gamma^\ast} = \+E[ \*d_i \*d_i^\T]^\I = \*A_{\bm\gamma^\ast}^\I$.

	If we use the true propensity score, then 
	\[\frac{1}{\sqrt{n}} \sum_{i = 1 }^n \varpi_{i,\hat{e}}  \*g_i = \frac{1}{\sqrt{n}} \sum_{i = 1 }^n \varpi_{i,e} \*g_i  + o_p(1)\]
	and 
	\begin{align*}
	\*V_{\bm\beta^\ast, \HT, e}^\dagger   = \*A_{\bm{\beta}^\ast, \varpi}^\I  \underbrace{\+E \big[\*k_i \*k_i^\T \big] }_{\*D_{\bm{\beta}^\ast, \varpi}} \*A_{\bm{\beta}^\ast, \varpi}^\I. 
	\end{align*}
	
	Next, let us consider the ATT weighting. In this case, $\varpi_{i, e} = W_i + \frac{e(\*X_i,  \bm\gamma)}{1 - e(\*X_i,  \bm\gamma)} (1 - W_i)$ and 
    \[ \frac{\partial \varpi_{i, e_{\bm\gamma}}}{\partial  \bm\gamma} \Big|_{\bm\gamma= \bm\gamma^\ast} = \frac{1 - W_i}{(1-e_i^\ast)^2}   \frac{\partial  e(\*X_i, \bm\gamma^\ast)}{\partial  \bm\gamma}.  \]
Using $W_i(1 - W_i) = 0$, we have 
	$ \frac{\partial \varpi_{i, e_{\bm\gamma}}}{\partial  \bm\gamma} \Big|_{\bm\gamma= \bm\gamma^\ast}= - \frac{1 -W_i}{1 - e_i^\ast} \*d_i$. Therefore,
	\[ \+E\Big[ \*g_i \Big( \frac{\partial \varpi_{i, e_{\bm\gamma}}}{\partial  \bm\gamma} \Big|_{\bm\gamma= \bm\gamma^\ast}\Big)^\T  \Big]   = - \+E \Big[ \underbrace{\frac{1 - W_i}{1 - e_i^\ast} \*g_i  }_{\*h_i} \*d_i^\T    \Big]. \]
	If the propensity score is estimated, then $\*V_{\bm\beta, \HT, \hat{e}}^\dagger$ takes the form of
	\begin{align*}
	& \*V_{\bm\beta^\ast, \HT, \hat{e}}^\dagger  \\ 
	=&  \*A_{\bm{\beta}^\ast, \varpi}^\I \Big( \underbrace{\+E \big[\*k_i \*k_i^\T \big] }_{\*D_{\bm{\beta}^\ast, \varpi}}   -  \underbrace{\+E \big[\*h_i \*d_i^\T \big]}_{\*C_{\bm{\beta}^\ast, \varpi,1}} \*V_{\bm\gamma}   \underbrace{\+E \big[\*d_i \*k_i^\T \big]}_{\*C^\T_{\bm{\beta}^\ast, \varpi,2}}  -  \underbrace{\+E \big[\*k_i \*d_i^\T \big]}_{\*C_{\bm{\beta}^\ast, \varpi,2}} \*V_{\bm\gamma}  \underbrace{\+E \big[\*d_i \*h_i^\T \big]}_{\*C^\T_{\bm{\beta}^\ast, \varpi,1}} + \underbrace{\+E \big[\*h_i \*d_i^\T \big]}_{\*C_{\bm{\beta}^\ast, \varpi,2}} \*V_{\bm\gamma}   \underbrace{\+E \big[\*d_i \*h_i^\T \big]}_{\*C^\T_{\bm{\beta}^\ast, \varpi,2}} \Big) \*A_{\bm{\beta}^\ast, \varpi}^\I, 
	\end{align*}
	where $\*V_{\bm\gamma}  = \*A_{\bm\gamma^\ast}^\I \*B_{\bm\gamma^\ast} \*A_{\bm\gamma^\ast}^\I $. If $e(\*X_i, \bm\gamma)$ is correctly specified, we have  $\*V_{\bm\gamma} = \+E[ \*d_i \*d_i^\T]^\I = \*A_{\bm\gamma^\ast}^\I$. If we use the true propensity score, then 
	\begin{align*}
	\*V_{\bm\beta^\ast, \HT, e}^\dagger   = \*A_{\bm{\beta}^\ast, \varpi}^\I  \underbrace{\+E \big[\*k_i \*k_i^\T \big] }_{\*D_{\bm{\beta}^\ast, \varpi}} \*A_{\bm{\beta}^\ast, \varpi}^\I. 
	\end{align*}

\end{proof}

\subsubsection{Proof of Theorem \ref{theorem:ht-est-prop} (Stable Propensity and Outcome Models)}
This proof holds for both correctly specified and misspecified propensity and outcome models.
\begin{proof}[Proof of Theorem \ref{theorem:ht-est-prop}  (Stable Propensity and Outcome Models)] 
In this proof, we show the results for the federated estimators where the estimated propensity is used. If the true propensity is used (Condition \ref{cond:known-propensity}), we can follow the same procedure to prove the results for this case. 
Our proof of Theorem \ref{theorem:ht-est-prop} consists of showing the following four equations
    \begin{enumerate} 
        \item $n_\npool^{1/2} (\hat{\*V}_{\bm\beta,\HT, \hat{e}}^{\concat,\dagger})^{-1/2} (\hat{\bm\beta}^\concat_\HT - \bm\beta^\ast) \xrightarrow{d} \mathcal{N} (0, \*I_d)$
        \item $n_\npool^{1/2} (\hat{\*V}_{\bm\beta,\HT, \hat{e}}^{\pool,\dagger})^{-1/2} (\hat{\bm\beta}^\concat_\HT - \bm\beta^\ast) \xrightarrow{d} \mathcal{N} (0, \*I_d)$
        \item $n_\npool^{1/2} (\hat{\*V}_{\bm\beta,\HT, \hat{e}}^{\pool,\dagger})^{-1/2} (\hat{\bm\beta}^\pool_\HT - \bm\beta^\ast) \xrightarrow{d} \mathcal{N} (0, \*I_d)$
        \item $n_\npool^{1/2} (\hat{\*V}_{\bm\beta,\HT, \hat{e}}^{\concat,\dagger})^{-1/2} (\hat{\bm\beta}^\pool_\HT - \bm\beta^\ast) \xrightarrow{d} \mathcal{N} (0, \*I_d)$.
        
    \end{enumerate}

    

    The first step is to show $n_\npool^{1/2} (\hat{\*V}_{\bm\beta,\HT, \hat{e}}^{\concat,\dagger})^{-1/2} (\hat{\bm\beta}^\concat_\HT - \bm\beta^\ast) \xrightarrow{d} \mathcal{N} (0, \*I_d)$. From Lemma \ref{lemma:expression-ht-var},  for the combined data (that can be viewed as a single data set), we have
	\begin{align*}
	\hat{\bm\beta}_\HT^\concat  \xrightarrow{p}& \bm\beta^\ast ,\\
	\sqrt{n_k} \big(\hat{\bm\beta}_\HT^\concat -  \bm\beta^\ast   \big) \xrightarrow{d}&\mathcal{N} \big(0,  \*V_{\bm\beta^\ast, \HT, \hat{e}}^{\dagger } \big),	
	\end{align*}
	where $\*V_{\bm\beta^\ast, \HT, \hat{e}}^{\dagger}$ is the asymptotic variance (see Lemma \ref{lemma:expression-ht-var} for its expression). From the law of large numbers and the consistency of $\hat{\bm\beta}_\HT^\concat $, we have  $\hat{\*V}_{\bm\beta,\HT, \hat{e}}^{\concat,\dagger}$ be a consistent estimator of $ \*V_{\bm\beta^\ast, \HT, \hat{e}}^\dagger $. Hence, by Slutsky's theorem, we have 
	\[n_\npool^{1/2} (\hat{\*V}_{\bm\beta,\HT,\hat{e}}^{\concat,\dagger})^{-1/2} (\hat{\bm\beta}^\concat_\HT  - \bm\beta^\ast) \xrightarrow{d} \mathcal{N} (0, \*I_d). \]
	
	The second step is to show the second equation (i.e., $n_\npool^{1/2} (\hat{\*V}_{\bm\beta,\HT, \hat{e}}^{\pool,\dagger})^{-1/2} (\hat{\bm\beta}^\concat_\HT - \bm\beta^\ast) \xrightarrow{d} \mathcal{N} (0, \*I_d)$) for the case where $ \*V_{\bm\beta^\ast, \HT, \hat{e}}^{ \sk,\dagger} $  is the same for all data sets.
	
	For this case, we drop superscript $k$ for notation simplicity. 
    In order to show the second equation, we need to additionally show the consistency of $\hat{\*V}_{\bm\beta,\HT, \hat{e}}^{\pool, \dagger} $ given what we have in the first step. To show the consistency of  $\hat{\*V}_{\bm\beta,\HT, \hat{e}}^{\pool, \dagger} $, we start with showing the consistency of $\hat{\bm\beta}^\pool_\HT$ and $\hat{\bm\gamma}_\mle^\pool$. We can follow the same procedure as the proof of $\norm{\hat{\bm\beta}^\pool_\mle - \bm\beta^\ast}_2 = o_p(1)$ in Inequality \eqref{eqn:pool-beta-mle-est-err} (in the proof of Theorem \ref{thm:pool-mle}) to show the consistency of $\hat{\bm\beta}^\pool_\HT$ and $\hat{\bm\gamma}_\mle^\pool$. 
    
    In more detail, for $\hat{\bm\beta}^\pool_\HT$ (recall we use Hessian weighting to pool $\hat{\bm\beta}_\HT^\sk$, denoting the Hessian on data set $k$ as $\hat{\*H}_{\bm\beta,\HT}^\sk$ and $\hat{p}_{n,j} = \frac{n_j}{\sum_{k = 1}^D n_k}$), 
    \begin{align*}
	& \norm{\hat{\bm\beta}^\pool_\HT - \bm\beta^\ast}_2  = \norm{\Big( \sum_{k = 1}^{D} \hat{\*H}_{\bm\beta,\HT}^\sk \Big)^\I \Big( \sum_{k = 1}^{D} \hat{\*H}_{\bm\beta,\HT}^\sk \big( \hat{\bm\beta}_\HT^\sk  - \bm\beta^\ast \big)  \Big) }_2 \\
	\leq&  \sum_{j = 1}^{D} \hat{p}_{n,j} \cdot  \underbrace{ \norm{   \Big( \sum_{k = 1}^{D} \hat{\*H}_{\bm\beta,\HT}^\sk \Big)^\I   \hat{\*H}_{\bm\beta,\HT}^\sj \cdot \frac{1}{\hat{p}_{n,j} }  \cdot \Big( \hat{\bm\beta}_\HT^\sj  - \bm\beta^\ast \Big)   }_2}_{o_p(1)} =  o_p(1),
	\end{align*}
    where we use the property that $ \big( \sum_{k = 1}^{D} \hat{\*H}_{\bm\beta,\HT}^\sk  \big)^\I   \hat{\*H}_{\bm\beta,\HT}^\sj \frac{1}{\hat{p}_{n,j}}\xrightarrow{p} \*I_d$ (which can be shown in the same procedure as Eq. \eqref{eqn:hessian-equality}, where we additionally use the consistency of $\hat{e}^\sk$). Therefore we finish the proof of the consistency of $\hat{\bm\beta}^\pool_\HT$.

    Next we show the consistency of $\hat{\*V}_{\bm\beta,\HT, \hat{e}}^{\pool, \dagger} $. Recall from Table \ref{tab:ht} that in the estimation of $\hat{\*V}_{\bm\beta,\HT, \hat{e}}^{\pool, \dagger} $, we use  $\hat{\*A}_{\betavarpi}^\sk, \hat{\*C}_{\betavarpi}^\sk, \hat{\*D}_{\betavarpi}^\sk, \hat{\*A}^\sk_{\bm\gamma}$, and $ \hat{\*B}^\sk_{\bm\gamma}$ (for ATT weighting, replace $\hat{\*C}_{\betavarpi}^\sk$ by $\hat{\*C}_{\betavarpi, 1}^\sk, \hat{\*C}_{\betavarpi, 2}^\sk$) which are estimated using  $\hat{\bm\gamma}^\pool$ and $\hat{\bm\beta}^\pool$. By the uniform weak law of large numbers, all these quantities are consistent. 
	Using exactly the same proof that showed $\hat{\bm\beta}^\pool_\HT \xrightarrow{p} \bm\beta^\ast$, we can show the consistency of $\hat{\*A}^\pool_{\betavarpi},  \hat{\*C}_{\betavarpi}^\pool, \hat{\*D}_{\betavarpi}^\pool, \hat{\*A}^\pool_{\bm\gamma}$, and $ \hat{\*B}^\pool_{\bm\gamma}$ (for ATT weighting, replace $\hat{\*C}_{\betavarpi}^\pool$ by $\hat{\*C}_{\betavarpi, 1}^\pool, \hat{\*C}_{\betavarpi, 2}^\pool$). Then, the consistency of $\hat{\*V}_{\bm\beta,\HT, \hat{e}}^{\pool, \dagger}$ can be shown:
	\begin{align}
   \nonumber \hat{\*V}_{\bm\beta,\HT, \hat{e}}^{\pool, \dagger} =& \big( \hat{\*A}^\pool_{\betavarpi} \big)^\I \big(   \hat{\*D}^\pool_{\betavarpi}  -  \hat{\*M}^\pool_{\betavarpi, \bm\gamma} \big)  \big( \hat{\*A}^\pool_{\betavarpi} \big)^\I \\
	\xrightarrow{p}&  \*A_{\bm\beta^\ast, \varpi}^\I \big(\*D_{\bm\beta^\ast, \varpi} - \*M_{\bm\beta^\ast, \varpi, \bm\gamma^\ast}   \big) \*A_{\bm\beta^\ast, \varpi}^\I = \*V_{\bm\beta^\ast, \HT, \hat{e}}^\dagger, \label{eqn:V-pool-est-ht-consistent}
	\end{align}
	where $\hat{\*M}^\pool_{\betavarpi, \bm\gamma} $ is a smooth function of $\hat{\*C}_{\betavarpi}^\pool, \hat{\*A}^\pool_{\bm\gamma}$ and $ \hat{\*B}^\pool_{\bm\gamma}$ for ATE weighting, and $\hat{\*M}^\pool_{\betavarpi, \bm\gamma} $ is a smooth function of $\hat{\*C}_{\betavarpi, 1}^\pool, \hat{\*C}_{\betavarpi, 2}^\pool, \hat{\*A}^\pool_{\bm\gamma}$, and $ \hat{\*B}^\pool_{\bm\gamma}$ for ATT weighting.
	
    Given the consistency of $\hat{\*V}_{\bm\beta,\HT, \hat{e}}^{\pool, \dagger}$, we have recovered the second equation:
	\[ n_\npool^{1/2} (\hat{\*V}_{\bm\beta,\HT,\hat{e}}^{\pool,\dagger})^{-1/2} (\hat{\bm\beta}^\concat_\HT  - \bm\beta^\ast) \xrightarrow{d} \mathcal{N} (0, \*I_d)   \]
	
	The third step is to show the third and fourth equations together for the case where $ \*V_{\bm\beta^\ast, \HT, \hat{e}}^{ \sk,\dagger} $  is the same for $k$ ($n_\npool^{1/2} (\hat{\*V}_{\bm\beta,\HT, \hat{e}}^{\pool,\dagger})^{-1/2} (\hat{\bm\beta}^\pool_\HT - \bm\beta^\ast) \xrightarrow{d} \mathcal{N} (0, \*I_d)$
        and $n_\npool^{1/2} (\hat{\*V}_{\bm\beta,\HT, \hat{e}}^{\concat,\dagger})^{-1/2} (\hat{\bm\beta}^\pool_\HT - \bm\beta^\ast) \xrightarrow{d} \mathcal{N} (0, \*I_d)$). Given the consistency of $\hat{\*V}_{\bm\beta,\HT,\hat{e}}^{\concat,\dagger}$ and $\hat{\*V}_{\bm\beta,\HT,\hat{e}}^{\pool,\dagger}$ (from the proofs of the first and second equations), if we can show $\hat{\bm\beta}_\HT^\pool$ converges to $\bm\beta^\ast$ in an asymptotic normal distribution with the convergence rate $n_\npool^{1/2}$ and asymptotic variance  with the asymptotic variance $\*V_{\bm\beta^\ast, \HT, \hat{e}}^{\dagger}$, then by Slutsky's theorem, we obtain the third and fourth equations. 
	
	Since observations between data sets are asymptotically independent, we have that $\Big(n_1^{1/2}  \big(\hat{\bm\beta}_\HT^{(1)} -  \bm\beta^\ast   \big) , n_2^{1/2} \big(\hat{\bm\beta}_\HT^{(2)} -  \bm\beta^\ast   \big), \cdots, n_D^{1/2} \big(\hat{\bm\beta}_\HT^D -  \bm\beta^\ast   \big)    \Big)$ converges jointly to a normal distribution, for any $j \neq k$, $n_j^{1/2}  \big(\hat{\bm\beta}_\HT^\sj -  \bm\beta^\ast   \big)$ and $n_k^{1/2}  \big(\hat{\bm\beta}_\HT^\sk -  \bm\beta^\ast   \big)$ are independent, and
	\begin{align*}
	 n_\npool ^{1/2} \Big( \hat{\bm\beta}^\pool_\HT  -  \bm\beta^\ast \Big)  
	=  \sum_{j = 1}^{D} \hat{p}_{n,j}^{1/2}  \Big[ \underbrace{\big( \sum_{k = 1}^{D} \hat{\*H}_{\bm\beta,\HT}^\sk  \big)^\I   \hat{\*H}_{\bm\beta,\HT}^\sj \frac{1}{\hat{p}_{n,j}}\cdot n_j^{1/2} \big(\hat{\bm\beta}_\HT^\sj   - \bm\beta^\ast \big)}_{\substack{\coloneqq \bm\xi_{n_j}^\sj \xrightarrow{d} \mathcal{N}(0,  \*V_{\bm\beta^\ast, \HT, \hat{e}}^\dagger )  \text{ from } \\ \big( \sum_{k = 1}^{D} \hat{\*H}_{\bm\beta,\HT}^\sk  \big)^\I   \hat{\*H}_{\bm\beta,\HT}^\sj \frac{1}{\hat{p}_{n,j}}\xrightarrow{p} \*I_d \text{ and Slutsky's theorem}  } }  \Big].
	\end{align*}
    As $\hat{p}^{1/2}_{n,j} \rightarrow p_j$, by Slutsky's theorem, we have
	\begin{align*}
	n_\npool ^{1/2}  \big(\hat{\*V}_{\bm\beta, \HT, \hat{e}}^{\pool, \dagger}\big)^{-1/2}  \big(\hat{\bm\beta}_\HT^\pool   - \bm\beta^\ast \big)  \xrightarrow{d}& \mathcal{N} \big( 0, \*I_d \big) \\
	n_\npool ^{1/2}  \big(\hat{\*V}_{\bm\beta, \HT, \hat{e}}^{\concat, \dagger}\big)^{-1/2}  \big(\hat{\bm\beta}_\HT^\pool   - \bm\beta^\ast \big)  \xrightarrow{d}& \mathcal{N} \big( 0, \*I_d \big).    
	\end{align*}
	
	The last step is to show the second to fourth equations for the case where $ \*V_{\bm\beta^\ast, \HT, \hat{e}}^{ \sk,\dagger} $ differs across data sets. Based on what we have from the first case, we only need to additionally show that $\hat{\bm\beta}^\pool_\HT $ and $ \hat{\*V}_{\bm\beta,\HT, \hat{e}}^{\pool, \dagger}$ are consistent and $\hat{\bm\beta}^\pool_\HT $ is asymptotically normal with variance $\*V_{\bm\beta^\ast, \HT, \hat{e}}^\dagger$ even when $ \*V_{\bm\beta^\ast, \HT, \hat{e}}^{ \sk,\dagger} $  differs across data sets. 
	
	Let us start with the consistency of $\hat{\bm\beta}^\pool_\HT $.
	Recall from our federation procedure of the IPW-MLE estimator that we first estimate the propensity model on the combined data and use this federated propensity model to estimate ${\bm\beta}^\sk_\HT$ on each data set. Then, for the ATE weighting, the asymptotic distribution of $\hat{\bm\beta}^\sk_\HT$ satisfies (ATT weighting can be shown analogously with a similar equation):
	\begin{align*}
	n_k^{1/2} \big( \hat{\bm\beta}^\sk_\HT - \bm\beta^\ast \big) =&  - (\*A^\sk_{\bm\beta^\ast, \varpi})^\I  \Bigg(  \frac{1}{n_j^{1/2}} \sum_{i = 1 }^{n_k} \*k_i -  \*C^\sk_{\bm\beta^\ast, \varpi} \cdot \big(\*A_{\bm\gamma^\ast}^\concat\big)^\I \cdot \hat{p}_{n,k}^{1/2} \cdot \frac{1}{n_\npool^{1/2}} \sum_{i = 1 }^{n_\npool}  \*d_i  \Bigg) + o_p(1) \\
	\xrightarrow{d}&  \mathcal{N} \bigg( 0,  (\*A^\sk_{\bm\beta^\ast, \varpi})^\I  \Big(\*D^\sk_{\bm\beta^\ast, \varpi}  - \*C^\sk_{\bm\beta^\ast, \varpi} \cdot p_k \*V_{\bm\gamma}^\concat \cdot \*C^\sk_{\bm\beta^\ast, \varpi}   \Big)  (\*A^\sk_{\bm\beta^\ast, \varpi})^\I   \bigg),
	\end{align*}
	 where the definitions of $\*k_i$ and $\*d_i$ can be found in the proof of Lemma \ref{lemma:expression-ht-var}. 
	 Note that we have $\hat{\*H}_{\bm\beta,\HT}^\sk/n_k \xrightarrow{p} \*A^\sk_{\bm\beta^\ast, \varpi}$. Since $\hat{\bm\beta}_\HT^\sk$ is consistent, we have $\sum_{k = 1}^{D} \hat{\*H}_{\bm\beta,\HT}^\sk /n_\npool \xrightarrow{p}  \*A^\concat_{\bm\beta^\ast, \varpi}$, and therefore $ \big( \sum_{k = 1}^{D} \hat{\*H}_{\bm\beta,\HT}^\sk  \big)^\I \cdot  \hat{\*H}_{\bm\beta,\HT}^\sj \cdot \frac{1}{\hat{p}_{n,j}}\xrightarrow{p} ( \*A^\concat_{\bm\beta^\ast, \varpi})^\I  \*A^\sk_{\bm\beta^\ast, \varpi}$. Given the assumption $\norm{(\*A^\concat_{\bm\beta^\ast, \varpi})^\I  \*A^\sk_{\bm\beta^\ast, \varpi}}_2 \leq M$, then $ \big( \sum_{k = 1}^{D} \hat{\*H}_{\bm\beta,\HT}^\sk  \big)^\I   \hat{\*H}_{\bm\beta,\HT}^\sj \cdot  \frac{1}{\hat{p}_{n,j}}\cdot \big( \hat{\bm\beta}_\HT^\sk  - \bm\beta^\ast \big)   = o_p(1)$ continues to hold, and therefore $ \norm{\hat{\bm\beta}^\pool_\HT - \bm\beta^\ast}_2  = o_p(1)$ (where 
	$\hat{\bm\gamma}_\mle^\pool \xrightarrow{p} \bm\gamma^\ast$ can be shown using exactly the same proof).
	
	Lastly, we show the asymptotic distribution of $\hat{\bm\beta}^\pool_\HT$. Using $ \big( \sum_{k = 1}^{D} \hat{\*H}_{\bm\beta,\HT}^\sk  \big)^\I   \hat{\*H}_{\bm\beta,\HT}^\sj \cdot \frac{1}{\hat{p}_{n,j}}\xrightarrow{p} (\*A^\concat_{\bm\beta^\ast, \varpi})^\I  \*A^\sk_{\bm\beta^\ast, \varpi} $, we have the following for ATE weighting (with similar arithmetic for ATT weighting):
	\begin{align*}
	& n_\npool ^{1/2} \Big( \hat{\bm\beta}^\pool_\HT  -  \bm\beta^\ast \Big) 
	=  -\sum_{j = 1}^{D} \hat{p}_{n,j}^{1/2}  \Big[  \big( \sum_{k = 1}^{D} \hat{\*H}_{\bm\beta,\HT}^\sk  \big)^\I   \hat{\*H}_{\bm\beta,\HT}^\sj \cdot \frac{1}{\hat{p}_{n,j}}\cdot n_j^{1/2} \big(\hat{\bm\beta}_\HT^\sj   - \bm\beta^\ast \big) \Big] \\
	=&  - (\*A^\concat_{\bm\beta^\ast, \varpi})^\I  \sum_{j = 1}^{D} \hat{p}_{n,j}^{1/2}  \Bigg(  \frac{1}{n_j^{1/2}} \sum_{i = 1 }^{n_j} \*k_i -  \*C^\sj_{\bm\beta^\ast, \varpi} \cdot \big(\*A_{\bm\gamma^\ast}^\concat\big)^\I \cdot \hat{p}_{n,j}^{1/2} \cdot \frac{1}{n_\npool^{1/2}} \sum_{i = 1 }^{n_\npool}  \*d_i  \Bigg) + o_p(1) \\
	=& - \frac{ (\*A^\concat_{\bm\beta^\ast, \varpi})^\I}{n_\npool^{1/2}}\Bigg(  \sum_{i = 1 }^{n_\npool} \*k_i - \bigg(  \sum_{j = 1}^{D} \frac{n_j}{n_\npool}  \*C^\sj_{\bm\beta^\ast, \varpi}    \bigg) \cdot \big(\*A_{\bm\gamma^\ast}^\concat\big)^\I \cdot \sum_{i = 1 }^{n_\npool}  \*d_i  \Bigg) + o_p(1) \\
	\xrightarrow{d}&  \mathcal{N} \bigg( 0, \*V_{\bm\beta^\ast, \HT, \hat{e}}^\dagger  \bigg),
	\end{align*}
 where 
 \[\*V_{\bm\beta^\ast, \HT, \hat{e}}^\dagger =(\*A^\concat_{\bm\beta^\ast, \varpi})^\I \Big( \*D^\concat_{\bm\beta^\ast, \varpi}  - \*C^\concat_{\bm\beta^\ast, \varpi} \cdot  \*V_{\bm\gamma^\ast}^\concat \cdot \*C^\concat_{\bm\beta^\ast, \varpi} \Big)  (\*A^\concat_{\bm\beta^\ast, \varpi})^\I   \]
    We have hence shown the asymptotic distribution of $\hat{\bm\beta}^\pool_\HT$, which completes the proof in the second case. 
\end{proof}

\subsubsection{Proof of Theorem \ref{theorem:ht-est-prop} (Unstable Propensity and/or Unstable Outcome Models)}

\begin{proof}[Proof of Theorem \ref{theorem:ht-est-prop} (Stable Propensity and Outcome Models)]
    The results follow directly from the proof of the unstable outcome models in Theorem \ref{thm:pool-mle} and Theorem \ref{theorem:ht-est-prop}. Details are therefore omitted and available upon request. 
\end{proof}

\subsection{Proof of Results for Federated AIPW in Section \ref{subsec:aipw-results}}

\begin{proof}[Proof of Theorem \ref{theorem:pool-aipw}]
    In order to prove Theorem \ref{theorem:pool-aipw}, let us first review some properties of $\hat{\tau}^\aipw$ estimated from a single data set. 
	If either the propensity or outcome model is correctly specified, $\hat{\tau}_\aipw$ is asymptotically linear \citep{tsiatis2007comment},
	\begin{align}
	\sqrt{n} \big( \hat{\tau}_\aipw - \tau_0 \big) = \frac{1}{\sqrt{n}} \sum_{i = 1 }^n \phi(\*X_i, W_i, Y_i ) + o_p(1) \xrightarrow{d} \mathcal{N} \big(0, {\*V}_\tau \big),
	\end{align}
	where  $\phi(\*x, w, y)$ is an  influence function that satisfies $\+E[\phi(\*x, w, y)] = 0$ and ${\*V}_\tau = \+E[\phi(\*x, w, y)^2] < \infty$. Suppose the score function of $s(\*X_i, W_i, Y_i)$ can be parameterized by $\bm{\theta}$, with the true value being $\bm{\theta}_0$; then, the treatment effect $\tau_0$ can also be parameterized, i.e., $\tau_0 = \tau(\bm{\theta}_0)$, and $\tau_0$ is differentiable in $\bm{\theta}$. From \cite{newey1994asymptotic}, $\phi(\*X_i, W_i, Y_i )$ as a valid influence function connects $\tau_0$ and $s(\*X_i, W_i, Y_i \mid \bm{\theta})$  via
	\begin{equation}\label{eqn:influence-function-equality}
	    \frac{\partial \tau(\bm{\theta}_0)}{\partial \bm{\theta}} = \+E \big[ \phi(\*X_i, W_i, Y_i ) s(\*X_i, W_i, Y_i \mid \bm{\theta}_0) \big]. 
	\end{equation}

	Now we are ready to show Theorem \ref{theorem:pool-aipw}. We aim to find a valid influence function that satisfies \eqref{eqn:influence-function-equality} on the combined data, and then we can use this valid influence function to provide the asymptotic distribution of $\hat{\tau}^\concat_\aipw$ and $\hat{\tau}^\pool_\aipw$. The population treatment effect and score function on the combined data set satisfy the following (recall that $p_j = \lim {n_j}/{n_\npool}$):
	\begin{align*}
	    \tau_0 =& \sum_{j = 1 }^D p_j \tau_0^\sj \\
	    s^\concat(\*X^\sk_i, W^\sk_i, Y^\sk_i \mid \bm\theta_0) =& \sum_{j = 1 }^D \mathbbm{1} (k = j)  s^\sj(\*X^\sk_i, W^\sk_i, Y^\sk_i \mid \bm{\theta}^\sj_0).
	\end{align*}
	
	Let a candidate influence function on the combined data set be
		\[ \phi^\concat(\*X^\sk_i, W^\sk_i, Y^\sk_i ) = \sum_{j = 1}^D \mathbbm{1} (k = j) \phi^\sj (\*X^\sk_i, W^\sk_i, Y^\sk_i ). \]
	This candidate influence function satisfies $\+E[\phi^\concat(\*x, w, y ) ] = 0$, $\+E[\phi^\concat(\*x, w, y )^2 ] < \infty$,
	\begin{equation}\label{eqn:influence-function-equivalence}
	    \phi^\concat(\*X^\sk_i, W^\sk_i, Y^\sk_i ) = \phi^\sk(\*X^\sk_i, W^\sk_i, Y^\sk_i ), 
	\end{equation}
	and
	\begin{align*}
	    \frac{\partial \tau(\bm\theta_0)}{\partial \bm\theta} =&  \sum_{j = 1}^D p_j \frac{\partial \tau(\bm{\theta}_0^\sj)}{\partial \bm{\theta}^\sj} =    \sum_{j = 1}^D p_j  \+E \big[ \phi^\sj(\*X_i^\sj, W_i^\sj, Y_i^\sj ) s(\*X_i^\sj, W_i^\sj, Y_i^\sj \mid \bm{\theta}_0^\sj) \big] \\  =& \+E \big[ \phi^\concat(\*X_i, W_i, Y_i ) s^\concat(\*X_i, W_i, Y_i \mid \bm\theta_0) \big],
	\end{align*} 
	i.e., equality \eqref{eqn:influence-function-equality} holds for $\phi^\concat(\*X_i, W_i, Y_i ) $, and therefore, $\phi^\concat(\*X^\sk_i, W^\sk_i, Y^\sk_i ) $ is a valid influence function. Based on this influence function, we have 
	\[n^{1/2}_\npool \big(\hat{\tau}^\concat_\aipw - \tau_0  \big) \xrightarrow{d} \mathcal{N} \big(0, {\*V}_\tau^\concat \big) \]
	where the asymptotic variance ${\*V}_\tau^\concat$ satisfies 
	\[ {\*V}_\tau^\concat = \+E \big[\phi^\concat(\*X^\sk_i, W^\sk_i, Y^\sk_i )^2 \big] = \sum_{j = 1}^D p_j \+E\big[ \phi^\sj (\*X^\sk_i, W^\sk_i, Y^\sk_i )^2 \big] = \sum_{j = 1}^D p_j {\*V}_\tau^\sk \]
	using the property that $\mathbbm{1} (k = j) \cdot \mathbbm{1} (k = l) = 0$ for $j \neq l$, where ${\*V}_\tau^\sk$ is the asymptotic variance on data set $k$.
	
	$\hat{\*V}^\concat_\tau$ is consistent from Lemma \ref{lemma:expression-aipw-var} and the definition of $\hat{\*V}^\concat_\tau$, and from Slutsky's theorem, we have 
	\[n_\npool^{1/2} (\hat{\*V}^\concat_\tau)^{-1/2} (\hat{\tau}^\concat_\aipw - \tau_0)  \xrightarrow{d} \mathcal{N}(0,1). \]
	
	For the case where $\phi(\*X_i, W_i, Y_i)$ varies with the data set, the federated treatment effect $\hat{\tau}^\pool_\aipw$ from sample size weighting in Section \ref{subsec:pool-aipw-model-shift} satisfies
	\begin{align}
	\nonumber n_\npool^{1/2} (\hat{\tau}^\pool_\aipw - \tau_0)=& n_\npool^{1/2}  \sum_{k = 1}^{D}  \frac{n_k}{n_\npool} \cdot \frac{1}{n_k} \sum_{i = 1 }^{n_k}  \phi^\sk(\*X^\sk_i, W^\sk_i, Y^\sk_i ) + o_p(1)  \\
	=& \frac{1}{ n_\npool^{1/2}}  \sum_{k = 1}^{D}  \sum_{i = 1 }^{n_k}  \phi^\concat(\*X^\sk_i, W^\sk_i, Y^\sk_i ) + o_p(1) 
	\xrightarrow{d} \mathcal{N}(0, {\*V}^\npool_\tau). \label{eqn:pool-aipw-heter}
	\end{align}
	The federated variance $\hat{\*V}^\pool$ from sample size weighting in Section \ref{subsec:pool-aipw-model-shift} satisfies
	\[\hat{\*V}^\pool_\tau =  \sum_{k = 1}^D \frac{n_k}{n_\npool} \hat{\*V}^\sk_\tau = \sum_{k = 1}^D \frac{n_k}{n_\npool} \hat{\*V}^\sk_\tau \xrightarrow{p}  \sum_{k = 1}^D  p_k \*V_\tau^\sk = \*V^\concat_\tau,  \]
	where we use the property that $\hat{\*V}_\tau^\sk \xrightarrow{p} \*V_\tau^\sk$ from Lemma \ref{lemma:expression-aipw-var}. 
	
	For the case where $\phi(\*X_i, W_i, Y_i)$ is the same across data sets, we have $\*V^\concat_\tau \equiv \*V^\sk_\tau = \*V_\tau $ for all $k$ and for some $\*V_\tau$. Then, the federated variance $\hat{\*V}^\pool_\tau$ from sample size weighting in Section \ref{subsec:pool-aipw-model-shift} satisfies 
	\[\hat{\*V}^\pool_\tau = \Big( \sum_{k = 1}^D \big(\hat{\*V}^\sk_\tau\big)^\I \Big)^\I \xrightarrow{p} \*V_\tau. \]
	The federated treatment effect $\hat{\tau}^\pool_\aipw$ from inverse variance weighting in Section \ref{subsec:pool-stable-aipw} satisfies
	\begin{align*}
 n_\npool^{1/2} \big(\hat{\tau}^\pool_\aipw - \tau_0 \big)	 =&   n_\npool^{1/2}  \Big( \sum_{k = 1}^{D} (\hat{\*V}_{\tau}^\sk)^\I \Big)^\I \Big( \sum_{k = 1}^{D} (\hat{\*V}_{\tau}^\sk)^\I  (\hat{\tau}_\aipw^\sk - \tau_0)  \Big)  \\
	 =& n_\npool^{1/2}  \sum_{k = 1}^{D} \frac{n_k}{n_\npool}  (\hat{\tau}_\aipw^\sk - \tau_0) + o_p(1)  \\ 
	=& n_\npool^{1/2}  \sum_{k = 1}^{D}  \frac{n_k}{n_\npool} \cdot \frac{1}{n_k} \sum_{i = 1 }^{n_k}  \phi^\concat(\*X^\sk_i, W^\sk_i, Y^\sk_i ) + o_p(1) 	\xrightarrow{d} \mathcal{N}(0, {\*V}^\npool_\tau),
	\end{align*}
	where the second equality uses Eq. \eqref{eqn:influence-function-equivalence}.
	
	For both cases, $\hat{\tau}^\pool_\aipw$ is asymptotically normal, and $\hat{\*V}^\pool_\tau$ is consistent. Then, from Slutsky's theorem, we have
	\begin{align*}
	    n_\npool^{1/2}  (\hat{\*V}^\pool_\tau)^{-1/2} (\hat{\tau}^\concat_\aipw - \tau_0)  \xrightarrow{d}& \mathcal{N}(0,1) \\
	    n_\npool^{1/2}  (\hat{\*V}^\concat_\tau)^{-1/2} (\hat{\tau}^\pool_\aipw - \tau_0)  \xrightarrow{d}& \mathcal{N}(0,1) \\
	    n_\npool^{1/2}  (\hat{\*V}^\pool_\tau)^{-1/2} (\hat{\tau}^\pool_\aipw - \tau_0)  \xrightarrow{d}& \mathcal{N}(0,1).
	\end{align*}
\end{proof}

}

\end{appendices}

\end{document}